%% file: 000_main_arxiv.tex
\documentclass{article}
\usepackage[preprint,nonatbib]{neurips_2022}
\makeatletter
\renewcommand{\@noticestring}{Preprint.}
\makeatother
\input{001_preamble}
\newcommand{\PaperArxivLayout}{}
\input{002_macro}
\input{003_metadata}
\title{\papertitle}
\author{\paperauthors\\\paperaffiliations}
\date{}

\begin{document}
\hypersetup{pdftitle={\papertitle},pdfauthor={\paperauthors}}
\maketitle
\input{009_abst_intro}
\input{010_body}
\input{015_equivalent_models}
\input{020_simulation}
\input{070_discussion}
\input{072_conclusion}
\FloatBarrier
\bibliographystyle{apalike}
\bibliography{references}

\clearpage
\appendix
\input{090_appendix}
\end{document}

%% file: 001_preamble.tex
\usepackage[utf8]{inputenc}
\usepackage[T1]{fontenc}
\usepackage{amsmath,amssymb,amsthm,mathtools,mathrsfs}
\usepackage{graphicx,booktabs,array,longtable,microtype,xcolor,placeins}
\usepackage{wrapfig,needspace}
\usepackage{tikz}
\usetikzlibrary{arrows.meta,positioning}
\DeclareSymbolFont{independencesymbols}{U}{txsyc}{m}{n}
\DeclareMathSymbol{\independent}{\mathrel}{independencesymbols}{121}
\usepackage[round,authoryear]{natbib}
\usepackage[hidelinks]{hyperref}
\usepackage[capitalize,noabbrev]{cleveref}
\graphicspath{{figs/}}
\theoremstyle{plain}
\newtheorem{theorem}{Theorem}[section]
\newtheorem{prop}[theorem]{Proposition}
\newtheorem{lemma}[theorem]{Lemma}
\newtheorem{cor}[theorem]{Corollary}
\theoremstyle{definition}

\newtheorem{assumption}[theorem]{Assumption}

\newtheorem{remark}[theorem]{Remark}
\crefformat{equation}{(#2#1#3)}
\Crefformat{equation}{(#2#1#3)}
\crefrangeformat{equation}{(#3#1#4)--(#5#2#6)}
\Crefrangeformat{equation}{(#3#1#4)--(#5#2#6)}
\crefmultiformat{equation}{(#2#1#3)}{ and~(#2#1#3)}{, (#2#1#3)}{, and~(#2#1#3)}
\Crefmultiformat{equation}{(#2#1#3)}{ and~(#2#1#3)}{, (#2#1#3)}{, and~(#2#1#3)}
\crefname{assumption}{Assumption}{Assumptions}
\Crefname{assumption}{Assumption}{Assumptions}
\crefname{cor}{Corollary}{Corollaries}
\Crefname{cor}{Corollary}{Corollaries}
\crefname{prop}{Proposition}{Propositions}
\Crefname{prop}{Proposition}{Propositions}
\numberwithin{equation}{section}

%% file: 002_macro.tex
\newcommand{\R}{\mathbb{R}}
\newcommand{\C}{\mathbb{C}}
\newcommand{\N}{\mathbb{N}}
\newcommand{\E}{\mathbb{E}}
\newcommand{\Prob}{\mathbb{P}}
\newcommand{\Var}{\operatorname{Var}}

\DeclareMathOperator{\Tr}{Tr}
\DeclareMathOperator{\tr}{tr}
\DeclareMathOperator{\diag}{diag}
\DeclareMathOperator{\rank}{rank}
\DeclareMathOperator{\softmax}{softmax}

\newcommand{\ntoken}{n}
\newcommand{\dimqk}{d_{\mathrm K}}
\newcommand{\dimvalue}{d_{\mathrm V}}
\newcommand{\dimout}{d_{\mathrm O}}
\newcommand{\Normal}{\mathcal N}
\input{002_notation}

\ifdefined\PaperMonochrome
  \colorlet{blockfirst}{black}
  \colorlet{blocklast}{black}
\else
  \definecolor{blockfirst}{RGB}{0,91,103}
  \definecolor{blocklast}{RGB}{100,58,103}
\fi
\newcommand{\firstblock}[1]{\textcolor{blockfirst}{#1}}
\newcommand{\lastblock}[1]{\textcolor{blocklast}{#1}}

%% file: 002_notation.tex
\newcommand{\Gmark}{\mathrm{G}}            % Gaussian side
\newcommand{\allheads}[1]{#1_{\oplus}}     % direct sum over the heads
\newcommand{\cntr}[1]{\mathring{#1}}       % conditional centering of rows

\newcommand{\Ahead}[1]{A_{#1}^{\perp}}     % centered attention of head #1
\newcommand{\Aperp}{A^{\perp}}             % ordinary concatenation
\newcommand{\scaledA}{\sqrt{\ntoken}\Aperp}
\newcommand{\MHA}{\operatorname{MHA}}
\newcommand{\MHAperp}{\operatorname{MHA}^{\perp}}
\newcommand{\MHAG}{\operatorname{MHA}^{\Gmark}}

\newcommand{\featfn}{f}                    % entry function of the features
\newcommand{\Yfh}[1]{Y_{#1}^{\featfn}}     % features of head #1
\newcommand{\Yf}{Y^{\featfn}}              % all heads
\newcommand{\Yfc}{\cntr{Y}^{\featfn}}      % rows centered given the keys
\newcommand{\noiseh}[1]{W_{#1}}            % comparison noise
\newcommand{\auxnoise}{\widetilde W}       % auxiliary noise
\newcommand{\YGh}[1]{Y^{\Gmark}_{#1}}        % Gaussian model of head #1
\newcommand{\YG}{Y^{\Gmark}}               % all heads
\newcommand{\YGcov}[1]{Y^{\Gmark}[#1]}     % auxiliary noise, row covariance #1/n
\newcommand{\YGSig}{\YGcov{\covf}}
\newcommand{\YGC}{\YGcov{\covG}}

\newcommand{\rowgram}{\widehat\Gamma}    % Gram matrix n^{-1}\sum_i y_iy_i^T of the rows

\newcommand{\covK}{\Sigma_K}
\newcommand{\covf}{\Sigma_\featfn}
\newcommand{\covG}{\Sigma_{\Gmark}}
\newcommand{\covKh}[1]{\Sigma_{K,#1}}
\newcommand{\covfh}[1]{\Sigma_{\featfn,#1}}
\newcommand{\covGh}[1]{\Sigma_{\Gmark,#1}}
\newcommand{\covKent}[1]{(\covK)_{#1}}
\newcommand{\covV}{\Sigma_{\mathrm V}}
\newcommand{\covVh}[1]{\Sigma_{\mathrm V,#1}}
\newcommand{\covVO}{\Sigma_{\mathrm{VO}}}

\newcommand{\Pperpall}{\allheads{P}^\perp}
\newcommand{\Kall}{\allheads{K}}
\newcommand{\Vall}{\allheads{V}}
\newcommand{\Aperpall}{\allheads{A}^{\perp}}

\newcommand{\ZV}{Z_V}
\newcommand{\ZVG}{Z_V^{\Gmark}}
\newcommand{\ZKV}{Z_{\mathrm{KV}}}
\newcommand{\ZKVG}{Z_{\mathrm{KV}}^{\Gmark}}

\newcommand{\keyratio}{\gamma_{\mathrm K}}
\newcommand{\valueratio}{\gamma_{\mathrm V}}
\newcommand{\outputratio}{\gamma_{\mathrm O}}
\newcommand{\nparam}{N_{\mathrm{par}}} % number of distinct Q/K/V/O weights
\newcommand{\freepois}{\pi}
\newcommand{\dilate}[2]{\mathcal D_{#1}\left(#2\right)}

\newcommand{\Stransform}{\mathscr{S}}
\newcommand{\MP}{\mathrm{MP}}
\newcommand{\attnlimit}{\nu_{\infty,h}}
\newcommand{\kvmap}{\varphi_{\mathrm{KV}}}
\newcommand{\vlaw}{\mu_{\mathrm V}}
\newcommand{\volaw}{\mu_{\mathrm{VO}}}
\newcommand{\kvlaw}{\mu_{\mathrm{KV}}}
\newcommand{\pushforward}[2]{{#1}_\# #2}

%% file: 003_metadata.tex
\newcommand{\papertitle}{Gaussian Equivalence for Multi-Head Self-Attention}
\newcommand{\paperauthors}{Tomohiro Hayase, Ryo Karakida}
\newcommand{\paperaffiliations}{Artificial Intelligence Research Center (AIRC), AIST\\RIKEN AIP}

%% file: 009_abst_intro.tex
\begin{abstract}
A theoretical understanding of multi-head self-attention is fundamental
to the study of modern neural networks.
Using random matrix theory, we establish Gaussian equivalence for
multi-head self-attention:
replacing softmax attention with rescaled scores plus Gaussian noise
preserves the limiting spectral law of the centered output.
This equivalence also covers value and output projections that depend
on the keys.
The resulting laws separate the effects of head allocation
and projection widths, and distinguish spectrum-preserving across-head
sharing from within-head key--value dependence.
\end{abstract}

\section{Introduction}
\label{sec:introduction}
Multi-head self-attention \citep{vaswani2017attention} is a core component
of Transformer architectures. Developing a mathematical account of its
behavior is a fundamental objective in the theory of deep learning.
For feedforward networks, random matrix theory has provided precise
descriptions of nonlinear models
\citep{pennington2017nonlinear,louart2018random}.
Gaussian equivalence \citep{peche2019note,benigni2021eigenvalue}
offers a useful principle:
a complex nonlinear model can share its limiting spectrum with a
simpler Gaussian model. We develop this approach for multi-head self-attention.

For a single softmax attention matrix, \citet{hayase2026gaussian}
establish Gaussian equivalence in a proportional regime with fixed
inverse temperature. Their analysis retains the bilinear query--key
scores and rowwise softmax normalization, yielding a limiting law for
the smaller squared singular values. This bulk law complements the
rank-collapse analysis of
\citet{naitsaada2025mind}: as the context length grows, one singular
value remains macroscopic, while the others lie on a smaller scale.
Together, these results describe both the dominant common component
and the distribution of feature strengths beneath it.

Multi-head attention (MHA) \citep{vaswani2017attention} applies each attention
matrix to its value features, concatenates the head outputs, and
applies an output projection.
These operations determine the representation passed to subsequent
layers. The spectral moments of a concatenated Gram matrix contain
products of different heads, and a value projection may depend on the
keys that also determine attention. Marginal single-head laws do not
specify these interactions. A layer-level comparison is therefore
needed to understand how head allocation, projection widths, and
weight sharing affect the output representation.

We establish this comparison for orthogonal inputs, Gaussian queries
and keys, proportional widths, and fixed head count $h$ and inverse
temperature $\beta$. The joint family of keys and value/output weights is
independent of the queries; the projections satisfy operator-norm
bounds and may depend on the keys. The proof replaces entire
concatenated rows with Gaussian rows of matching covariance and
controls the exponential directly. The comparison is preserved under
key-dependent right multiplication satisfying \cref{ass:right-multipliers}
(\cref{lem:replacement}). Under a deterministic limiting spectrum for
the conditional covariance of the Gaussian output rows, the centered
MHA output and its Gaussian equivalent converge to the same spectral law, almost surely
both weakly and in every fixed moment (\cref{thm:main}).

Frozen-feature readouts provide a controlled setting for isolating
the effects of architectural choices on representations and readout learning.
The resulting spectral laws provide quantitative predictions for
linear-readout learning from frozen multi-head self-attention features
(\cref{sec:consequences}). At a fixed projection parameter count,
changing the head count can produce multiple estimation-error peaks.
Across-head K/V sharing preserves the limiting mean training-loss
curve at fixed head count and widths, allowing parameter savings to
be reinvested in wider heads. Complete within-head K/V tying changes
the spectrum and can reduce estimation error at the same widths
and parameter count as GQA.
At a fixed projection parameter count, inverse temperature can also
reverse which Q/K--V width allocation yields lower estimation error.
Together, these examples illustrate how Gaussian equivalence makes
the learning consequences of head counts, projection widths, sharing,
and temperature analytically accessible.

\subsection{Related Work}
The closest spectral analyses are those of
\citet{naitsaada2025mind} and \citet{hayase2026gaussian}.
The former study a spectral gap for query--key attention and, in their
Appendix A.4, compute an output bulk for a random Markov surrogate
with independent pre-normalization entries. The latter preserve the
bilinear score structure and prove a single-head Gaussian equivalence
whose bulk generally differs from that surrogate's law.
\citet{liao2026interpolation} analyze nonlinear attention interpolation
in a proportional regime with random inputs and fixed structured weights,
and discuss an extension to finitely many heads.
We instead fix the input and compare softmax output spectra across
concatenated heads and key-dependent random projections.

At fixed sequence length, \citet{hron2020infinite} derive neural
network Gaussian process (NNGP) and neural tangent kernel (NTK)
limits for attention with infinitely many heads, and
\citet{sakai2025infinite} characterize non-Gaussian output limits
with finitely many heads. \citet{bordelon2024infinite} study
transformer training dynamics in limits of head width, head count,
and depth, also keeping sequence length fixed.
We establish spectral Gaussian equivalence as token count and
projection widths grow proportionally at fixed head count,
retaining finite width-to-token ratios in the limiting output spectrum.

Spectral descriptions also connect feature representations to learning.
\citet{dobriban2018high} analyze ridge prediction risk with general
feature covariance, and \citet{mei2022generalization} obtain precise
risk asymptotics for random-feature regression.
\citet{hu2023universality} establish Gaussian equivalence for training
and generalization errors in random-feature learning.
The relation between regularization and early stopping
\citep{ali2019continuous} also motivates studying learning curves.
Multiple-descent analyses identify distinct sources of risk peaks
\citep{dascoli2020triple,meng2024multiple}.
For frozen MHA features, \cref{sec:consequences} specifies gradient-flow
and ridge models whose training loss and coefficient estimation error
can be predicted from the GE spectrum.

Head number and width have been studied through expressivity and
training. \citet{bhojanapalli2020low} identify restrictions on attention
maps imposed by head width, while \citet{michel2019sixteen} investigate
pruning heads from trained models. Unequal Q/K and V widths are
used in LeViT and Thin Keys, Full Values
\citep{graham2021levit,yao2026thinkeys}.
Our experiments vary head count at fixed per-head widths and compare
width allocations at fixed projection parameter count.
The spectral laws connect these choices to readout behavior for
randomly initialized attention features.

Multi-query attention \citep{shazeer2019fast} and grouped-query attention (GQA)
\citep{ainslie2023gqa} share keys and values across heads to reduce the
cost of decoding. Their sharing pattern motivates our comparison
between independent and shared head families.
Within-head K/V tying is studied by \citet{kayyam2026projections}.
Our laws distinguish the two sharing mechanisms: across-head sharing
preserves the explicit Gaussian bulk laws at fixed head count and
widths, while within-head K/V dependence changes the covariance
governing the output.

Rank collapse and signal propagation provide a complementary
perspective. \citet{dong2021attention} prove rank collapse with depth
in pure attention stacks, and \citet{noci2022signal} connect collapse
to vanishing query and key gradients. Propagation theories
\citep{dinan2023effective,kedia2024transformers,cowsik2025geometric} track activation or
gradient statistics across layers. Our proportional-limit analysis
separates the common mode from the smaller feature directions in head
concatenation and follows the centered spectrum through projections.

%% file: 010_body.tex
\section{Model and Assumptions}
\label{sec:setting}
\subsection{Multi-Head Self-Attention}
We consider multi-head self-attention with $h$ heads
(see Appendix~\ref{app:notation} for notation).
Let $X\in\R^{\ntoken\times d}$ contain the token representations as
rows; $\ntoken$ is the context length.
For each head $r\in\{1,\ldots,h\}$, let
$W_r^Q,W_r^K\in\R^{d\times\dimqk}$ and
$W_r^V\in\R^{d\times \dimvalue}$ be the query, key, and value weights,
and let $W^O\in\R^{h\dimvalue\times \dimout}$ be the output weight.
The projected matrices are $Q_r=XW_r^Q$, $K_r=XW_r^K$, and
$V_r=XW_r^V$. For inverse temperature $\beta>0$, define
\[
 S_r=\frac{Q_rK_r^\top}{\sqrt{\dimqk}},
 A_r(X)=\softmax(\beta S_r).
\]
The output concatenates the heads and applies $W^O$:
\[
 \MHA(X)=[A_1(X)V_1,\ldots,A_h(X)V_h]W^O.
\]
Softmax acts on each row, brackets denote horizontal concatenation, and
$\MHA(X)\in\R^{\ntoken\times \dimout}$.

\subsection{Assumptions}
\label{sec:gaussian-input}
Throughout, $X=X_\ntoken\in\R^{\ntoken\times d}$, $\ntoken\le d$, is
deterministic with $XX^\top/d=I_\ntoken$
\citep{naitsaada2025mind,hayase2026gaussian}.
The query weights $(W_r^Q)_r$ have jointly i.i.d.\ $\Normal(0,1/d)$ entries.
Each $W_r^K$ has i.i.d.\ $\Normal(0,1/d)$ entries, with arbitrary
dependence or sharing across heads.
Thus each $Q_r$ and $K_r$ has i.i.d.\ $\Normal(0,1)$ entries.

The value and output weights are random matrices. The joint family of
key, value, and output weights is independent of the query weights.
The laws of the value and output weights are otherwise arbitrary: they
may be non-Gaussian, dependent on the keys, and dependent or shared
across heads. Write $W^O=G_O/\sqrt{h\dimvalue}$. We assume
that the operator norms satisfy, almost surely,
\begin{equation}\label{eq:weight-norms}
 \max\{\|V_1\|,\ldots,\|V_h\|,\|G_O\|\}=O(\sqrt\ntoken).
\end{equation}

We work in the proportional regime. The head count $h$ and the inverse
temperature $\beta$ are fixed, and with $d=d(\ntoken)$ and
$\dimqk=\dimqk(\ntoken)$,
\begin{equation}\label{eq:aspect}
 \ntoken,d,\dimqk\longrightarrow\infty,
 \frac{h\dimqk}{\ntoken}\longrightarrow\keyratio\in(0,\infty).
\end{equation}
The value and output widths satisfy
\begin{equation}\label{eq:widths}
 \frac{h\dimvalue}{\ntoken}\longrightarrow\valueratio\in(0,\infty),
 \frac{\dimout}{\ntoken}\longrightarrow\outputratio\in(0,\infty).
\end{equation}
All limits are along \eqref{eq:aspect} and \eqref{eq:widths}.
All sizes are realized on a common probability space.

\section{Gaussian Equivalence}
\label{sec:equivalence}
\label{sec:main-theorem}
Write $A_r=A_r(X)$. Let $\mathbf1_\ntoken\in\R^\ntoken$ be the
all-ones vector and set $P=\mathbf1_\ntoken\mathbf1_\ntoken^\top/\ntoken$
and $\Ahead{r}=A_r-P$. The centered output is
\begin{equation}\label{eq:centered-mha}
 \MHAperp(X)=[\Ahead{1}V_1,\ldots,\Ahead{h}V_h]W^O.
\end{equation}
The removed component has rank at most one, since every block
$PV_r$ has identical rows. We study the squared singular values of
$\MHAperp(X)$.

\ifdefined\PaperArxivLayout
 \input{012_attention_structure}
\fi
The Gaussian model replaces each $\sqrt\ntoken\Ahead{r}$ by a linear
function of the scores plus independent noise.
Set $\theta_1=e^{\beta^2}-1$ and $\theta_2=\beta^2$.
Let $\noiseh{1},\ldots,\noiseh{h}\in\R^{\ntoken\times\ntoken}$ have jointly independent
$\Normal(0,1)$ entries, independent of all weights at the same size.
The Gaussian model for head $r$ is
\begin{equation}\label{eq:linear-model}
 \YGh{r}=\frac{\sqrt{\theta_2}S_r+
                    \sqrt{\theta_1-\theta_2}\noiseh{r}}{\sqrt{\ntoken}}.
\end{equation}
The resulting Gaussian layer is
\begin{equation}\label{eq:gaussian-layer}
 \MHAG(X)
   =[\YGh{1}V_1,\ldots,\YGh{h}V_h]W^O.
\end{equation}
\ifdefined\PaperArxivLayout\else
 \input{012_attention_structure}
\fi
\Cref{fig:self-attention-structure} illustrates this replacement
through the value and output projections.

The key Gram matrix is the conditional covariance of a score row.
For each head, put
\begin{equation}\label{eq:key-gram}
 \covKh{r}=\frac{K_rK_r^\top}{\dimqk},
 \covGh{r}=\theta_2\covKh{r}+(\theta_1-\theta_2)I_\ntoken,
\end{equation}
and define $\covK=\bigoplus_r\covKh{r}$ and
$\covG=\bigoplus_r\covGh{r}
=\theta_2\covK+(\theta_1-\theta_2)I_{h\ntoken}$.
With $\Vall=V_1\oplus\cdots\oplus V_h$, set
\begin{equation}\label{eq:population}
 \covVO=\frac1\ntoken(W^O)^\top\Vall^\top \covG\Vall W^O .
\end{equation}
Given the key, value, and output weights, the rows of $\MHAG(X)$
are i.i.d.\ $\Normal(0,\covVO)$.

For a positive semidefinite $N\times N$ matrix $C$, let $\mu_C$
assign weight $1/N$ to each eigenvalue. Write $\nu_M=\mu_{MM^\top}$
for the distribution of squared singular values, including zeros,
and $s_1(M)$ for the largest singular value.
Let $\freepois(\mu,t)$ be the compound free Poisson law with
jump law $\mu$, rate $t$, and free cumulants
$\kappa_q=t m_q(\mu)$, where $m_q(\mu)=\int x^q\,d\mu(x)$.

\begin{theorem}\label{thm:main}
Suppose that almost surely $\mu_{\covVO}$ converges weakly to a
deterministic law $\volaw$. Then, almost surely, weakly and in every
fixed moment,
\begin{equation}\label{eq:main-centered}
 \lim_{\ntoken\to\infty}\nu_{\MHAperp(X)}
 =\lim_{\ntoken\to\infty}\nu_{\ntoken^{-1/2}\MHAG(X)}
 =\freepois(\volaw,\outputratio).
\end{equation}
The uncentered output $\MHA(X)$ has the same weak
limit, and almost surely,
\begin{equation}\label{eq:main-leading}
 \frac{s_1(\MHA(X))}{\sqrt \dimout}
 =\frac{\|\mathbf1_\ntoken^\top[V_1,\ldots,V_h]W^O\|_2}
        {\sqrt{\ntoken \dimout}}+o(1).
\end{equation}
\end{theorem}

The bulk limit is proved in \cref{sec:main-proof};
\cref{app:leading-output} proves \eqref{eq:main-leading}.

The theorem reduces the bulk limit to a sample covariance model:
the softmax enters the limit through the two coefficients $\theta_1$
and $\theta_2$, and the weights enter through the spectrum of
$\covVO$. The leading singular value is determined by the
common-row component after the value and output projections.

\begin{table*}[t]
\centering
\small
\renewcommand{\arraystretch}{1.15}
\caption{Equivalent random matrix models, defined in
\cref{sec:transformations}. Each step preserves all fixed spectral
moments after right multiplication under \cref{ass:right-multipliers}. The third
column states the control used.}
\label{tab:transformations}
\begin{tabular}{@{}>{\raggedright\arraybackslash}p{.10\textwidth}
                  >{\raggedright\arraybackslash}p{.42\textwidth}
                  >{\raggedright\arraybackslash}p{.42\textwidth}@{}}
\toprule
Model & Transformation & Control used\\
\midrule
$\scaledA$ & Centered multi-head self-attention & Initial model\\
$\Yf$ & Replace the row denominators by a constant & Operator-norm error $\ntoken^{-1/2+o(1)}$; rank at most $h$\\
$\Yfc$ & Subtract the conditional row mean & Rank at most one; norms $\ntoken^{o(1)}$\\
$\YGSig$ & Gaussian rows with matching covariance & Quadratic-form concentration; self-averaging\\
$\YGC$ & Approximate $\covf$ by $\covG$ & Normalized Frobenius error $O(\sqrt{\log\ntoken/\ntoken})$\\
$\YG$ & Use the original queries and independent noise & Same conditional law; self-averaging\\
\bottomrule
\end{tabular}
\end{table*}

%% file: 012_attention_structure.tex
\ifdefined\PaperArxivLayout
 \begin{wrapfigure}{R}{0.50\textwidth}
 \small
\else
 \begin{figure}[t]
\fi
 \centering
 \ifdefined\PaperMonochrome
  \includegraphics[width=\ifdefined\PaperArxivLayout\linewidth\else7.98cm\fi]{figs/attention_structure_mono.pdf}
 \else
  \includegraphics[width=\ifdefined\PaperArxivLayout\linewidth\else7.98cm\fi]{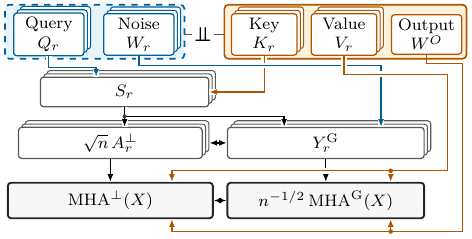}
 \fi
 \caption{Gaussian equivalence through concatenation and projections.
 The \ifdefined\PaperMonochrome dashed and solid\else blue dashed and orange solid\fi\
 families are independent of each other.}
 \label{fig:self-attention-structure}
\ifdefined\PaperArxivLayout
 \end{wrapfigure}
\else
 \end{figure}
\fi

%% file: 015_equivalent_models.tex
\section{Proof of Equivalence}
\label{sec:transformations}
\label{app:deterministic-comparisons}
We prove the bulk limit in \cref{thm:main} by comparing the centered
attention concatenation with its Gaussian counterpart.
Write $A=[A_1,\ldots,A_h]$ and define its centered concatenation
$\Aperp=[\Ahead{1},\ldots,\Ahead{h}]$, so that
\begin{equation}\label{eq:mha-factorization}
 \MHAperp(X)=\Aperp\Vall W^O.
\end{equation}
Write $\YG=[\YGh{1},\ldots,\YGh{h}]$ for the Gaussian concatenation.
This factorization separates the attention comparison from the value
and output projections.

\Cref{tab:transformations} summarizes the four steps from $\scaledA$ to $\YG$:
normalizer removal, centering, Gaussian replacement, and covariance identification.
Each step gives vanishing differences of all fixed spectral moments almost surely after
admissible right multiplication, carrying the comparison through the value
and output projections.

We use $m_q(\nu_M)=\tr_\ntoken[(MM^\top)^q]$ with normalized trace
$\tr_N=N^{-1}\Tr$.
By \eqref{eq:mha-factorization} and \eqref{eq:gaussian-layer},
the original and scaled Gaussian outputs are
$\scaledA T_\ntoken$ and $\YG T_\ntoken$ with
$T_\ntoken=\ntoken^{-1/2}\Vall W^O$, so each step is proved after right
multiplication by a matrix of the following kind.
\begin{assumption}[Conditions on right multipliers]\label{ass:right-multipliers}
The sequence $T_\ntoken\in\R^{h\ntoken\times {\dimout}_\ntoken}$, with
${\dimout}_\ntoken\ge1$ deterministic, satisfies $\|T_\ntoken\|=O(1)$ a.s.
At each size, the joint family $((K_{r,\ntoken})_{r=1}^h,T_\ntoken)$
is independent of the queries $(Q_{r,\ntoken})_{r=1}^h$.
\end{assumption}
Every comparison below uses Gaussian noise independent of the joint
family of queries, keys, and $T_\ntoken$ at each size.
\Cref{sec:main-proof} verifies that $\ntoken^{-1/2}\Vall W^O$
satisfies \cref{ass:right-multipliers}. Replacing $T_\ntoken$ by
$(T_\ntoken T_\ntoken^\top)^{1/2}$ preserves the left Gram matrices
and the multiplier norm. Until \cref{sec:main-proof}, we write
$T\in\R^{hn\times hn}$ for this positive semidefinite multiplier.

\subsection{Removing the Row Normalizers}
\label{sec:headwise-adaptations}
Put $P^\perp=I_\ntoken-P$. Row stochasticity gives $A_rP=P$ and
$\Ahead{r}=A_rP^\perp$. Define the row denominators
$Z_{r,i}=\sum_j e^{\beta S_{r,ij}}$ and the entrywise function
$\featfn(x)=e^{\beta x-\beta^2/2}-1$. For $\chi\sim\Normal(0,1)$,
$\theta_1=\E[\featfn(\chi)^2]$ and $\theta_2=(\E[\featfn'(\chi)])^2$.
The normalized denominators $Z_{r,i}/\ntoken$ concentrate around
$e^{\beta^2/2}$, yielding exponential feature matrices. Set
\begin{align}
 \Yfh{r}&=\frac{\featfn(S_r)}{\sqrt{\ntoken}},&
 D_r&=\frac{\diag(Z_{r,1},\ldots,Z_{r,\ntoken})}
                  {\ntoken e^{\beta^2/2}},\label{eq:Dh}
\end{align}
and put $\Yf=[\Yfh{1},\ldots,\Yfh{h}]$. Since $P^\perp\mathbf1_\ntoken=0$ and
$D_r\sqrt{\ntoken}A_r=e^{\beta S_r-\beta^2/2}/\sqrt{\ntoken}$, the row
normalizers are the only remaining difference:
\begin{equation}\label{eq:exact-normalizer}
 \Yfh{r}P^\perp=D_r\sqrt{\ntoken}\Ahead{r}.
\end{equation}
Write $\Pperpall=P^\perp\oplus\cdots\oplus P^\perp$.
\begin{lemma}[Row normalizers]\label{lem:normalizer}
For any $q\in\N$, $m_q(\nu_{\scaledA T})-m_q(\nu_{\Yf T})\to0$ a.s.
The operator norms satisfy
$\|\scaledA-\Yf \Pperpall\|\le\ntoken^{-1/2+o(1)}$ a.s.
\end{lemma}
\begin{proof}
The diagonal matrices $D_r$ satisfy
$\max_r\|D_r-I_\ntoken\|\le\ntoken^{-1/2+\varepsilon}$ eventually at
any fixed $\varepsilon>0$, by \cref{lem:normalizer-array}, the
rectangular-array form of the normalizer concentration in
\citet[Section 4.1]{hayase2026gaussian}; hence
$\max_r\|D_r^{-1}-I_\ntoken\|\le\ntoken^{-1/2+o(1)}$. With
$\|\Yf\|=\ntoken^{o(1)}$ (\eqref{eq:feature-norms}) and a countable
intersection over $\varepsilon=1/j$, the identity
\eqref{eq:exact-normalizer} gives the operator-norm bound through
$\|\sqrt{\ntoken}\Ahead{r}-\Yfh{r}P^\perp\|
\le\|D_r^{-1}-I_\ntoken\|\|\Yfh{r}\|$. The difference
$\Yf \Pperpall-\Yf$ has rank at most $h$, and all norms stay
$\ntoken^{o(1)}$ after right multiplication by $T$
(\cref{lem:attention-norm}), so \eqref{eq:op-moment-bound} and
\eqref{eq:rank-moment-bound} of \cref{app:moment-stability} give the
moment comparison.
\end{proof}

\subsection{Centering the Rows}
\label{sec:conditional-covariance}
The step of \cref{sec:joint-replacement} replaces rows by Gaussian rows
of the same conditional covariance, so we record that covariance and
subtract the conditional mean. Condition on
$\mathcal F_n=\sigma(K_{1,\ntoken},\ldots,K_{h,\ntoken},T_\ntoken)$.
The rows of $\Yf$ are then independent.
Each concatenated score row has law $\Normal(0,\covK)$, giving
\begin{equation}\label{eq:main-mean}
 \sqrt{\ntoken}\E[(\Yf)_{ij}\mid\mathcal F_n]
 =e^{\beta^2(\covKent{jj}-1)/2}-1,
\end{equation}
which is independent of $i$ and small because the key norms are nearly
equal. Write $\covf$ for the conditional covariance of a row of
$\sqrt{\ntoken}\Yf$. Its head blocks are $\covfh{r}$, with zero
cross-head blocks by query independence (\cref{fig:block-structure}).
Subtracting the conditional mean from each row gives $\Yfc$.
\begin{lemma}[Conditional centering]\label{lem:centering}
For any $q\in\N$, $m_q(\nu_{\Yf T})-m_q(\nu_{\Yfc T})\to0$ a.s.
\end{lemma}
\begin{proof}
All rows of $\Yf-\Yfc$ equal the conditional mean, so the difference
has rank one, and its entries are $O(\sqrt{\log\ntoken}/\ntoken)$ by
\eqref{eq:main-mean} and \cref{lem:gaussian-keys}, so its norm is
$O(\sqrt\ntoken\sqrt{\log\ntoken/\ntoken})=\ntoken^{o(1)}$. The rank estimate
\eqref{eq:rank-moment-bound} of \cref{app:moment-stability}, with
$\max\{\|\Yf\|,\|\Yfc\|\}=\ntoken^{o(1)}$ from \eqref{eq:feature-norms},
proves the comparison.
\end{proof}
\subsection{Concentration and Gaussian Comparison}
\label{sec:joint-replacement}
We construct Gaussian rows with the conditional covariance of the
centered feature rows, then apply the same right multiplier $T$. Let
$\auxnoise\in\R^{\ntoken\times h\ntoken}$ have i.i.d.\ $\Normal(0,1)$ entries,
independent of $\mathcal F_n$ and the queries, and set
\begin{equation}\label{eq:gaussian-auxiliary}
 \YGcov{\Lambda}=\frac{\auxnoise\Lambda^{1/2}}{\sqrt{\ntoken}}
\end{equation}
for $\mathcal F_n$-measurable $\Lambda\ge0$.
The rows of $\YGSig$ and $\Yfc$ have the same conditional covariance.
\begin{lemma}[Gaussian replacement]\label{lem:direct-replacement}
For any $q\in\N$,
\[
 m_q(\nu_{\Yfc T})-m_q(\nu_{\YGSig T})\longrightarrow0
 \quad\text{a.s.}
\]
\end{lemma}
\begin{proof}
Condition on $\mathcal F_n$. Write a row of $\Yfc T$ as
$y^\top/\sqrt n$ and put $\Gamma=\E[yy^\top]$.
Here $\|U\|_{L^s}=(\E\|U\|^s)^{1/s}$ is conditional, using the
absolute value, Euclidean norm, or operator norm as appropriate.
To bound the query derivatives, fix $M\ge1$ and work on $\Omega_n(M)$,
defined by
\begin{equation}\label{eq:replacement-regularity}
\begin{gathered}
 \|T\|\le M,\qquad \max_r\|\covKh{r}\|\le M,\\
 \max_{r,i,j}|(\covKh{r}-I_n)_{ij}|\le M\sqrt{\log n/n}.
\end{gathered}
\end{equation}
Gaussian keys and the multiplier assumption give eventual membership
in some $\Omega_n(M)$ a.s. (\cref{lem:gaussian-keys}).

By \cref{lem:feature-reduction}, it suffices, for each fixed $s\ge1$,
to prove $\|{y^\top Hy-\Tr(\Gamma H)}\|_{L^s}\le n^{1/2+o(1)}\|H\|$
uniformly on $\Omega_n(M)$ and over deterministic symmetric
$H\in\R^{hn\times hn}$.
Let $u=(u_r)_{r=1}^h$ collect independent queries $u_r\sim\Normal(0,I_{\dimqk})$
and put $J_y=D_u y$.
The Maurey--Pisier inequality \eqref{eq:Gaussian-Lp} and H\"older give
\begin{equation}\label{eq:quadratic-derivative}
 \|y^\top Hy-\Tr(\Gamma H)\|_{L^s}
 \le 2C_s\|H\|\|J_y\|_{L^{2s}}\|y\|_{L^{2s}}.
\end{equation}

Put $z_r=K_r u_r/\sqrt{\dimqk}$ and $L=\max_{r,j}|z_{rj}|$.
The conditional mean does not depend on $u$, so
\begin{equation}\label{eq:jacobian}
 J_y=T\left(\bigoplus_{r=1}^h
 \left[\diag(\featfn'(z_r))\frac{K_r}{\sqrt{\dimqk}}\right]\right),
\end{equation}
where $\featfn'$ acts entrywise. Since
$\|T\|\max_r\|K_r/\sqrt{\dimqk}\|\le M^{3/2}$,
\begin{equation}\label{eq:jacobian-bound}
 \|J_y\|\le M^{3/2}\max_{r,j}|\featfn'(z_{rj})|
 \le M^{3/2}\beta e^{\beta L-\beta^2/2}.
\end{equation}
Each score has variance at most $M$, so \cref{lem:gaussian-maximum}
with $N=hn$ yields
\begin{equation}\label{eq:jacobian-moment}
 \|J_y\|_{L^{2s}}
 \le e^{c\sqrt{\log n}}=n^{o(1)},
\end{equation}
where $c>0$ depends only on $s,M,\beta,h$.
Together with $\|y\|_{L^{2s}}=O(\sqrt n)$ from uniform coordinate moments,
\eqref{eq:quadratic-derivative} gives the required concentration bound.
\end{proof}

\subsection{Identifying the Conditional Covariance}
\label{sec:identification}
The diagonal entries give the variances of the feature coordinates;
the off-diagonal entries describe correlations inherited from the keys.
\begin{lemma}[Covariance identification]\label{lem:covariance-replacement}
For every positive integer $q$,
\[
 m_q(\nu_{\YGSig T})-m_q(\nu_{\YG T})\longrightarrow0
 \quad\text{a.s.}
\]
\end{lemma}
\begin{proof}
Condition on $\mathcal F_n$ and work on $\Omega_n(M)$ from
\cref{sec:joint-replacement} for fixed $M$.
Write $\rho_{ij}=(\covK)_{ij}$ and
$\delta_n=\max_{i,j}|(\covK-I_{hn})_{ij}|$.
Gaussian exponential moments give
\begin{equation}\label{eq:main-covariance}
 (\covf)_{ij}
 =e^{\beta^2(\rho_{ii}+\rho_{jj}-2)/2}(e^{\beta^2\rho_{ij}}-1).
\end{equation}
On the diagonal, $\rho_{ii}=1+O(\delta_n)$, so
$(\covf)_{ii}=\theta_1+O(\delta_n)$.
Thus $\theta_1=e^{\beta^2}-1$ is the limiting variance of each
centered feature coordinate.

Off the diagonal, $|\rho_{ij}|\le\delta_n$.
Taylor expansion gives
$e^{\beta^2\rho_{ij}}-1=\beta^2\rho_{ij}+O(\rho_{ij}^2)$,
and the prefactor in \eqref{eq:main-covariance} is $1+O(\delta_n)$.
Since $\rho_{ij}^2\le\delta_n|\rho_{ij}|$, we obtain
$(\covf)_{ij}=\theta_2\rho_{ij}+O(\delta_n|\rho_{ij}|)$.
Thus $\theta_2=\beta^2$ is the coefficient of the leading
correlations between distinct coordinates.

The term $\theta_2\covK$ contributes $\theta_2+O(\delta_n)$
on the diagonal. Adding the remaining variance gives
\[
 \covG=\theta_2\covK+(\theta_1-\theta_2)I_{hn}.
\]
\Cref{lem:entrywise-covariance-comparison} turns these uniform entrywise estimates
into the claimed moment comparison.
In $\YG$, the linear scores carry the leading correlations,
while independent noise supplies the remaining variance.
\end{proof}

\subsection{Conclusion of the Proof}
\label{sec:main-proof}
The four comparisons combine into the following lemma.
\begin{lemma}[Gaussian equivalence under right multiplication]\label{lem:replacement}
Under \cref{ass:right-multipliers}, with the comparison noise chosen
as above, for every positive integer $q$,
\[
 m_q(\nu_{\scaledA T_\ntoken})-m_q(\nu_{\YG T_\ntoken})\to0
 \text{ a.s.}
\]
\end{lemma}
\begin{proof}
After the square-root reduction,
\cref{lem:normalizer,lem:centering,lem:direct-replacement,lem:covariance-replacement}
compare the consecutive models of \cref{tab:transformations}, and a
countable intersection handles the fixed orders.
\end{proof}
The Gaussian side has bounded norm, so the moment comparison implies
weak and resolvent comparison (\cref{app:weak-comparison}).

\begin{proof}[Proof of \eqref{eq:main-centered}]
The weight assumptions and \eqref{eq:widths} ensure that
$T_\ntoken=\ntoken^{-1/2}\Vall W^O$ satisfies
\cref{ass:right-multipliers}. The two outputs in
\eqref{eq:main-centered} are $\scaledA T_\ntoken$ and $\YG T_\ntoken$.

Conditionally on the key, value, and output weights, $\YG T_\ntoken$
has independent centered Gaussian rows with covariance $\covVO/\ntoken$.
Moreover, $\covVO=T_\ntoken^\top\covG T_\ntoken$ has bounded norm
almost surely by \cref{lem:gaussian-keys} and $\|T_\ntoken\|=O(1)$.
The assumed convergence of $\mu_{\covVO}$ and
\cref{lem:gaussian-law}, with $D_m=\covVO$ and
$\dimout/\ntoken\to\outputratio$, give
$\nu_{\YG T_\ntoken}\to\freepois(\volaw,\outputratio)$ almost surely,
weakly and in every fixed moment.

\Cref{lem:replacement} transfers the moment limits to $\MHAperp(X)$.
Since the Gaussian limit has compact support, \cref{app:moment-to-weak}
also gives weak convergence, proving \eqref{eq:main-centered}.
Finally, $\MHA(X)-\MHAperp(X)$ has rank at most one, so
$\MHA(X)$ has the same weak limit by the rank inequality.
\end{proof}

%% file: 020_simulation.tex
\ifdefined\PaperArxivLayout
 \Needspace{24\baselineskip}
\fi
\section{Spectra and Learning}
\label{sec:consequences}
\label{sec:projections}
We use the Gaussian equivalent in \cref{thm:main} to study coefficient
estimation and readout training.

\subsection{Spectra and Readout Predictions}
\label{sec:readout-setting}
\ifdefined\PaperArxivLayout
 \input{027_public_model_ge_summary}
\fi
Finite GE closely matches the centered softmax spectrum at dimensions
taken from publicly available models.
The configurations in \cref{tab:public-finite-ge-summary} preserve
these models' projection widths, head counts, and K/V sharing,
while using $\ntoken\le d$, orthogonal inputs, Gaussian projection
weights, and $\beta=1$.
We compare all eigenvalues of the centered softmax and finite GE
output Gram matrices in \cref{thm:main}.
The relative 1-Wasserstein distance is the mean absolute difference
between ordered eigenvalues divided by their common limiting mean.
Across the seven configurations, including across-head and within-head
sharing, the mean relative distances range from $0.24\%$ to $2.82\%$.

\ifdefined\PaperArxivLayout\else
 \input{027_public_model_ge_summary}
 \input{023_readout_overview_figure}
\fi

The limiting spectrum also gives predictions for linear readouts.
Under the assumptions of \cref{thm:main}, we freeze the attention
weights and use $F=\MHA(X)/\sqrt m$, where
$m:=\lim_{\ntoken\to\infty}\ntoken^{-1}\|\MHAperp(X)\|_F^2>0$.

For coefficient estimation, generate
$y=\sqrt{\ntoken}F\vartheta+\varepsilon$ with independent
$\vartheta\sim\Normal(0,I_{\dimout}/\dimout)$ and
$\varepsilon\sim\Normal(0,\sigma^2I_\ntoken)$, also independent of $F$.
For fixed $\lambda>0$, let $\vartheta_*$ be the minimizer of
$\|Fv-y/\sqrt{\ntoken}\|^2+\lambda\|v\|^2$ over $v$.
The limiting mean squared coefficient error is
$\mathcal{E}_\lambda:=\lim_{\ntoken\to\infty}\mathbb E_{\vartheta,\varepsilon}
[\|\vartheta_*-\vartheta\|^2\mid F]$.
Writing $\nu$ for the limiting eigenvalue law of $FF^\top$, the
ridge formula and GE give (\cref{app:readout-observables})
\begin{equation}\label{eq:readout-risk-mha}
 \mathcal{E}_\lambda=1-\outputratio^{-1}
 +\int\frac{\outputratio^{-1}\lambda^2+\sigma^2x}{(x+\lambda)^2}\,d\nu(x).
\end{equation}

To evaluate this integral, define
$w_\nu(\lambda):=\int\frac{x}{x+\lambda}\,d\nu(x)$,
the limiting effective degrees of freedom per token.
For $\lambda>0$, $w=w_\nu(\lambda)$ satisfies
$\lambda w=(1-w)/\Stransform_\nu(-w)$ and determines the error by
\begin{equation}\label{eq:readout-scalar-risk}
 \mathcal{E}_\lambda=1-\frac{w_\nu(\lambda)}{\outputratio}
 +\left(\frac{\lambda}{\outputratio}-\sigma^2\right)w_\nu'(\lambda).
\end{equation}
Here the prime denotes differentiation in $\lambda$;
\cref{app:readout-computation} gives the derivation.

The same law predicts the training-loss trajectory of a linear readout.
For a nonzero centered target $y$ fixed independently of initialization,
unregularized gradient flow from $\vartheta(0)=0$ minimizes
$\frac12\|F\vartheta-y\|^2$.
Its relative loss $L_\ntoken(t;y)=\|F\vartheta(t)-y\|^2/\|y\|^2$
satisfies (\cref{app:readout-observables})
\begin{equation}\label{eq:learning-curve-bulk}
 \lim_{\ntoken\to\infty}\E L_\ntoken(t;y)
 =\int e^{-2tx}\,d\nu(x).
\end{equation}
We evaluate this prediction by numerically inverting the time-Laplace
transform $(1-w_\nu(s/2))/s$, $s>0$
(\cref{app:readout-computation}).

For independent Gaussian K/V and an independent Gaussian output
projection as in \cref{cor:vo}, these predictions reduce to an explicit
scalar equation. Here $m=\outputratio\theta_1$ and
$c_\beta:=\theta_2/\theta_1=\beta^2/(e^{\beta^2}-1)$.
For $t>0$, let $\rho_{\beta,t}$ be the law of $1-c_\beta+c_\beta x$
for $x\sim\MP_{1/t}$, with
\[
 \Stransform_{\rho_{\beta,t}}(z)=\frac{2}{1+c_\beta z/t+
 \sqrt{(1-c_\beta z/t)^2+4c_\beta^2z/t}}.
\]
Then $w_\nu(\lambda)$ is the unique solution of
\begin{equation}\label{eq:readout-ridge-root}
 \lambda w=\frac{(1-w)(1-w/h)(1-w/\valueratio)(1-w/\outputratio)}
 {\Stransform_{\rho_{\beta,\keyratio}}(-w)}
\end{equation}
in $0<w<\min\{1,h,\valueratio,\outputratio\}$.
Solving this equation and differentiating implicitly evaluates
\eqref{eq:readout-scalar-risk} without generating feature matrices.

\ifdefined\PaperArxivLayout
 \input{023_readout_overview_figure}
\fi
\Cref{fig:readout-overview}(a) shows agreement with the softmax
learning curves as increasing head count changes a nonzero plateau
into slow and then rapid decay: mass at zero determines the plateau,
while small positive eigenvalues govern slow decay.

Unless varied, the readout experiments use $\beta=1$, with
$\lambda=10^{-4}$ and $\sigma^2=0.1$ for coefficient estimation;
\cref{tab:learning-dimensions} lists the dimensions.
Curves use the limiting law. Markers are computed from exact
finite-matrix ridge errors or gradient-flow losses averaged over ten
softmax initializations (\cref{app:experiments}); arrival times are
obtained from the averaged losses.

\subsection{Width Allocation and Temperature}
\label{sec:head-count-width}
\label{sec:temperature-width}

The total value width and output width constrain the available
feature directions through $h\dimvalue/\ntoken$ and
$\dimout/\ntoken$. Near boundaries where the smallest width changes,
small nonzero Gram eigenvalues amplify observation noise in
coefficient estimation. This mechanism is shared with analyses of
interpolation and multiple descent in random-feature regression
\citep{dascoli2020triple,meng2024multiple}.
GE makes its consequences quantitatively accessible for softmax MHA
through the spectral law and coefficient-error formula in
\cref{sec:readout-setting}.

\Cref{fig:readout-overview}(b,c) compares two width-allocation paths.
With fixed per-head widths and $\dimout=\ntoken/2$, the error has
one peak near $h\dimvalue=\dimout$.
Holding the number $\nparam$ of distinct Q/K/V/O projection weights
fixed instead makes total value width increase while output width
decreases, producing two peaks near $h\dimvalue=\ntoken$ and
$\dimout=\ntoken$.
Thus, at a fixed parameter count, width allocation determines which
projection limits estimation accuracy.

The Q/K--V allocation depends on inverse temperature through the linear
and nonlinear covariance weights $c_\beta$ and $1-c_\beta$.
At small $\beta$, narrow Q/K projections leave weak feature directions;
the nonlinear contribution strengthens these directions as $\beta$ increases.
\Cref{fig:readout-overview}(d) compares the coefficient error of
$(\dimqk,\dimvalue)=(64,192)$ relative to $(128,128)$ at equal parameter counts
(\cref{app:temperature-width-experiment}).
GE predicts the observed reversal in their ranking.
For coefficient estimation near $\beta=1$, allocate more width to V
than to Q/K at a fixed projection budget.

\subsection{Parameter Sharing}
\label{sec:across-head-sharing}
\label{sec:within-head-sharing}
\ifdefined\PaperArxivLayout
 \input{024_sharing_comparison_figure}
\fi
Across-head K/V sharing reduces the parameter count without
changing the limiting spectrum at fixed head count and widths.
GQA shares a key projection and a separate
value projection among several query heads \citep{ainslie2023gqa}.
Under the Gaussian projection assumptions of \cref{cor:vo},
this sharing preserves $\nu$ and \eqref{eq:readout-ridge-root}, hence
both the coefficient error and learning curve at fixed widths.
\Cref{fig:sharing-comparison}(a) fixes $h=64$ and compares this reduction with reallocating
the saved weights to wider heads at fixed $\nparam$
(\cref{app:cross-head-learning-experiment}).
The latter reduces the time to relative loss $0.1$ in this example,
so sharing can support either parameter savings or faster readout training
through increased width.

We next compare GQA with two query heads per K/V pair and complete
within-head tying, $W_r^V=W_r^K$ in every head.
When $\dimqk=\dimvalue$, these schemes have equal parameter counts
at the same head count and widths.

Tying makes the covariance a quadratic function of the key Gram:
put $\kvmap(x)=(\theta_1-\theta_2)x+(h\theta_2/\keyratio)x^2$ and
let $\kvlaw$ be the law of $\kvmap(x)$ for
$x\sim\freepois_{\keyratio/h}$.
With the same independent Gaussian output projection,
$m=\outputratio(\theta_1+h\theta_2/\keyratio)$, and
\eqref{eq:readout-ridge-root} becomes
\begin{equation}\label{eq:readout-tied-root}
 \lambda w=\frac{(1-w)(h-w)(\outputratio-w)}
 {\keyratio m\Stransform_{\kvlaw}(-w/h)}.
\end{equation}
Its solution in $0<w<\min\{1,\keyratio,\outputratio\}$ gives the
coefficient error through the same \eqref{eq:readout-scalar-risk}
(\cref{app:shared-values-output}).

\Cref{fig:sharing-comparison}(b) shows that complete K/V tying becomes
preferable as a fixed total width is divided among more heads.
GQA has lower coefficient error at $h=2,4$, whereas complete tying
has lower error at $h=8,16,32$.
Equation~\eqref{eq:readout-tied-root} predicts this reversal.
Near $\beta=1$, these results suggest favoring within-head K/V tying
over pairwise GQA for coefficient estimation when a fixed total width
is divided among heads narrow relative to the token length.

\ifdefined\PaperArxivLayout\else
 \input{024_sharing_comparison_figure}
\fi

%% file: 027_public_model_ge_summary.tex
% Generated by mha-paper render from public experiment summaries.
\ifdefined\PaperArxivLayout
 \begin{wraptable}{R}{0.53\textwidth}
\else
 \begin{table}[t]
\fi
\centering\footnotesize
\setlength{\belowcaptionskip}{6pt}
\caption{Finite-size MHA--GE spectral agreement at dimensions
taken from publicly available models. Means with 95\% CIs.
Full dimensions and sources: \cref{app:public-finite-ge}.}
\label{tab:public-finite-ge-summary}
\begingroup\setlength{\tabcolsep}{3pt}
\begin{tabular}{@{}llrrr@{}}\toprule
Dimensions of & Sharing & $\ntoken$ & $h$ & Rel. $W_1$ (\%)\\\midrule
BERT-base & None & 512 & 12 & $1.66 \pm 0.23$ \\
ModernBERT-L & None & 1,024 & 16 & $1.37 \pm 0.21$ \\
DINOv3-L/16 & None & 261 & 16 & $2.82 \pm 0.40$ \\
LLaDA-8B & None & 4,096 & 32 & $0.92 \pm 0.06$ \\
GPT-3 175B & None & 2,048 & 96 & $0.62 \pm 0.07$ \\
Falcon-7B & MQA & 2,048 & 71 & $1.92 \pm 0.46$ \\
Kayyam 1.2B & $K=V$ & 2,048 & 32 & $0.24 \pm 0.01$ \\
\bottomrule\end{tabular}\endgroup
\ifdefined\PaperArxivLayout
 \end{wraptable}
\else
 \end{table}
\fi

%% file: 023_readout_overview_figure.tex
% Regenerate with mha-paper render.
\ifdefined\PaperArxivLayout
 \begin{figure}[t]
\else
 \begin{figure*}[t]
\fi
 \centering
 \includegraphics[width=\textwidth]{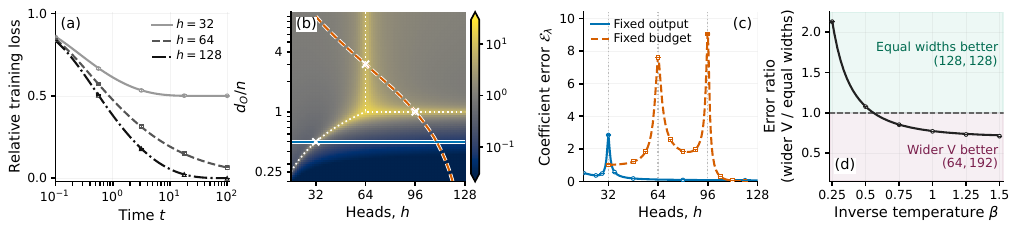}
 \caption{Spectral predictions for readout learning and width allocation.
 (a) Training loss for three head counts.
 (b) Coefficient-error map with fixed-output (blue) and fixed-budget
 (orange) paths; dotted lines mark width boundaries.
 (c) Errors along these paths.
 (d) Coefficient-error ratio, $(64,192)/(128,128)$, at fixed $\nparam$;
 above the dashed unit line, equal widths give lower error;
 below it, wider V gives lower error.
 Curves: GE; markers: MHA; bars: 95\% CIs.
 Maximum error bars: (a) $0.00513$; (c) $0.02381$; (d) $0.00168$.
 Details: \cref{app:head-learning-experiment,app:head-budget-experiment,app:temperature-width-experiment}.}
 \label{fig:readout-overview}
\ifdefined\PaperArxivLayout
 \end{figure}
\else
 \end{figure*}
\fi

%% file: 024_sharing_comparison_figure.tex
% Regenerate with mha-paper render.
\ifdefined\PaperArxivLayout
 \begin{wrapfigure}[17]{R}{0.56\textwidth}
 \small
\else
 \begin{figure}[t]
\fi
 \centering
 \includegraphics[width=\ifdefined\PaperArxivLayout\linewidth\else3.25in\fi]{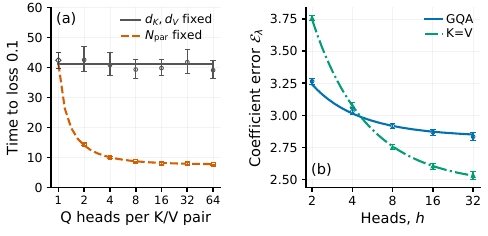}
 \caption{Parameter sharing.
 (a) Across-head sharing; diamond: unshared.
 (b) GQA versus within-head K=V at equal widths and parameter count,
 $h\dimqk=h\dimvalue=\dimout=\ntoken/2$.
 Lines: GE; markers: MHA; bars: 95\% CIs.
 Maximum error bars: (a) $4.43476$; (b) $0.03015$.
 Settings in \cref{app:cross-head-learning-experiment,app:sharing-head-count-experiment}.}
 \label{fig:sharing-comparison}
\ifdefined\PaperArxivLayout
 \end{wrapfigure}
\else
 \end{figure}
\fi

%% file: 070_discussion.tex
\section{Discussion}
\label{sec:discussion}
Our theory yields concrete recommendations for width allocation and
parameter sharing in frozen-feature readout models.
By predicting training loss and coefficient error from architectural
parameters, it identifies how a fixed parameter budget can be used
more effectively.

\subsection{Implications for Learning}
\label{sec:discussion-learning}
For frozen-feature readouts, width allocation should address V/O
bottlenecks and the temperature-dependent role of Q/K.
The ratios $h\dimvalue/\ntoken$ and $\dimout/\ntoken$ constrain the
available feature directions: increasing the head count cannot
eliminate a loss plateau caused by insufficient output width.
The preferred Q/K--V allocation depends on inverse temperature:
at low $\beta$, preserving Q/K width avoids weak feature directions,
while the nonlinear contribution reduces this disadvantage as
$\beta$ increases (\cref{fig:readout-overview}(d)).
For coefficient estimation near $\beta=1$, use wider V than Q/K
at a fixed projection budget.

For frozen-feature readouts, choose the sharing mechanism according
to the design objective. Use across-head K/V sharing to reduce
parameters while maintaining training and estimation performance at
fixed widths, or reinvest the savings in wider heads to accelerate
readout training. When coefficient estimation at the same width and
parameter budget is the priority, favor within-head K/V tying for
narrow heads near $\beta=1$.

\subsection{Extensions}
\label{sec:discussion-extensions}
For general inputs, the spectrum of the input Gram matrix may
interact with the width constraints and K/V dependence.
Extending GE to this setting would determine how these input
correlations alter the preferred width allocations and sharing schemes.

A second question is how fast $\beta$ may grow with $\ntoken$ while
preserving GE.
\citet[Section 6.1]{hayase2026gaussian} suggest
$\beta=\Theta(\sqrt{\log\ntoken})$ as a candidate scale at which
bulk growth changes the separation from the dominant common mode.
Determining the admissible growth rates would locate the boundary
of the Gaussian comparison as this spectral separation changes.

%% file: 072_conclusion.tex
\section{Conclusion}
\label{sec:conclusion}
We establish Gaussian equivalence for multi-head self-attention,
including key-dependent value and output projections.
The resulting random-matrix model predicts readout training dynamics
and coefficient estimation error, yielding design guidance for
frozen-feature learning.
More broadly, the Gaussian equivalence provides an analytical
foundation for understanding how multi-head attention shapes the
representations on which subsequent learning and computation depend.

%% file: 090_appendix.tex
\numberwithin{figure}{section}
\numberwithin{table}{section}

\ifdefined\PaperArxivLayout\else
The appendix collects notation, the detailed equivalence proof, and
supporting results for \cref{sec:consequences}.
\fi
\input{097_notation}
\input{091_replacement_proof}

\input{116_architectural_support}

\input{098_experiments}
\FloatBarrier

%% file: 097_notation.tex
\section{Notation}
\label{app:notation}
\Cref{tab:notation-common} collects the general mathematical notation.
\Cref{tab:notation-parameters,tab:notation-matrices,tab:notation-functions,tab:notation-laws}
list the dimensions and parameters, matrices, functions, and probability laws,
respectively. Function symbols omit evaluation arguments; matrix and law
entries denote the objects shown.

A subscript $\oplus$ marks a matrix defined as a direct sum over heads.
The aggregate covariances retain their role subscripts:
$\covK=\bigoplus_r\covKh{r}$,
$\covf=\bigoplus_r\covfh{r}$, and $\covG=\bigoplus_r\covGh{r}$
are $h\ntoken\times h\ntoken$ block matrices.

Scalar functions of score matrices act entrywise.

\begingroup
\small
\renewcommand{\arraystretch}{1.15}
\setlength{\LTleft}{0pt}
\setlength{\LTright}{0pt}
\begin{longtable}{@{}>{\raggedright\arraybackslash}p{.20\textwidth}
                    >{\raggedright\arraybackslash}p{.29\textwidth}
                    >{\raggedright\arraybackslash}p{\dimexpr.51\textwidth-4\tabcolsep\relax}@{}}
\caption{General mathematical notation.}
\label{tab:notation-common}\\
\toprule
Symbol & Name & Definition\\
\midrule
\endfirsthead
\multicolumn{3}{@{}l}{\tablename\ \thetable\ (continued)}\\
\toprule
Symbol & Name & Definition\\
\midrule
\endhead
\midrule
\multicolumn{3}{r@{}}{Continued on the next page}\\
\endfoot
\bottomrule
\endlastfoot
\multicolumn{3}{@{}l}{\textbf{Sets and linear algebra}}\\*
\addlinespace[2pt]
$\N$ & Natural numbers & $\{1,2,\ldots\}$; zero is excluded.\\
$\R$, $\C$ & Real and complex numbers & The real and complex number fields.\\
$\C_+$ & Upper half-plane & $\{z\in\C:\operatorname{Im}z>0\}$.\\
$I_N$ & Identity matrix & $N\times N$ identity.\\
$\mathbf1_N$ & All-ones vector & $(1,\ldots,1)^\top\in\R^N$.\\
$M^\top$ & Transpose & $(M^\top)_{ij}=M_{ji}$.\\
$\Tr$ & Trace & $\Tr C=\sum_{j=1}^N C_{jj}$ for an $N\times N$ matrix.\\
$\tr_N$ & Normalized trace & $\tr_N C=N^{-1}\Tr C$.\\
$\rank$ & Matrix rank & $\rank M$ is the dimension of the image of $M$.\\
$\diag$ & Diagonal matrix construction & $\diag(v)$ has diagonal $v$; also written $\diag(v_1,\ldots,v_m)$.\\
$s_j$ & Singular-value function & $s_j(M)$ is the $j$th largest singular value of $M$, with zeros included as needed.\\
$A\ge B$, $B\le A$ & Positive semidefinite order & For real symmetric matrices, $v^\top(A-B)v\ge0$ for every real $v$; $A\ge0$ means positive semidefiniteness.\\
$\oplus$, $\bigoplus$ & Matrix direct sum & Block diagonal assembly, allowing rectangular blocks. A sum over $r$ runs from $1$ to $h$ unless specified otherwise.\\
$[M_1,\ldots,M_h]$ & Horizontal concatenation & Matrices joined by columns, without normalization.\\
$\|\cdot\|$ & Euclidean or operator norm & $\|v\|=\|v\|_2=(\sum_i|v_i|^2)^{1/2}$ for vectors; $\|M\|=\sup_{\|v\|=1}\|Mv\|$ for matrices.\\
$\|\cdot\|_F$ & Frobenius norm & $\|M\|_F=(\sum_{i,j}|M_{ij}|^2)^{1/2}$.\\
\addlinespace[5pt]
\multicolumn{3}{@{}l}{\textbf{Probability and measure operations}}\\*
\addlinespace[2pt]
$\E$ & Expectation & $\E Z=\int Z\,d\Prob$; conditioning and the variables averaged over are specified where needed.\\
$\Prob$ & Probability & $\Prob(E)$ is the probability of an event $E$; $\Prob(E\mid\mathcal F)$ denotes conditional probability.\\
$\Var$ & Variance & $\Var(Z)=\E[(Z-\E Z)^2]$ for real $Z$; conditional versions are indicated locally.\\
$m_q$ & Moment function & $m_q(\mu)=\int x^q\,d\mu(x)$, when finite.\\
$\kappa_q$ & $q$th free cumulant & $\kappa_q(\nu)$ is the coefficient of $z^{q-1}$ in $R_\nu(z)$.\\
$\mathbf1_E$ & Event indicator & One on $E$ and zero otherwise.\\
$\sigma(Z_1,\ldots,Z_k)$ & Generated sigma-field & Smallest sigma-field making the listed random objects measurable.\\
$Z\sim\mu$ & Distribution of a random variable & $Z$ has law $\mu$.\\
$\operatorname{supp}$ & Support & $\operatorname{supp}\mu$ is the smallest closed set of full $\mu$-measure.\\
$T_\#$ & Pushforward & $\pushforward{T}{\mu}$ is the law of $T(Z)$ when $Z\sim\mu$.\\
$\mathcal D_t$ & Dilation & $\mathcal D_t\mu=(x\mapsto tx)_\#\mu$, $t>0$.\\
$\boxtimes$ & Multiplicative free convolution & Free multiplicative convolution of probability laws on $[0,\infty)$.\\
\addlinespace[5pt]
\multicolumn{3}{@{}l}{\textbf{Convergence and asymptotics}}\\*
\addlinespace[2pt]
a.s. & Almost surely & On an event of probability one; \cref{app:convergence}. Asymptotic bounds then hold pathwise, with constants that may depend on the outcome.\\
i.i.d. & Independent and identically distributed & Independence and a common distribution for the stated family.\\
$\mu_n\Rightarrow\mu$ & Weak convergence & Convergence against every bounded continuous function; \cref{app:convergence}.\\
$a_n=O(b_n)$ & Bounded asymptotic ratio & $|a_n|\le Cb_n$ eventually, with $C$ independent of $n$ and $b_n>0$.\\
$a_n=o(b_n)$ & Vanishing asymptotic ratio & $a_n/b_n\to0$, for $b_n>0$.\\
$a_n\asymp b_n$ & Comparable positive sequences & $0<c\le a_n/b_n\le C<\infty$ eventually.\\
$a_n\sim b_n$ & Asymptotic equivalence & $a_n/b_n\to1$, with $b_n\ne0$ eventually.\\
\end{longtable}
\endgroup

\begingroup
\small
\renewcommand{\arraystretch}{1.15}
\setlength{\LTleft}{0pt}
\setlength{\LTright}{0pt}
\begin{longtable}{@{}>{\raggedright\arraybackslash}p{.16\textwidth}
                    >{\raggedright\arraybackslash}p{.33\textwidth}
                    >{\raggedright\arraybackslash}p{\dimexpr.51\textwidth-4\tabcolsep\relax}@{}}
\caption{Dimensions and parameters.}
\label{tab:notation-parameters}\\
\toprule
Symbol & Name & Definition or range\\
\midrule
\endfirsthead
\multicolumn{3}{@{}l}{\tablename\ \thetable\ (continued)}\\
\toprule
Symbol & Name & Definition or range\\
\midrule
\endhead
\midrule
\multicolumn{3}{r@{}}{Continued on the next page}\\
\endfoot
\bottomrule
\endlastfoot
$\ntoken$ & Context length & $\ntoken\in\N$.\\
$d$ & Input width & $d\in\N$, $\ntoken\le d$.\\
$\dimqk$ & Query/key width per head & $\dimqk\in\N$.\\
$\dimvalue$ & Value width per head & $\dimvalue\in\N$.\\
$\dimout$ & Output width & $\dimout\in\N$.\\
$h$ & Number of heads & $h\in\N$.\\
$\nparam$ & Projection parameter count & Number of distinct Q/K/V/O projection parameters.\\
$\beta$ & Inverse temperature & $\beta>0$.\\
$\keyratio$ & Total Q/K width ratio & \eqref{eq:aspect}.\\
$\valueratio$ & Total value width ratio & \eqref{eq:widths}.\\
$\outputratio$ & Output width ratio & \eqref{eq:widths}.\\
$\theta_1$ & Feature variance & $\theta_1=e^{\beta^2}-1$.\\
$\theta_2$ & Linear covariance coefficient & $\theta_2=\beta^2$.\\
$c_\beta$ & Linear variance fraction & $c_\beta=\theta_2/\theta_1\in(0,1)$.\\
\end{longtable}
\endgroup

The matrix definitions use the conditioning sigma-field
$\mathcal F_n=\sigma(K_{1,\ntoken},\ldots,K_{h,\ntoken},T_\ntoken)$.

\begingroup
\small
\renewcommand{\arraystretch}{1.15}
\setlength{\LTleft}{0pt}
\setlength{\LTright}{0pt}
\begin{longtable}{@{}>{\raggedright\arraybackslash}p{.17\textwidth}
                    >{\raggedright\arraybackslash}p{.16\textwidth}
                    >{\raggedright\arraybackslash}p{.20\textwidth}
                    >{\raggedright\arraybackslash}p{\dimexpr.47\textwidth-6\tabcolsep\relax}@{}}
\caption{Matrices and their dimensions.}
\label{tab:notation-matrices}\\
\toprule
Symbol & Size & Name & Definition\\
\midrule
\endfirsthead
\multicolumn{4}{@{}l}{\tablename\ \thetable\ (continued)}\\
\toprule
Symbol & Size & Name & Definition\\
\midrule
\endhead
\midrule
\multicolumn{4}{r@{}}{Continued on the next page}\\
\endfoot
\bottomrule
\endlastfoot
\multicolumn{4}{@{}l}{\textbf{Input and projection matrices}}\\*
\addlinespace[2pt]
$X$ & $\ntoken\times d$ & Input matrix & $XX^\top/d=I_\ntoken$; \cref{sec:gaussian-input}.\\
$W_r^Q$ & $d\times\dimqk$ & Query weights & \Cref{sec:gaussian-input}.\\
$W_r^K$ & $d\times\dimqk$ & Key weights & \Cref{sec:gaussian-input}.\\
$W_r^V$ & $d\times\dimvalue$ & Value weights & \Cref{sec:gaussian-input}.\\
$W^O$ & $h\dimvalue\times\dimout$ & Output weights & \Cref{sec:gaussian-input}.\\
$G_O$ & $h\dimvalue\times\dimout$ & Rescaled output weights & $\sqrt{h\dimvalue}W^O$.\\
$Q_r$ & $\ntoken\times\dimqk$ & Query matrix & $XW_r^Q$.\\
$K_r$ & $\ntoken\times\dimqk$ & Key matrix & $XW_r^K$.\\
$V_r$ & $\ntoken\times\dimvalue$ & Value matrix & $XW_r^V$.\\
$\Vall$ & $h\ntoken\times h\dimvalue$ & Value direct sum & $V_1\oplus\cdots\oplus V_h$.\\
\addlinespace[5pt]
\multicolumn{4}{@{}l}{\textbf{Attention and layer outputs}}\\*
\addlinespace[2pt]
$S_r$ & $\ntoken\times\ntoken$ & Score matrix & $Q_rK_r^\top/\sqrt{\dimqk}$.\\
$A_r$ & $\ntoken\times\ntoken$ & Attention matrix & $\softmax(\beta S_r)$, rowwise.\\
$P$ & $\ntoken\times\ntoken$ & Common-row projection & $\mathbf1_\ntoken\mathbf1_\ntoken^\top/\ntoken$.\\
$P^\perp$ & $\ntoken\times\ntoken$ & Centering projection & $I_\ntoken-P$.\\
$\Ahead{r}$ & $\ntoken\times\ntoken$ & Centered attention & $A_r-P$.\\
$A$ & $\ntoken\times h\ntoken$ & Attention concatenation & $[A_1,\ldots,A_h]$.\\
$\Aperp$ & $\ntoken\times h\ntoken$ & Centered concatenation & $[\Ahead{1},\ldots,\Ahead{h}]$.\\
$\MHA(X)$ & $\ntoken\times\dimout$ & Layer output & $A\Vall W^O$.\\
$\MHAperp(X)$ & $\ntoken\times\dimout$ & Centered output & \eqref{eq:mha-factorization}.\\
$\MHAG(X)$ & $\ntoken\times\dimout$ & Gaussian output & \eqref{eq:gaussian-layer}.\\
\addlinespace[5pt]
\multicolumn{4}{@{}l}{\textbf{Exponential features and Gaussian comparison}}\\*
\addlinespace[2pt]
$\Yfh{r}$ & $\ntoken\times\ntoken$ & Exponential features & \eqref{eq:Dh}.\\
$\Yf$ & $\ntoken\times h\ntoken$ & Feature concatenation & \eqref{eq:block-features}.\\
$\Yfc$ & $\ntoken\times h\ntoken$ & Conditionally centered features & $\Yf-\E[\Yf\mid\mathcal F_n]$.\\
$\noiseh{r}$ & $\ntoken\times\ntoken$ & Comparison noise & Independent standard Gaussian; \eqref{eq:linear-model}.\\
$\YGh{r}$ & $\ntoken\times\ntoken$ & Gaussian head & \eqref{eq:linear-model}.\\
$\YG$ & $\ntoken\times h\ntoken$ & Gaussian concatenation & $[\YGh{1},\ldots,\YGh{h}]$.\\
$\covKh{r}$ & $\ntoken\times\ntoken$ & Key Gram matrix & \eqref{eq:key-gram}.\\
$\covfh{r}$ & $\ntoken\times\ntoken$ & Feature covariance & \eqref{eq:exact-covariance}.\\
$\covGh{r}$ & $\ntoken\times\ntoken$ & Gaussian covariance & \eqref{eq:key-gram}.\\
$T_\ntoken$ & $h\ntoken\times {\dimout}_\ntoken$ & Right multiplier & \Cref{ass:right-multipliers}.\\
$\YGcov{\Lambda}$ & $\ntoken\times h\ntoken$ & Gaussian auxiliary matrix & \eqref{eq:gaussian-auxiliary}.\\
$\covVh{r}$ & $\dimvalue\times\dimvalue$ & Value covariance block & $V_r^\top\covGh{r}V_r/\ntoken$.\\
$\covV$ & $h\dimvalue\times h\dimvalue$ & Value covariance & $\bigoplus_r\covVh{r}$.\\
$\covVO$ & $\dimout\times\dimout$ & Output covariance & \eqref{eq:population}.\\
\end{longtable}
\endgroup

Function symbols in \cref{tab:notation-functions} retain parameter
subscripts but omit evaluation arguments. The transforms use a compactly
supported law $\nu$ on $[0,\infty)$; $\chi_\nu$ and $\Stransform_\nu$
require positive mean. The inverse $\chi_\nu$ is defined locally at zero
and on $(\nu(\{0\})-1,0)$; transform domains and continuations are given in
\cref{app:transforms}.

\begingroup
\small
\renewcommand{\arraystretch}{1.15}
\setlength{\LTleft}{0pt}
\setlength{\LTright}{0pt}
\begin{longtable}{@{}>{\raggedright\arraybackslash}p{.16\textwidth}
                    >{\raggedright\arraybackslash}p{.33\textwidth}
                    >{\raggedright\arraybackslash}p{\dimexpr.51\textwidth-4\tabcolsep\relax}@{}}
\caption{Model functions and spectral transforms.}
\label{tab:notation-functions}\\
\toprule
Symbol & Name & Definition\\
\midrule
\endfirsthead
\multicolumn{3}{@{}l}{\tablename\ \thetable\ (continued)}\\
\toprule
Symbol & Name & Definition\\
\midrule
\endhead
\midrule
\multicolumn{3}{r@{}}{Continued on the next page}\\
\endfoot
\bottomrule
\endlastfoot
$\featfn$ & Exponential feature function & $\featfn(x)=e^{\beta x-\beta^2/2}-1$.\\
$\softmax$ & Row softmax & $\softmax(B)_{ij}=e^{B_{ij}}/\sum_k e^{B_{ik}}$.\\
$G_\nu$ & Cauchy transform & \eqref{eq:cauchy-transform}.\\
$\psi_\nu$ & $\psi$-function & \eqref{eq:psi-transform}.\\
$\chi_\nu$ & Inverse $\psi$-function & $\chi_\nu=\psi_\nu^{\langle-1\rangle}$.\\
$\Stransform_\nu$ & $S$-transform & \eqref{eq:s-transform}.\\
$R_\nu$ & $R$-transform & $R_\nu(z)=\sum_{q\ge1}\kappa_q(\nu)z^{q-1}$.\\
$w_\nu$ & Effective degrees of freedom per token & $w_\nu(\lambda)=\int x/(x+\lambda)\,\nu(dx)=-\psi_\nu(-1/\lambda)$, $\lambda>0$.\\
$\kvmap$ & Key--value covariance polynomial & \eqref{eq:phikv}.\\
\end{longtable}
\endgroup

In \cref{tab:notation-laws}, $C\ge0$ is $N\times N$ with eigenvalues
$\lambda_j(C)$ and $t>0$. The measure operations are defined in
\cref{tab:notation-common}.
The value and output covariance limits use the assumptions of
\cref{app:vo,thm:main}, respectively.

\begingroup
\small
\renewcommand{\arraystretch}{1.15}
\setlength{\LTleft}{0pt}
\setlength{\LTright}{0pt}
\begin{longtable}{@{}>{\raggedright\arraybackslash}p{.16\textwidth}
                    >{\raggedright\arraybackslash}p{.33\textwidth}
                    >{\raggedright\arraybackslash}p{\dimexpr.51\textwidth-4\tabcolsep\relax}@{}}
\caption{Probability laws.}
\label{tab:notation-laws}\\
\toprule
Symbol & Name & Definition\\
\midrule
\endfirsthead
\multicolumn{3}{@{}l}{\tablename\ \thetable\ (continued)}\\
\toprule
Symbol & Name & Definition\\
\midrule
\endhead
\midrule
\multicolumn{3}{r@{}}{Continued on the next page}\\
\endfoot
\bottomrule
\endlastfoot
$\Normal(m,\Sigma)$ & Gaussian law & Mean $m$ and covariance $\Sigma$; in one dimension the second parameter is the variance.\\
$\delta_a$ & Dirac measure & Unit point mass at $a$.\\
$\mu_C$ & Empirical eigenvalue law & $N^{-1}\sum_{j=1}^N\delta_{\lambda_j(C)}$.\\
$\nu_M$ & Squared-singular-value law & $\mu_{MM^\top}$.\\
$\vlaw$ & Value covariance law & $\lim\mu_{\covV}$.\\
$\volaw$ & Output covariance law & $\lim\mu_{\covVO}$.\\
$\freepois(\mu,t)$ & Compound free Poisson law & $\kappa_q(\freepois(\mu,t))=t m_q(\mu)$.\\
$\freepois_t$ & Free Poisson law & $\freepois(\delta_1,t)$.\\
$\MP_t$ & Mean-one Marchenko--Pastur law & $\mathcal D_t\freepois_{1/t}$; row-to-column ratio $t$.\\
$\rho_t$ & Affine Marchenko--Pastur law & \eqref{eq:rho}.\\
$\rho_{\beta,t}$ & Mean-one affine law & $\mathcal D_{1/\theta_1}\rho_t$.\\
$\attnlimit$ & Attention bulk law & \eqref{eq:limiting-law}.\\
$\kvlaw$ & Key--value covariance law & $\pushforward{\kvmap}{\freepois_{\keyratio/h}}$.\\
\end{longtable}
\endgroup

%% file: 091_replacement_proof.tex
\section{Proof of Gaussian Equivalence}
\label{app:replacement}
This section supplies the estimates used in \cref{sec:transformations}.
Constants may depend on fixed $h,\beta,\keyratio$,
the moment order, and an eventual norm bound on the right multiplier
in \cref{ass:right-multipliers}. At each size, conditional statements refer to
$\mathcal F_n=\sigma((K_{r,n})_r,T_n)$, with $T_n=I_{hn}/\sqrt h$ for
\cref{cor:attention}. The same statements hold for every larger
sigma-field that is independent of the queries and the noise, such as
the one generated by the key, value, and output weights.
Independence in \cref{ass:right-multipliers} and the
comparison-noise construction preserve the Gaussian query and noise
laws under this conditioning. All estimates below are uniform over
keys and multipliers satisfying the specified bounds.

For real $M\ge1$, use the events
$\Omega_n(M)\in\mathcal F_n$ defined in \eqref{eq:replacement-regularity}.
Almost surely there exist integers $M,N_0$ such that $\Omega_n(M)$ holds
for every $n\ge N_0$, by \cref{ass:right-multipliers,lem:gaussian-keys}.
On $\Omega_n(M)$ the conditional moment estimates have deterministic constants
depending only on $M$ and the fixed parameters.

The transition to almost-sure statements uses unconditional
Borel--Cantelli. If $E_n$ is an exceptional event and
\[
 \mathbf1_{\Omega_n(M)}\Prob(E_n\mid\mathcal F_n)\le a_n(M),
   \sum_n a_n(M)<\infty,
\]
then
\[
 \Prob(E_n\cap \Omega_n(M))
 =\E[\mathbf1_{\Omega_n(M)}\Prob(E_n\mid\mathcal F_n)]\le a_n(M).
\]
Consequently $E_n\cap \Omega_n(M)$ occurs only finitely often a.s.
Intersecting over integer $M$, positive rational tolerances, fixed
moment orders, and any prescribed countable family of matrix sequences
gives the common probability-one event. Eventual membership in some
$\Omega_n(M)$ removes the localization. Thus these summable estimates
allow dependence across matrix sizes and arbitrarily varying
sigma-fields $\mathcal F_n$.

The auxiliary Gaussian matrix $\auxnoise$ of
\eqref{eq:gaussian-auxiliary} is independent of $\mathcal F_n$ and the
queries jointly; the probability space is enlarged to carry it when
needed. The conclusion of \cref{lem:replacement} involves only the original
variables, so it holds almost surely on the original space.

All conditional expectations can be defined as finite integrals against
the Gaussian conditional kernels at fixed keys and multipliers.
For signed quantities they agree with the usual conditional expectations
after multiplication by $\mathbf1_{\Omega_n(M)}$.

Conditional $L^s$ norms mean
$\|U\|_{L^s}=(\E[\|U\|^s\mid\mathcal F_n])^{1/s}$,
with absolute values, Euclidean norms, or operator norms as in
\cref{sec:joint-replacement}.
We write $\ntoken^{o(1)}$ for a nonnegative bound growing
more slowly than every fixed positive power of $\ntoken$; the bound may
depend on a fixed moment order.

\input{100_scaling_and_convergence}
\input{102_auxiliary_lemmas}
\input{108_block_representation}
\input{094_covariance_and_polynomials}
\input{095_direct_moments}
\input{104_headwise_estimates}

\input{105_gaussian_replacement}
\input{096_normalization_and_assembly}

%% file: 100_scaling_and_convergence.tex
\subsection{Input Scaling}
\label{app:input-scaling}
Head-count comparisons use
$A/\sqrt h$ and $\YG/\sqrt h$ explicitly; the attention bulk is observed
through $\sqrt{\ntoken/h}\Aperp$.

\begin{remark}[Scaling correspondence with the single-head model]\label{rem:scaling}
\citet[Section 2 and Remark 2.1]{hayase2026gaussian} use
$X_0X_0^\top=I_\ntoken$ and unit-variance Gaussian query and key weights.
Our convention in \cref{sec:gaussian-input} gives the same projected
queries and keys: set $X_0=X/\sqrt d$ and
$W_{0,r}^Q=\sqrt d W_r^Q$, $W_{0,r}^K=\sqrt d W_r^K$, so that
\[
 X_0X_0^\top=I_\ntoken,\qquad
 X_0W_{0,r}^Q=XW_r^Q=Q_r,\qquad
 X_0W_{0,r}^K=XW_r^K=K_r.
\]
Thus the scores, softmax attention matrices, and their spectra are
unchanged. The same rescaling of value weights preserves $V_r$;
with $W^O$ held fixed, it also preserves the MHA output.
This input--weight conversion is separate from the head-comparison
factor $h^{-1/2}$ and the $\sqrt\ntoken$ scaling used to study the
attention bulk.
\end{remark}

\subsection{Weak Convergence and Moment Convergence}
\label{app:convergence}
For Borel probability measures $\mu_\ntoken,\mu$ on $[0,\infty)$,
weak convergence $\mu_\ntoken\Rightarrow\mu$ means
$\int g\,d\mu_\ntoken\to\int g\,d\mu$ for every bounded continuous $g$.
Moment convergence means that all moments are finite and
$m_q(\mu_\ntoken)\to m_q(\mu)$ for every $q\in\N$, where
$m_q(\mu)=\int x^q\,d\mu(x)$.
For a matrix $M_\ntoken$ with $\ntoken$ rows, these are moments of
squared singular values, including zeros:
\[
 m_q(\nu_{M_\ntoken})=\tr_\ntoken[(M_\ntoken M_\ntoken^\top)^q].
\]
For random measures on a common probability space, almost-sure
convergence requires one measurable event of probability one on which
the stated convergence holds simultaneously for all test functions
or moment orders. A random limit is evaluated at the same outcome.
These conventions apply to \cref{thm:main,cor:attention}.

Moment convergence implies weak convergence when the limit is
compactly supported (\cref{app:moment-to-weak}). Conversely, an eventual
common compact support makes weak convergence imply moment convergence.
For random measures, these implications apply pathwise under the
respective support conditions.

%% file: 102_auxiliary_lemmas.tex
\subsection{Moment Perturbation Bounds}
\label{app:moment-stability}
For $X,Y\in\R^{n\times m}$ and an integer $q\ge1$, set
$M=\max(\|X\|,\|Y\|)$ and
$\Delta_q=|m_q(\nu_X)-m_q(\nu_Y)|$. Singular-value perturbation,
Mirsky's inequality, and the Gram rank bound give, respectively,
\begin{align}
 \Delta_q&\le2qM^{2q-1}\|X-Y\|,\label{eq:op-moment-bound}\\
 \Delta_q&\le2qM^{2q-1}\|X-Y\|_F/\sqrt n,\label{eq:frob-moment-bound}\\
 \Delta_q&\le2\rank(X-Y)M^{2q}/n.\label{eq:rank-moment-bound}
\end{align}
Indeed, $s\mapsto s^{2q}$ is $2qM^{2q-1}$-Lipschitz on $[0,M]$;
apply the operator and Frobenius singular-value bounds with zero padding.
For the last inequality, the Gram difference has rank at most
$2\rank(X-Y)$; integrate the resulting empirical distribution bound
against $qt^{q-1}$ on $[0,M^2]$.

%% file: 108_block_representation.tex
\subsection{Block Representation}
\label{sec:blocks}
Collect the keys as
\begin{equation}\label{eq:block-keys}
 \Kall=K_1\oplus\cdots\oplus K_h.
\end{equation}
Here $\Kall\in\R^{h\ntoken\times h\dimqk}$ has rectangular head blocks.
Its product with the concatenated queries gives the scores exactly:
\begin{equation}\label{eq:block-features}
 \frac{[Q_1,\ldots,Q_h]\Kall^\top}{\sqrt{\dimqk}}=[S_1,\ldots,S_h],
 \Yf=\frac{\featfn([S_1,\ldots,S_h])}{\sqrt{\ntoken}}.
\end{equation}
Given $\mathcal F_n$, the concatenated query rows are i.i.d.\
$\Normal(0,I_{h\dimqk})$. The block covariances from \eqref{eq:key-gram} are
\begin{align}
 \covK&=\Kall\Kall^\top/\dimqk
       =\bigoplus_r\covKh{r},\label{eq:block-key-gram}\\
 \covG&=\bigoplus_r\covGh{r}
       =\theta_2\covK+(\theta_1-\theta_2)I_{h\ntoken}.
       \label{eq:blockC}
\end{align}
By \eqref{eq:linear-model}, the concatenated Gaussian model
$\YG=[\YGh{1},\ldots,\YGh{h}]$ has conditional row
covariance $\covG/\ntoken$.
The conditional covariance $\covf$ is block diagonal
because the queries of different heads are independent. The key blocks
$K_r$ may be dependent across heads, since the proof uses the Gram
matrices $\covKh{r}$ within each head.
The two dimension ratios are $\ntoken/\dimqk\to h/\keyratio$ for each
key Gram matrix and $\ntoken/(h\ntoken)=1/h$ for the feature matrix.

The attention concatenation also has an exact block representation.
With $\Aperpall=\Ahead{1}\oplus\cdots\oplus\Ahead{h}$,
\begin{equation}\label{eq:block-compression}
 \Aperp=\sqrt hU_h\Aperpall,
 U_h=\frac1{\sqrt h}[I_\ntoken,\ldots,I_\ntoken].
\end{equation}
Since $U_hU_h^\top=I_\ntoken$, the projection $P_h=U_h^\top U_h$ has rank $\ntoken$, and
\begin{equation}\label{eq:block-grams}
 \Aperp(\Aperp)^\top=\sum_r \Ahead{r}(\Ahead{r})^\top,
 (\Aperp)^\top\Aperp=h(\Aperpall)^\top P_h\Aperpall.
\end{equation}
For the scaled matrix $\scaledA$, the empirical spectral distributions
of the two Gram matrices satisfy
\begin{equation}\label{eq:gram-measures}
 \mu_{(\scaledA)^\top\scaledA }
 =(1-1/h)\delta_0+h^{-1}\nu_{\scaledA}.
\end{equation}
\Cref{fig:block-structure} displays these matrices. The projection
retains the mixed head products needed for the spectrum of the
concatenation.

\begin{figure*}[t]
\centering
\begingroup
\begin{minipage}[t]{.48\textwidth}
\centering
\textbf{(a) Keys and conditional covariances}\par\smallskip
\(\displaystyle
 \Kall=\left[\begin{array}{c|c|c}
 \firstblock{K_1}&\cdots&0\\ \hline
 \vdots&\ddots&\vdots\\ \hline
 0&\cdots&\lastblock{K_h}
 \end{array}\right]_{h\ntoken\times h\dimqk}
\)\par\medskip
\(\displaystyle
 \covf=\left[\begin{array}{c|c|c}
 \firstblock{\covfh{1}}&\cdots&0\\ \hline
 \vdots&\ddots&\vdots\\ \hline
 0&\cdots&\lastblock{\covfh{h}}
 \end{array}\right]_{h\ntoken\times h\ntoken}
\)\par\smallskip
\(\covK=\covKh{1}\oplus\cdots\oplus \covKh{h},  \covG=\covGh{1}\oplus\cdots\oplus \covGh{h}.\)
\end{minipage}\hfill
\begin{minipage}[t]{.49\textwidth}
\centering
\textbf{(b) Attention and aggregation}\par\smallskip
\(\displaystyle
 \Aperpall=\left[\begin{array}{c|c|c}
 \firstblock{\Ahead{1}}&\cdots&0\\ \hline
 \vdots&\ddots&\vdots\\ \hline
 0&\cdots&\lastblock{\Ahead{h}}
 \end{array}\right],  \Aperp=\sqrt hU_h\Aperpall
\)\par\medskip
\(\displaystyle
 P_h=\frac1h\left[\begin{array}{c|c|c}
 I_\ntoken&\cdots&I_\ntoken\\ \hline
 \vdots&\ddots&\vdots\\ \hline
 I_\ntoken&\cdots&I_\ntoken
 \end{array}\right],
 P_h^2=P_h
\)\par\smallskip
\( ((\Aperp)^\top\Aperp)_{rs}=(\Ahead{r})^\top\Ahead{s}.\)
\end{minipage}
\endgroup
\caption{Block structure of multi-head self-attention.
In (a), each key block is $\ntoken\times\dimqk$, and each conditional
covariance block is $\ntoken\times\ntoken$; $\covfh{r}$ is defined in
\eqref{eq:exact-covariance}. In (b), all blocks are
$\ntoken\times\ntoken$ and $U_h$ maps the $h$ head spaces to the common
token space. In $P_h$, every block is $I_\ntoken/h$, and
$((\Aperp)^\top\Aperp)_{rs}$ denotes the $(r,s)$ head block.}
\label{fig:block-structure}
\end{figure*}

\FloatBarrier

%% file: 094_covariance_and_polynomials.tex
\subsection{Conditional Means and Covariances}
Use the block keys $\Kall$ and key Grams $\covKh{r}$ from
\eqref{eq:block-keys}--\eqref{eq:block-key-gram}. Fix a row index $\ell$.
For its independent queries $q_{r,\ell}=Q_{r,\ell,:}^\top$, define
\begin{equation}\label{eq:feature}
 (x_r)_j=\featfn(S_{r,\ell j})=\featfn(K_{r,j,:}q_{r,\ell}/\sqrt \dimqk),
 \bar x_r=\E[x_r\mid\mathcal F_n],  \xi_r=x_r-\bar x_r.
\end{equation}
Write $\xi=(\xi_1^\top,\ldots,\xi_h^\top)^\top\in\R^{h\ntoken}$.
Thus $\xi=\sqrt{\ntoken}(\Yfc)_{\ell,:}^\top$, where $\Yfc$ is the
conditionally centered feature matrix.

\begin{lemma}[Regularity of Gaussian keys]\label{lem:gaussian-keys}
For any joint law of the key weights across heads and any dependence
across sizes, a.s.
\begin{equation}\label{eq:key-good}
 \max_r\|\covKh{r}\|=O(1),
 \delta_\ntoken=\max_{r,i,j}|(\covKh{r}-I_\ntoken)_{ij}|
       =O(\sqrt{\log \ntoken/\ntoken}).
\end{equation}
\end{lemma}
\begin{proof}
For standard Gaussian $G\in\R^{n\times m}$,
$\Prob\{\|G\|>\sqrt n+\sqrt m+t\}\le2e^{-c_0t^2}$
\citep[Theorem 7.3.1 and Corollary 7.3.3]{vershynin2018high}.
Each $K_r$ has i.i.d.\ $\Normal(0,1)$ entries, so this bound and
Borel--Cantelli give the norm estimate.
For the diagonal bound use the chi-square exponential bound for
$\|K_{r,i,:}\|^2/\dimqk$, followed by a union bound over $h\ntoken$ indices.
Conditionally on a row with squared norm at most $2\dimqk$, its inner
product with any other row, divided by $\dimqk$, is Gaussian with variance
at most $2/\dimqk$; the complementary event is contained in the
diagonal exceptional event, whose probabilities are summable.
A union bound over $h\ntoken(\ntoken-1)$ pairs, and a sufficiently large constant in
$\sqrt{\log \ntoken/\dimqk}$, makes the exceptional probabilities summable.
Since $\dimqk\asymp \ntoken$, this proves \eqref{eq:key-good}. Each bound
concerns one head at one size, and the unions are finite.
\end{proof}
The row-replacement estimates below use the keys only through
\eqref{eq:key-good} and their independence from the queries and the noise.

\begin{lemma}[Conditional covariances and norm bounds]\label{lem:covariance}
The conditional mean and covariance are
\begin{align}
 (\bar x_r)_i&=e^{\beta^2((\covKh{r})_{ii}-1)/2}-1,\notag\\
 (\covfh{r})_{ij}
 &=e^{\beta^2((\covKh{r})_{ii}+(\covKh{r})_{jj}-2)/2}
      (e^{\beta^2(\covKh{r})_{ij}}-1),\label{eq:exact-covariance}\\
 \E[\xi\xi^\top\mid\mathcal F_n]
 &=\covf=\covfh{1}\oplus\cdots\oplus\covfh{h}.\notag
\end{align}
With $\covGh{r}$ from \eqref{eq:blockC}, uniformly on $\Omega_n(M)$
for each fixed $M$,
\begin{equation}\label{eq:covariance-bound}
 \max_r\max\{\|\covfh{r}\|,\|\covGh{r}\|\}=O(1).
\end{equation}
\end{lemma}
\begin{proof}
Conditionally on $\mathcal F_n$,
$S_{r,\ell,:}^\top\sim\Normal(0,\covKh{r})$, independently across heads.
Gaussian exponential moments prove \eqref{eq:exact-covariance}, including the zero
covariance blocks between different heads.

Let $D_{0,r}=\diag((e^{\beta^2((\covKh{r})_{ii}-1)/2})_{i=1}^{\ntoken})$.
The covariance has the
expansion in Hadamard powers
\[
 \covfh{r}=D_{0,r}\left(\sum_{j\ge1}
          \frac{\beta^{2j}}{j!}\covKh{r}^{\circ j}\right)D_{0,r}.
\]
Here $\covKh{r}^{\circ j}$ denotes the entrywise $j$th power.
For positive semidefinite $R$, the Schur multiplier norm formula
\citep[Theorem~1.4.1]{bhatia2007positive} gives
$\|R\circ A\|\le(\max_iR_{ii})\|A\|$, hence
$\|R^{\circ j}\|\le(\max_iR_{ii})^{j-1}\|R\|$.
On $\Omega_n(M)$, $\|\covKh{r}\|\le M$ and $(\covKh{r})_{ii}\le M$,
so the convergent exponential series bounds $\|\covfh{r}\|$ uniformly.
The bound for $\covGh{r}$ follows from its definition.

Since the full covariances are block diagonal, the same bound holds
for $\max\{\|\covf\|,\|\covG\|\}$, and eventual membership in
some $\Omega_n(M)$ gives bounded norms almost surely.
\end{proof}

%% file: 095_direct_moments.tex
\subsection{Concentration and Norm Estimates for Feature Rows}
\label{app:joint-replacement}
\subsubsection{Gaussian Estimates}
Condition on $\mathcal F_n$ and work on the event $\Omega_n(M)$ of
\cref{app:replacement}, where the bounds in \eqref{eq:key-good} hold
with the constant $M$. Use the symmetric positive semidefinite
multiplier $T$ obtained in \cref{sec:transformations}; it is
$\mathcal F_n$-measurable and satisfies $\|T\|\le M$ on $\Omega_n(M)$.
Fix a row $i$.
The centered feature column is $\xi=\sqrt{\ntoken}(\Yfc)_{i,:}^\top$.
Set
\begin{equation}\label{eq:transformed-row}
 y=T\xi,
 \Gamma=\E[yy^\top\mid\mathcal F_n].
\end{equation}
\Cref{lem:covariance} gives $\|\Gamma\|=O(1)$.
All expectations and $L^s$ norms below are taken under this conditioning.
Constants in the $\ntoken^{o(1)}$ bounds may depend on a fixed $s$.

For a standard Gaussian vector $u$, the Maurey--Pisier inequality
\citep[(12)]{adamczak2015concentration} gives
\begin{equation}\label{eq:Gaussian-Lp}
 \|F(u)-\E F(u)\|_{L^s}
 \le C_s\|\nabla F(u)\|_{L^s},\qquad s\ge1,
\end{equation}
where $C_s=\frac\pi2(\E|\chi|^s)^{1/s}$ for $\chi\sim\Normal(0,1)$, $F$ is continuously differentiable,
and both $F(u)$ and $\|\nabla F(u)\|$ belong to $L^s$.

\begin{lemma}[Exponential moments of a Gaussian maximum]\label{lem:gaussian-maximum}
Let $Z_1,\ldots,Z_N$ be centered Gaussian variables, not necessarily
independent, with $\E Z_j^2\le M$ for $M>0$ and $N\ge1$.
For $L=\max_j|Z_j|$ and each $a>0$,
\begin{equation}\label{eq:gaussian-maximum-moment}
 \E e^{aL}\le C_{a,M}e^{a\sqrt{2M\log(2N)}},
\end{equation}
where $C_{a,M}$ depends only on $a,M$.
\end{lemma}
\begin{proof}
A union bound gives $\Prob(L>t)\le\min\{1,2Ne^{-t^2/(2M)}\}$.
Put $t_0=\sqrt{2M\log(2N)}$ and $V=(L-t_0)_+$.
Then $\Prob(V>t)\le e^{-t^2/(2M)}$ for $t\ge0$, so
\[
 \E e^{aV}=1+a\int_0^\infty e^{at}\Prob(V>t)\,dt
 \le 1+a\sqrt{2\pi M}e^{a^2M/2}.
\]
Since $L\le t_0+V$, this proves \eqref{eq:gaussian-maximum-moment}
with the last expression as $C_{a,M}$.
The same estimate applies to conditional Gaussian laws with this
variance bound, uniformly over the conditioning data.
\end{proof}

In particular, for fixed $a,M$ and $N=O(n^k)$ with fixed $k>0$,
$\E e^{aL}=n^{o(1)}$.

For the exponential features,
\begin{equation}\label{eq:exponential-envelope}
 |\featfn(t)|+|\featfn'(t)|\le C_\beta e^{\beta|t|},
\end{equation}
where $C_\beta>0$ depends only on $\beta$.
Together with \cref{lem:gaussian-maximum}, this ensures the
$L^s$ integrability of $y^\top Hy$ and its gradient for deterministic
symmetric $H$, as required by
\eqref{eq:Gaussian-Lp}. The feature-specific Jacobian and
quadratic-form concentration are proved in \cref{sec:joint-replacement}.

\subsubsection{Truncation and Control of the Intermediate Matrices}
For independent PSD $N\times N$ matrices $A_i$ with $\|A_i\|\le L$
and $\mu_{\max}=\|\sum_i\E A_i\|>0$, matrix Chernoff gives
\begin{equation}\label{eq:matrix-chernoff}
 \Prob\{\|\textstyle\sum_i A_i\|\ge t\}
 \le N(e\mu_{\max}/t)^{t/L},\qquad t\ge e\mu_{\max},\quad L>0
\end{equation}
\citep[Corollary 5.2 and Remark 5.3]{tropp2012user}.
When $\mu_{\max}=0$ the sum is zero a.s.
We apply this bound after a rowwise cutoff, which preserves independence,
and then remove the cutoff.

\begin{lemma}[Bounds on operator norms]\label{lem:truncated-norm}
Let $\Gamma'\ge0$ be an $\mathcal F_n$-measurable
$h\ntoken\times h\ntoken$ matrix with $\|\Gamma'\|\le C(M)$ on
$\Omega_n(M)$.
Let $y_1,\ldots,y_\ntoken$ be independent rows, each chosen either from the
feature distribution in \eqref{eq:transformed-row} or from $\Normal(0,\Gamma')$.
For $\rowgram=\ntoken^{-1}\sum_i y_iy_i^\top$ and every fixed $s\ge1$,
\begin{equation}\label{eq:mixed-norm-bound}
 \sup_{\text{choices of row types}}\E\|\rowgram\|^s=\ntoken^{o(1)}.
\end{equation}
The bound also holds with one row omitted. When all rows have the
feature distribution or all rows are Gaussian, $\|\rowgram\|=\ntoken^{o(1)}$
almost surely.
\end{lemma}
\begin{proof}
For any prescribed $D>0$ choose $C_D$ sufficiently large. A union bound
over the $O(\ntoken^2)$ Gaussian score coordinates, followed by
\eqref{eq:exponential-envelope}, gives
\begin{equation}\label{eq:row-cutoff}
 \Prob\{\max_i\|y_i\|>\sqrt \ntoken e^{C_D\sqrt{\log \ntoken}}\}\le \ntoken^{-D}.
\end{equation}
The Gaussian rows obey the same bound by writing them as
$(\Gamma')^{1/2}z_i$ and using a union bound for the coordinates of $z_i$.
Centering a feature row changes the bound only by a fixed factor since
all its coordinate means are bounded.

Truncate each row separately, preserving independence, and put
\[
 A_i=\frac{y_iy_i^\top}{\ntoken}
   \mathbf1_{\{\|y_i\|\le\sqrt \ntoken e^{C_D\sqrt{\log \ntoken}}\}},
   L_\ntoken=e^{2C_D\sqrt{\log \ntoken}}.
\]
These matrices are positive semidefinite, $\|A_i\|\le L_\ntoken$, and
$\mu_{\max}:=\|\sum_i\E A_i\|\le\max(\|\Gamma\|,\|\Gamma'\|)=O(1)$.
Apply the quoted matrix Chernoff bound \eqref{eq:matrix-chernoff}
with summand bound $L=L_\ntoken$.
If $\mu_{\max}=0$, the truncated sum vanishes a.s. Otherwise taking
$t=C L_\ntoken\log \ntoken$ gives any prescribed polynomial tail decay for
sufficiently large $C$. Integrating this bound gives
$\E\|\sum_iA_i\|^s=\ntoken^{o(1)}$ for every fixed $s$.

To remove the cutoff, the coordinate moment estimates give
\[
 \E\bigl(h\ntoken\tr_{h\ntoken}\rowgram\bigr)^{2s}
 =\E\left(\ntoken^{-1}\sum_i\|y_i\|^2\right)^{2s}=O(\ntoken^{2s}).
\]
Since $\|\rowgram\|\le h\ntoken\tr_{h\ntoken}\rowgram$, Cauchy--Schwarz bounds the contribution on the
cutoff event by $C_s \ntoken^s \ntoken^{-D/2}$. Choose $D>2s$ to prove
\eqref{eq:mixed-norm-bound}, uniformly over the row types.
Omitting one row changes none of these upper bounds.
For an actual sequence of arrays, choose $D>1$ in \eqref{eq:row-cutoff}
and a summable Chernoff tail. Integrate the conditional tail bounds
on each $\Omega_n(M)$ and apply the argument of \cref{app:replacement}.
Borel--Cantelli gives an eventual bound
of the form $C e^{C\sqrt{\log \ntoken}}\log \ntoken=\ntoken^{o(1)}$ for $\|\rowgram\|$.

\end{proof}

We use \cref{lem:truncated-norm} with $\Gamma'=\Gamma$ and
$\Gamma'=T\covG T$.
In particular, almost surely,
\begin{equation}\label{eq:feature-norms}
 \max\{\|\Yf\|,\|\Yfc\|\}=\ntoken^{o(1)}.
\end{equation}
To obtain this bound, apply the lemma with $T=I$ to the centered exponential features.
The uncentered matrix differs by a rank-one conditional mean matrix.
By \eqref{eq:exact-covariance} and \eqref{eq:key-good}, each mean coordinate is
$O(\sqrt{\log \ntoken/\ntoken})$. The norm of the mean matrix, including the
$1/\sqrt{\ntoken}$ normalization, is $O(\sqrt{\log \ntoken})$.
This proves \eqref{eq:feature-norms}.

\begin{lemma}[Concentration of trace moments]\label{lem:selfaverage}
For the arrays in \cref{lem:truncated-norm} with identically distributed
feature rows or identically distributed Gaussian rows, and any fixed
integer $q\ge1$, on $\Omega_n(M)$ with constants depending on $M$,
\begin{equation}\label{eq:moment-variance}
 \Var\left(\tr_{h\ntoken}(\rowgram^q)\mid\mathcal F_n\right)\le \ntoken^{-2+o(1)}.
\end{equation}
Hence $\tr_{h\ntoken}(\rowgram^q)-\E[\tr_{h\ntoken}(\rowgram^q)\mid\mathcal F_n]\to0$ a.s.
\end{lemma}
\begin{proof}
For $\Phi=\tr_{h\ntoken}(\rowgram^q)$, differentiation with respect to row $y_i$
gives $\nabla_{y_i}\Phi=2q(h\ntoken^2)^{-1}\rowgram^{q-1}y_i$.
Let $J_{\max}$ be the maximum norm of the Jacobians of the rows with
respect to the independent standard Gaussian vectors generating them.
For feature rows, \eqref{eq:jacobian-bound} and
\cref{lem:gaussian-maximum} applied to all $h\ntoken^2$ scores give
$\E J_{\max}^s=\ntoken^{o(1)}$ for every fixed $s$.
For Gaussian rows $(\Gamma')^{1/2}z_i$, the Jacobian is the bounded
matrix $(\Gamma')^{1/2}$, so the same moment bound holds.
Gaussian Poincar\'e, $\Var F(u)\le\E\|\nabla F(u)\|^2$
\citep[Theorem 3.20]{boucheron2013concentration}, yields
\[
 \Var\Phi
 \le\frac{4q^2}{h^2\ntoken^4}
      \E\left[J_{\max}^2\sum_i y_i^\top \rowgram^{2q-2}y_i\right]
 =\frac{4q^2}{h\ntoken^2}\E[J_{\max}^2\tr_{h\ntoken}(\rowgram^{2q-1})]\le \ntoken^{-2+o(1)}.
\]
The last step uses $\tr_{h\ntoken}(\rowgram^{2q-1})\le\|\rowgram\|^{2q-1}$,
H\"older, and
\eqref{eq:mixed-norm-bound}. These bounds also justify the Gaussian
Sobolev calculation by truncation. More explicitly, for fixed
$\varepsilon>0$ and any $0<\eta<1$, conditional Chebyshev gives,
on $\Omega_n(M)$, uniformly for sufficiently large $n$,
\[
 \Prob\{|\Phi-\E[\Phi\mid\mathcal F_n]|>\varepsilon
           \mid\mathcal F_n\}
 \le C_{q,M,\eta}\varepsilon^{-2}\ntoken^{-2+\eta}.
\]
Multiply by $\mathbf1_{\Omega_n(M)}$ and take expectations. The right
side is summable in $\ntoken$, so the unconditional Borel--Cantelli
argument of \cref{app:replacement}, with a countable intersection over
$M$ and $\varepsilon=1/j$, proves the assertion.
\end{proof}

%% file: 104_headwise_estimates.tex
\subsection{Concentration of the Row Normalizers}
\label{app:headwise-estimates}
\begin{lemma}[Concentration of row normalizers]\label{lem:normalizer-array}
For every fixed $\varepsilon>0$, almost surely,
\begin{equation}\label{eq:normalizer-rate}
 \max_{r,i}|D_{r,ii}-1|=O(\ntoken^{-1/2+\varepsilon}).
\end{equation}
\end{lemma}
\begin{proof}
We prove the normalizer rate on the original probability space by
conditioning on a query row.
For head $r$ and row $i$, let $q_{r,i}=Q_{r,i,:}^\top$ and
$v_{r,i}=\|q_{r,i}\|^2/\dimqk$.
Chi-square concentration and a union bound give
$\max_{r,i}|v_{r,i}-1|=O(\sqrt{\log \ntoken/\ntoken})$ a.s.; the probabilities of
$\max_{r,i}v_{r,i}>v_{\max}$ are summable for any fixed
$v_{\max}>1$. Conditional on $q_{r,i}$, the variables
$\zeta_{r,ij}=e^{\beta S_{r,ij}-\beta^2/2}$ are independent over $j$, with mean
$\mu_{r,i}=e^{\beta^2(v_{r,i}-1)/2}$. On $v_{r,i}\le v_{\max}$, all their fixed centered
moments are uniformly bounded.

For an integer $m\ge2$, Rosenthal's inequality
\citep[Theorem~15.11]{boucheron2013concentration}, applied conditionally
to the independent centered variables $\zeta_{r,ij}-\mu_{r,i}$, gives
\[
 \E[|D_{r,ii}-\mu_{r,i}|^{2m}\mid q_{r,i}]
 \le C_{m,v_{\max}} \ntoken^{-m},\qquad v_{r,i}\le v_{\max}.
\]
Markov and the union bound over $h\ntoken$ rows give
\[
 \Prob\{\max_{r,i}|D_{r,ii}-\mu_{r,i}|>\ntoken^{-1/2+\varepsilon},
                      \max_{r,i}v_{r,i}\le v_{\max}\}
       \le C_{m,v_{\max}} \ntoken^{1-2m\varepsilon}.
\]
Choose $m\varepsilon>1$ and apply Borel--Cantelli.
The conditional means differ from one by
$O(\sqrt{\log \ntoken/\ntoken})$, proving \eqref{eq:normalizer-rate}.
\end{proof}

\begin{lemma}[Norm of the centered attention matrices]\label{lem:attention-norm}
Almost surely $\|\scaledA\|=\ntoken^{o(1)}$.
\end{lemma}
\begin{proof}
By \eqref{eq:exact-normalizer},
$\sqrt{\ntoken}\Ahead{r}=D_r^{-1}\Yfh{r}P^\perp$, and
$\|D_r^{-1}\|=1+o(1)$ by \eqref{eq:normalizer-rate}. A
concatenation satisfies
$\|[N_1,\ldots,N_h]\|^2\le\sum_r\|N_r\|^2$, and
$\|\Yfh{r}\|\le\|\Yf\|=\ntoken^{o(1)}$ by
\eqref{eq:feature-norms}. Hence
$\|\scaledA\|\le(1+o(1))\sqrt h\|\Yf\|=\ntoken^{o(1)}$.
\end{proof}

Equivalently, $\|\Aperp\|\le\ntoken^{-1/2+o(1)}$ almost surely.

\input{112_normalizer_identity}

%% file: 112_normalizer_identity.tex
\subsubsection{A Finite-Size Normalizer Identity}
\label{app:normalizer-variance}
With $\zeta_{r,ij}=e^{\beta S_{r,ij}-\beta^2/2}$, the row normalizer in
\eqref{eq:Dh} is $D_{r,ii}=\ntoken^{-1}\sum_{j=1}^{\ntoken}\zeta_{r,ij}$.

\begin{prop}\label{prop:normalizer-variance}
For each head $r$ and row $i$, if $\dimqk>4\beta^2$, then
\begin{align*}
 \E D_{r,ii}&=e^{-\beta^2/2}(1-\beta^2/\dimqk)^{-\dimqk/2},\\
 \E D_{r,ii}^2&=e^{-\beta^2}\left[
 (1-1/\ntoken)(1-2\beta^2/\dimqk)^{-\dimqk/2}
 +(1/\ntoken)(1-4\beta^2/\dimqk)^{-\dimqk/2}\right].
\end{align*}
Consequently, for fixed $\beta$ and $\dimqk\asymp \ntoken$,
\[
 \Var(D_{r,ii})=\frac{e^{\beta^2}-1}{\ntoken}+\frac{\beta^4}{2\dimqk}
                         +O(\dimqk^{-2}+(\ntoken \dimqk)^{-1}).
\]
\end{prop}
\begin{proof}
Condition on the query row and write $v_{r,i}=\|Q_{r,i,:}\|^2/\dimqk$.
Its scores across keys are conditionally independent $\Normal(0,v_{r,i})$, so
\[
 \E[D_{r,ii}\mid v_{r,i}]=e^{\beta^2(v_{r,i}-1)/2},
 \Var(D_{r,ii}\mid v_{r,i})=\ntoken^{-1}e^{\beta^2(v_{r,i}-1)}(e^{\beta^2v_{r,i}}-1).
\]
Using $\E e^{tv_{r,i}}=(1-2t/\dimqk)^{-\dimqk/2}$ gives the exact formulas.
Expanding $(1-x/\dimqk)^{-\dimqk/2}=e^{x/2}(1+x^2/(4\dimqk)+O(\dimqk^{-2}))$
at $x=\beta^2,2\beta^2,4\beta^2$ and subtracting $(\E D_{r,ii})^2$
gives the stated variance expansion.
\end{proof}

%% file: 105_gaussian_replacement.tex
\subsection{Gaussian Replacement from Concentration}
\label{app:gaussian-replacement-criterion}
In the following criterion, expectations and norms are conditional on
$\mathcal F_n$.

\begin{lemma}[A criterion for Gaussian replacement]\label{lem:gaussian-replacement-criterion}
Fix $h$. Conditionally on $\mathcal F_n$, let
$Y,G\in\R^{n\times hn}$ have independent rows with laws
$y^\top/\sqrt n$ and $g^\top/\sqrt n$, respectively, where
$\E y=0$, $\Gamma=\E yy^\top$, and $g\sim\Normal(0,\Gamma)$.
Let $\Omega_n(M)\in\mathcal F_n$, $M\ge1$, be increasing in $M$,
with eventual membership in some $\Omega_n(M)$ almost surely.
Suppose the following conditional bounds hold uniformly on each
$\Omega_n(M)$ for fixed $M$:
$\|\Gamma\|=O(1)$, and, for every fixed $s\ge1$,
\begin{equation}\label{eq:replacement-concentration}
 \left\|\frac{y^\top Hy-\Tr(\Gamma H)}n\right\|_{L^s}
 \le n^{-1/2+o(1)}\|H\|
\end{equation}
uniformly over deterministic symmetric $H\in\R^{hn\times hn}$;
$\E\|W^\top W\|^s=n^{o(1)}$ uniformly over
$W\in\R^{n\times hn}$ whose independent rows each follow either row law;
and, for every fixed integer $q\ge1$,
\[
 \Var(m_q(\nu_Y)\mid\mathcal F_n)
 +\Var(m_q(\nu_G)\mid\mathcal F_n)\le n^{-2+o(1)}.
\]
Then, uniformly on each $\Omega_n(M)$,
$|\E m_q(\nu_Y)-\E m_q(\nu_G)|\le n^{-1/2+o(1)}$,
and $m_q(\nu_Y)-m_q(\nu_G)\to0$ almost surely.
\end{lemma}

\begin{lemma}[One-row replacement]\label{lem:one-row-replacement}
Under the hypotheses of \cref{lem:gaussian-replacement-criterion},
work conditionally on $\mathcal F_n$ at a realization in $\Omega_n(M)$.
Let $X$ have $n$ independent rows, each drawn either from the law of
$y^\top/\sqrt n$ or from the law of $g^\top/\sqrt n$.
Replace one feature row by an independent matched Gaussian row to
obtain $X'$. For every fixed integer $q\ge1$,
\begin{equation}\label{eq:one-row-replacement}
 \left|\E m_q(\nu_X)-\E m_q(\nu_{X'})\right|
 \le n^{-3/2+o(1)},
\end{equation}
uniformly on $\Omega_n(M)$, over the row being replaced, and over the
choices of the other row types, for each fixed $M$.
\end{lemma}

\begin{proof}[Proof of \cref{lem:one-row-replacement}]
Fix $M$ and $q$. All expectations below are under the conditional
row laws on $\Omega_n(M)$.
Let $x_i^\top/\sqrt n$, $i\ne\ell$, be the unchanged rows and put
\[
 R=\frac1n\sum_{i\ne\ell}x_ix_i^\top.
\]
The rows are independent, so the feature row $y$ and its Gaussian
replacement $g$ are independent of $R$. The right Gram matrices of
$X$ and $X'$ are $R+yy^\top/n$ and $R+gg^\top/n$.
For $q\ge1$, the left and right Gram matrices have the same nonzero
eigenvalues. Therefore $m_q(\nu_X)=h\tr_{hn}[(R+yy^\top/n)^q]$.

For a fixed symmetric $R$, the rank-one trace expansion is
\begin{equation}\label{eq:rank-one-trace-expansion}
 \tr_{hn}[(R+yy^\top/n)^q]-\tr_{hn}(R^q)
 =\frac1{hn}\sum_{k=1}^q\frac qk
    \sum_{\substack{a_1+\cdots+a_k=q-k\\a_j\ge0}}
    \prod_{j=1}^k\frac{y^\top R^{a_j}y}{n}.
\end{equation}
Marking one of the $k$ updates and applying cyclicity gives $q$
possible positions and counts each word $k$ times, yielding $q/k$.
Rank-one products factor into the displayed quadratic forms.

Condition also on $R$ and write
\[
 \phi_a(y)=\frac{y^\top R^a y}{n},\qquad
 \mu_a=\frac{\Tr(\Gamma R^a)}n.
\]
Until averaging over $R$ below, the $L^s$ norms refer to this further
conditioning.
Since both row laws have covariance $\Gamma$,
$\E[\phi_a(y)\mid R]=\E[\phi_a(g)\mid R]=\mu_a$.
For every fixed $s\ge1$, \eqref{eq:replacement-concentration} gives
\[
 \|\phi_a(y)-\mu_a\|_{L^s}\le n^{-1/2+o(1)}\|R\|^a.
\]
For $g=\Gamma^{1/2}u$ with $u\sim\Normal(0,I_{hn})$, Maurey--Pisier gives the same bound from
$\|\Gamma\|=O(1)$ and $\|g\|_{L^{2s}}=O(\sqrt n)$. Also
$|\mu_a|\le h\|\Gamma\|\|R\|^a$ and $\|\Gamma\|=O(1)$ uniformly
on $\Omega_n(M)$, so $\|\phi_a(y)\|_{L^s}$ and $\|\phi_a(g)\|_{L^s}$
are $O(\|R\|^a)$ for sufficiently large $n$. We use $\|R\|^0=1$.

The $k=1$ terms in \eqref{eq:rank-one-trace-expansion} have exactly
the same expectation. For $k\ge2$, use the identity
\[
 \prod_{j=1}^k \phi_{a_j}-\prod_{j=1}^k\mu_{a_j}
 =\sum_{t=1}^k(\phi_{a_t}-\mu_{a_t})
       \prod_{j<t}\phi_{a_j}\prod_{j>t}\mu_{a_j}.
\]
H\"older with exponent $k$ bounds the expected absolute value of each
summand by $n^{-1/2+o(1)}\|R\|^{a_1+\cdots+a_k}$.
Apply this separately to the feature and Gaussian products and
subtract their common product of means. Since $\sum_j a_j=q-k$,
the difference of their conditional expectations is bounded by
$n^{-1/2+o(1)}\|R\|^{q-k}$, uniformly on $\Omega_n(M)$.

Adding an independent allowed row gives $R\le W^\top W$.
The assumed Gram-norm bound therefore yields $\E\|R\|^s=n^{o(1)}$
for every fixed $s\ge1$, uniformly over the unchanged row types;
for $s=0$ the expectation is one.
We may therefore average over $R$ and sum the finitely many terms.
Multiplying the prefactor $(hn)^{-1}$ by the factor $h$ converting
right-Gram traces to $m_q$ yields
\[
 \left|\E m_q(\nu_X)-\E m_q(\nu_{X'})\right|
 \le n^{-3/2+o(1)}.
\]
For $q=1$ the difference is zero; all bounds are uniform as asserted.
\end{proof}

\begin{proof}[Proof of \cref{lem:gaussian-replacement-criterion}]
Starting from $Y$, replace all $n$ rows by independent Gaussian rows.
The resulting matrix has the law of $G$.
Telescoping the bound of \cref{lem:one-row-replacement} gives, on each
$\Omega_n(M)$,
\[
 |\E m_q(\nu_Y)-\E m_q(\nu_G)|
 \le n\cdot n^{-3/2+o(1)}=n^{-1/2+o(1)}.
\]
For $Z=Y,G$, conditional Chebyshev and the variance assumption give
\[
 \mathbf1_{\Omega_n(M)}
 \Prob\{|m_q(\nu_Z)-\E m_q(\nu_Z)|>\varepsilon\mid\mathcal F_n\}
 \le\varepsilon^{-2}n^{-2+o(1)}.
\]
Integrating and applying Borel--Cantelli for each integer $M$ and
$\varepsilon=1/j$ shows that, for $Z=Y,G$,
\[
 \mathbf1_{\Omega_n(M)}
 \bigl(m_q(\nu_Z)-\E[m_q(\nu_Z)\mid\mathcal F_n]\bigr)
 \longrightarrow0\qquad\text{a.s.}
\]
Thus the localized deviations from the conditional means vanish.
Eventual membership in some
$\Omega_n(M)$, and monotonicity in $M$, remove the localization.
Combining these two deviations with the difference of conditional expectations
proves the almost-sure comparison.
\end{proof}

\begin{lemma}[Reduction for the centered feature model]
\label{lem:feature-reduction}
In the setting of \cref{lem:direct-replacement}, write a row of
$\Yfc T$ as $y^\top/\sqrt n$ and put
$\Gamma=\E[yy^\top\mid\mathcal F_n]$.
Suppose \eqref{eq:replacement-concentration} holds conditionally
for every fixed $s\ge1$, uniformly over deterministic symmetric $H$
and on each $\Omega_n(M)$ of \eqref{eq:replacement-regularity} for fixed $M\ge1$.
Then, for every fixed integer $q\ge1$,
$m_q(\nu_{\Yfc T})-m_q(\nu_{\YGSig T})\to0$ almost surely.
\end{lemma}
\begin{proof}
We verify the remaining hypotheses of \cref{lem:gaussian-replacement-criterion}
for $Y=\Yfc T$ and $G=\YGSig T$.
Conditionally on $\mathcal F_n$, both matrices have independent
centered rows with common covariance $\Gamma/n$ by construction.
The events \eqref{eq:replacement-regularity} increase with $M$;
\cref{lem:gaussian-keys} and the multiplier assumption give
eventual membership in some $\Omega_n(M)$ almost surely.

Fix $M$. The covariance bound of \cref{lem:covariance},
together with $\|T\|\le M$, gives $\|\Gamma\|=O(1)$.
\Cref{lem:truncated-norm} supplies the Gram-norm moments
uniformly over all mixtures of the two row laws.
For $Z=Y,G$, the identity
$m_q(\nu_Z)=h\tr_{hn}[(Z^\top Z)^q]$ converts the variance estimate
of \cref{lem:selfaverage} into
$\Var(m_q(\nu_Z)\mid\mathcal F_n)\le n^{-2+o(1)}$.
All these bounds are uniform on $\Omega_n(M)$.
Thus the assumed quadratic-form bound completes the hypotheses of
\cref{lem:gaussian-replacement-criterion}, which proves the claim.
\end{proof}

%% file: 096_normalization_and_assembly.tex
\subsection{Completing the Gaussian Comparison}
We turn the entrywise covariance estimates of \cref{sec:identification}
into the spectral moment comparison using the following square-root bound.

\begin{lemma}[Perturbation of the square root]\label{lem:sqrt-perturbation}
Let $\Sigma,C\ge0$ be $N\times N$ matrices with $C\ge cI_N$ for some
$c>0$. Then
\begin{equation}\label{eq:sqrt-bound}
 \|\Sigma^{1/2}-C^{1/2}\|_F\le c^{-1/2}\|\Sigma-C\|_F.
\end{equation}
\end{lemma}
\begin{proof}
With $E=\Sigma^{1/2}-C^{1/2}$ we have
$\Sigma^{1/2}E+EC^{1/2}=\Sigma-C$. Diagonalizing the two square roots
separately divides each transformed entry by a number at least
$\sqrt c$, which proves \eqref{eq:sqrt-bound}.
\end{proof}

\begin{lemma}[From entrywise covariance matching to Gaussian comparison]
\label{lem:entrywise-covariance-comparison}
In the setting of \cref{sec:transformations}, put
$\rho_{ij}=(\covK)_{ij}$ and
$\delta_n=\max_{i,j}|(\covK-I_{hn})_{ij}|$.
Suppose $(\covf)_{ii}=\theta_1+O(\delta_n)$ and
$(\covf)_{ij}=\theta_2\rho_{ij}+O(\delta_n|\rho_{ij}|)$ for $i\ne j$,
uniformly in the indices and on each $\Omega_n(M)$ for fixed $M\ge1$.
Then, for every $q\in\N$,
$m_q(\nu_{\YGSig T})-m_q(\nu_{\YG T})\to0$ almost surely.
\end{lemma}
\begin{proof}
Condition on $\mathcal F_n$ and fix $M,q$. Set $E=\covf-\covG$.
Since $(\covG)_{ii}=\theta_1+\theta_2(\rho_{ii}-1)$,
the assumed estimates give $E_{ii}=O(\delta_n)$ and
$E_{ij}=O(\delta_n|\rho_{ij}|)$ for $i\ne j$.
On $\Omega_n(M)$, $\sum_{ij}\rho_{ij}^2=\Tr(\covK^2)\le hnM^2$, so
\begin{equation}\label{eq:covariance-approximation}
 \frac{\|E\|_F^2}{n}
 \le c\delta_n^2\left(h+\frac1n\sum_{i\ne j}\rho_{ij}^2\right)
 =O(\delta_n^2)=O(\log n/n).
\end{equation}
Here $c>0$ and all implicit constants depend only on $M,q$ and the fixed model parameters.
Eventual membership in some $\Omega_n(M)$ makes this estimate hold almost surely.

Construct $\YGSig$ and $\YGC$ from the same $\auxnoise$.
By \cref{lem:covariance}, both covariances have bounded operator norms.
Since $\covG\ge(\theta_1-\theta_2)I_{hn}$ and
$\theta_1-\theta_2=e^{\beta^2}-1-\beta^2>0$,
\cref{lem:sqrt-perturbation} gives
\[
 \frac{\|(\YGSig-\YGC)T\|_F}{\sqrt n}
 \le \frac{\|\auxnoise\|}{\sqrt n}
      \frac{\|T\|}{\sqrt{\theta_1-\theta_2}}
      \frac{\|E\|_F}{\sqrt n}
 =O(\sqrt{\log n/n})\quad\text{a.s.}
\]
The Gaussian norm bound $\|\auxnoise\|/\sqrt n=O(1)$ a.s.\ also bounds
the operator norms of $\YGSig T$ and $\YGC T$.
Thus \eqref{eq:frob-moment-bound} makes their fixed moment differences vanish.

Conditionally on $\mathcal F_n$, $\YGC T$ and $\YG T$ have the same
matrix law: their rows are independent centered Gaussians with covariance
$T\covG T/n$.
The coupling estimate and uniform Gaussian operator-norm moments also give
$|\E[m_q(\nu_{\YGSig T})-m_q(\nu_{\YG T})\mid\mathcal F_n]|
=O(\sqrt{\log n/n})$ on each $\Omega_n(M)$, by
\eqref{eq:frob-moment-bound} and H\"older.
The Gaussian case of \cref{lem:selfaverage}, with
$m_q(\nu_Z)=h\tr_{hn}[(Z^\top Z)^q]$, makes the moments of
$\YGC T$ and $\YG T$ differ from their common conditional expectation
by $o(1)$ almost surely. Adding these comparisons proves the claim.
\end{proof}

\paragraph{Weak comparison.}
\label{app:weak-comparison}
\label{app:moment-to-weak}
If two sequences of laws on $[0,\infty)$ have asymptotically equal
moments of every fixed order and one is eventually supported in a
fixed compact interval, then the difference between their integrals
of any bounded continuous function tends to zero.
Indeed, from any subsequence, the compactly supported laws admit a
further subsequence converging weakly and in all moments to a law on
$[0,M]$. The other laws have the same limiting moments, whose even
moments satisfy $m_{2k}^{1/(2k)}\le M$. The method of moments
\citep[Theorem~3.3.26]{durrett2019probability} therefore gives the same
weak limit along that further subsequence, proving the comparison.
For random laws this deterministic argument applies on the common
event of all moment comparisons and the support bound, simultaneously
for all bounded continuous functions, including the resolvents.

For the Gaussian matrices in \cref{lem:replacement}, the required
eventual compact support holds pathwise because
\[
 \|\YGh{r}\|\le\beta\|Q_r/\sqrt \ntoken\|\|K_r/\sqrt \dimqk\|
                    +\sqrt{\theta_1-\theta_2}\|\noiseh{r}/\sqrt \ntoken\|=O(1)
\]
a.s., and $\|T_\ntoken\|=O(1)$. The comparison
of moments in \cref{lem:replacement} therefore implies that the integrals
of every bounded continuous function against
$\nu_{\scaledA T_\ntoken}$ and $\nu_{\YG T_\ntoken}$ have vanishing
difference, almost surely. In particular the normalized resolvent traces
$\tr_\ntoken[(zI_\ntoken-ZZ^\top)^{-1}]$ of the two matrices $Z$ have
vanishing difference for all $z\in\C_+$.

\subsection{Leading Singular Value of the Output}
\label{app:leading-output}
\begin{proof}[Proof of \eqref{eq:main-leading}]
Write $B=[PV_1,\ldots,PV_h]W^O$, so that
$\MHA(X)=B+\MHAperp(X)$. By \eqref{eq:mha-factorization},
\cref{lem:attention-norm}, and the weight norm assumptions,
\[
 \frac{\|\MHAperp(X)\|}{\sqrt \dimout}
 \le\dimout^{-1/2}\|\Aperp\|\|\Vall\|\|W^O\|
 \le\ntoken^{-1/2+o(1)}\longrightarrow0
\]
almost surely. Since
$B=\mathbf1_\ntoken(\mathbf1_\ntoken^\top[V_1,\ldots,V_h]W^O)/\ntoken$
has rank at most one,
\[
 s_1(B)=\frac{\|\mathbf1_\ntoken^\top[V_1,\ldots,V_h]W^O\|_2}
                   {\sqrt\ntoken}.
\]
The singular-value perturbation bound
$|s_1(\MHA(X))-s_1(B)|\le\|\MHAperp(X)\|$
proves the claim.
\end{proof}

%% file: 116_architectural_support.tex
\section{Architectural Effects and Feature Utilization}
\label{app:architectural-support}
\input{092_gaussian_limits}

\input{106_key_dependent_values}
\subsection{Projection Summary}
\input{109_projection_catalog}
\FloatBarrier
\input{107_effective_dimension}
\input{117_readout_observables}
\input{115_additional_derivations}

%% file: 092_gaussian_limits.tex
\subsection{Spectral Laws and Projection Choices}
\label{app:transforms}
\subsubsection{Transforms and Gaussian Covariance Limits}
\input{093_spectral_transforms}

For $t>0$, the compound free Poisson law $\freepois(\nu,t)$ has
free cumulants $\kappa_q=t m_q(\nu)$
\citep[Proposition 12.15 and Definition 12.16]{nica2006lectures}.
The moment--cumulant relations \citep[Proposition 11.4]{nica2006lectures} give
\begin{equation}\label{eq:compound-moments}
 m_1(\freepois(\nu,t))=t m_1(\nu),\qquad
 m_2(\freepois(\nu,t))=t m_2(\nu)+t^2m_1(\nu)^2.
\end{equation}
Its $R$-transform is $t\psi_\nu(z)/z$; hence \eqref{eq:r-s-transform}
gives, for $m_1(\nu)>0$,
\begin{equation}\label{eq:column-pop-S}
 \Stransform_{\freepois(\nu,t)}(w)=\frac{\Stransform_\nu(w/t)}{t+w}.
\end{equation}
Write $\freepois_t=\freepois(\delta_1,t)$ and
$\mathcal D_t\mu=(x\mapsto tx)_\#\mu$.
The mean-one Marchenko--Pastur law with row-to-column ratio $y>0$ is
$\MP_y=\mathcal D_y\freepois_{1/y}$.

\begin{lemma}[Gaussian sample covariance limits]
\label{lem:gaussian-law}
Let $D_n\ge0$ be a random $n\times n$ matrix with
$\mu_{D_n}\Rightarrow\nu$ a.s.\ and $\|D_n\|=O(1)$ a.s., where $\nu$ is
deterministic.
Let $Z_n\in\R^{n\times p_n}$ be standard Gaussian and independent of $D_n$.
If $p_n/n\to t\in(0,\infty)$, then a.s., weakly
and in every fixed moment,
\begin{equation}\label{eq:row-pop-law}
 \mu_{D_n^{1/2}Z_nZ_n^\top D_n^{1/2}/n}
       \longrightarrow\nu\boxtimes\freepois_t.
\end{equation}
For $D_m\ge0$ of size $m\times m$ satisfying the same spectral assumptions,
and standard Gaussian $Z\in\R^{n\times m}$ independent of $D_m$,
if $m/n\to t$, then $\mu_{ZD_mZ^\top/n}\to\freepois(\nu,t)$
with the same convergence.
\end{lemma}
\begin{proof}
The Marchenko--Pastur law \citep[Section 4.5.1]{mingo2017free}
gives $\mu_{Z_nZ_n^\top/n}\Rightarrow\freepois_t$ a.s.
The Gaussian norm bound \citep[Corollary 7.3.3]{vershynin2018high}
and the assumed bound on $D_n$ upgrade both marginal weak limits
to convergence in every fixed moment.
By orthogonal invariance and independence,
\citet[Proposition 2.9]{collins2023asymptotic} yields almost-sure
asymptotic freeness of $D_n$ and $Z_nZ_n^\top/n$.
Thus \eqref{eq:row-pop-law} holds in moments and, by the norm bounds, weakly.

Applying \eqref{eq:row-pop-law} to $Z^\top$ and rescaling by $m/n$
gives $\nu\boxtimes\MP_t$ for $D_m^{1/2}Z^\top ZD_m^{1/2}/n$.
Its nonzero eigenvalues coincide with those of $ZD_mZ^\top/n$,
so the latter has limiting moments $t\,m_q(\nu\boxtimes\MP_t)$.
The product moment formula \citep[Theorem 14.4]{nica2006lectures}
identifies these as the moments of $\freepois(\nu,t)$;
the norm bounds give weak convergence.
\end{proof}

\subsubsection{The Attention Law}
\label{app:attention-law}
Define
\begin{align}
 \rho_t&=\pushforward{(x\mapsto\theta_1-\theta_2+\theta_2x)}{\MP_{1/t}},
 \label{eq:rho}\\
 \attnlimit&=\rho_{\keyratio}\boxtimes\MP_{1/h}.
 \label{eq:limiting-law}
\end{align}
Here $\pushforward{T}{\mu}$ is the law of $T(U)$ for $U\sim\mu$, and
$\boxtimes$ denotes multiplicative free convolution.
The attention law satisfies, almost surely, weakly and in every
fixed moment,
\begin{equation}\label{eq:attention-centered}
 \lim_{\ntoken\to\infty}\nu_{\sqrt{\ntoken/h}\Aperp}
 =\lim_{\ntoken\to\infty}\nu_{\YG/\sqrt h}=\attnlimit.
\end{equation}
Its first two moments are
\begin{equation}\label{eq:firstmoments}
 m_1(\attnlimit)=\theta_1,\qquad
 m_2(\attnlimit)=(1+1/h)\theta_1^2+\theta_2^2/\keyratio.
\end{equation}

\begin{lemma}[Scaling of affine MP laws]\label{lem:rho-scaling}
For $\rho_t$ in \eqref{eq:rho},
\begin{equation}\label{eq:rho-R}
 R_{\rho_t}(z)=\theta_1-\theta_2+
       \frac{\theta_2}{1-\theta_2z/t},\qquad
 \Stransform_{\rho_{\keyratio/h}}(w/h)=\Stransform_{\rho_{\keyratio}}(w).
\end{equation}
\end{lemma}
\begin{proof}
The $q$th cumulant of $\MP_{1/t}$ is $t^{1-q}$. Scaling by
$\theta_2$ multiplies it by $\theta_2^q$, and translation by
$\theta_1-\theta_2$ changes only the first cumulant, proving the first identity.
It gives $R_{\rho_{\keyratio/h}}(z/h)=R_{\rho_{\keyratio}}(z)$.
For $s=\Stransform_{\rho_{\keyratio}}(w)$, \eqref{eq:r-s-transform} gives
$sR_{\rho_{\keyratio/h}}((w/h)s)=sR_{\rho_{\keyratio}}(ws)=1$.
The analytic solution with $s(0)=1/\theta_1$ is unique because
$\theta_1>0$, proving the second identity.
\end{proof}

Given the keys, the rows of $\YG/\sqrt h$ are i.i.d.\
$\Normal(0,\covG/(h\ntoken))$. Such Gaussian matrices admit an
independent-noise representation on an extension of the probability space:
$\YG/\sqrt h=Z\covG^{1/2}/\sqrt{h\ntoken}$, with
$Z\in\R^{\ntoken\times h\ntoken}$ standard Gaussian and independent of the keys.
Gaussian Gram limits give
\[
 \|\covG\|=O(1),\qquad
 \mu_{\covG}\Rightarrow\rho_{\keyratio/h},\quad\text{a.s.}
\]
Only the common marginal law of the $h$ key blocks is used.
Applying \cref{lem:gaussian-law} with $D_m=\covG$ identifies the attention law as
$\dilate{1/h}{\freepois(\rho_{\keyratio/h},h)}$. Its transform is
\[
 h\frac{\Stransform_{\rho_{\keyratio/h}}(w/h)}{h+w}
 =\frac{\Stransform_{\rho_{\keyratio}}(w)}{1+w/h},
\]
which is the transform of
$\rho_{\keyratio}\boxtimes\MP_{1/h}$.
Compact support identifies the laws. Their first two cumulants are
\[
 \kappa_1=\theta_1,\qquad
 \kappa_2=\frac{m_2(\rho_{\keyratio/h})}{h}
 =\frac{\theta_1^2}{h}+\frac{\theta_2^2}{\keyratio},
\]
proving \eqref{eq:firstmoments}.

\subsubsection{The Attention-Matrix Corollary}
\label{app:main-proofs}
\begin{cor}[Attention matrices]\label{cor:attention}
The limits \eqref{eq:attention-centered} hold weakly and in every
moment, the limiting law has the moments \eqref{eq:firstmoments}, and
$s_1(A)/\sqrt h\to1$, all a.s. The uncentered
matrix $\sqrt{\ntoken/h}A$ has the same weak limit.
\end{cor}
\begin{proof}
Apply \cref{lem:replacement} with $T_n=I_{hn}/\sqrt h$ to compare
$\sqrt{n/h}\Aperp$ and $\YG/\sqrt h$.
The latter has conditional row covariance $\covG/(hn)$, and
\cref{app:attention-law} identifies its limiting law as $\attnlimit$.
The transform calculation gives \eqref{eq:firstmoments}.
The difference between $\sqrt{n/h}A$ and $\sqrt{n/h}\Aperp$ has
rank one, so their Gram distributions have the same weak limit.

For the leading singular value and its direction, put
$u_\ntoken=\mathbf1_\ntoken/\sqrt\ntoken$ and
$w_h=h^{-1/2}(u_\ntoken^\top,\ldots,u_\ntoken^\top)^\top$.
Then $\|w_h\|=1$, and row stochasticity gives
\[
 A=\sqrt h u_\ntoken w_h^\top+\Aperp,
   \Aperp w_h=0.
\]
Consequently,
\begin{equation}\label{eq:raw-gram-rank-one}
 AA^\top
   =hP+\Aperp(\Aperp)^\top.
\end{equation}
By \cref{lem:attention-norm}, $\|\Aperp\|
\le\ntoken^{-1/2+o(1)}$ almost surely. Hence
\[
 h\le\lambda_{\max}(AA^\top)
 \le h+\|\Aperp\|^2\longrightarrow h.
\]
The same rank-one identity gives
$s_2(A)\le\|\Aperp\|\to0$ by eigenvalue interlacing. A unit leading
left singular vector $x_1$ satisfies
$1-|u_\ntoken^\top x_1|^2\le\|\Aperp\|^2/h\to0$, proving alignment
with the all-ones direction.
\end{proof}

The identity $\|A\|_F^2=h+\|\Aperp\|_F^2$ and
$\|\Aperp\|_F^2/h=m_1(\nu_{\sqrt{\ntoken/h}\Aperp})\to\theta_1$ give the Frobenius
norm and stable rank limits, a.s.,
\begin{equation}\label{eq:stable-rank}
 \|A\|_F^2/h\longrightarrow e^{\beta^2},
 \frac{\|A\|_F^2}{s_1(A)^2}\longrightarrow e^{\beta^2},
\end{equation}
where $\beta^2$ is the variance of the logits.

\subsubsection{Independent Value and Output Projections}
\label{app:vo}
\begin{cor}[Gaussian value and output weights]\label{cor:vo}
Let $W_r^V\in\R^{d\times\dimvalue}$ have i.i.d.\ $\Normal(0,1/d)$
entries and $G_O\in\R^{h\dimvalue\times\dimout}$ have i.i.d.\
$\Normal(0,1)$ entries. The value family, $G_O$, and the family of
query, key, and noise matrices are mutually independent. Value
weights may be dependent or shared across heads. Then
almost surely, weakly and in every fixed moment,
\begin{equation}\label{eq:volaw}
 \nu_{\MHAperp(X)}\longrightarrow
 \attnlimit\boxtimes\MP_{1/\valueratio}\boxtimes\freepois_{\outputratio}.
\end{equation}
\end{cor}
\begin{prop}[Values without output weights]\label{prop:values-only}
Under the assumptions of \cref{cor:vo},
\[
 \ZV=[\Ahead{1}V_1,\ldots,\Ahead{h}V_h]
 \quad\text{satisfies}\quad
 \nu_{\ZV}\longrightarrow\attnlimit\boxtimes\freepois_{\valueratio}
\]
a.s., weakly and in every fixed moment.
\end{prop}
\begin{proof}
Put $\ZVG=\ntoken^{-1/2}\YG\Vall$. Given the keys and values,
its rows are i.i.d.\ $\Normal(0,\covV/\ntoken)$, where
\[
 \covV=\bigoplus_r \covVh{r},\qquad
 \covVh{r}=V_r^\top\covGh{r}V_r/\ntoken.
\]
The off-diagonal blocks vanish because the query and noise heads
are independent. Since $XX^\top/d=I_n$, each $V_r$ is standard
Gaussian and independent of the keys. Apply \cref{lem:gaussian-law}
to $V_r^\top$ with $D_m=\covGh{r}$, and then rescale from
denominator $\dimvalue$ to $n$.
This gives
\[
 \mu_{\covVh{r}}\longrightarrow\vlaw
 :=\dilate{\valueratio/h}{\freepois(\rho_{\keyratio/h},h/\valueratio)}.
\]
Gaussian norm bounds imply that the norms of $\covVh{r}$ are bounded.
The empirical eigenvalue distribution of $\covV$ is the average of
those of its blocks, so it has the same limit regardless of their
dependence across heads.
A second application to the independent-noise representation of $\ZVG$ gives
$\nu_{\ZVG}\to\freepois(\vlaw,\valueratio)$, with transform
\[
 \frac{\Stransform_{\vlaw}(w/\valueratio)}{\valueratio+w}
 =\frac{\Stransform_{\rho_{\keyratio/h}}(w/h)}
        {(1+w/h)(\valueratio+w)}
 =\Stransform_{\attnlimit}(w)\Stransform_{\freepois_{\valueratio}}(w).
\]
This proves the value law on the Gaussian side.

The full Gaussian output is exactly
$\ntoken^{-1/2}\MHAG(X)=\ZVG G_O/\sqrt{h\dimvalue}$.
Conditionally on $\ZVG$, its columns are i.i.d.\ with covariance
$D_n/(h\dimvalue)$, $D_n=\ZVG(\ZVG)^\top$, and $\|D_n\|=O(1)$ a.s.
Equation~\eqref{eq:row-pop-law} at width ratio $\outputratio$, together
with $n/(h\dimvalue)\to1/\valueratio$ and
$\mathcal D_{1/\valueratio}\freepois_{\valueratio}=\MP_{1/\valueratio}$, gives
\[
 \nu_{\ntoken^{-1/2}\MHAG(X)}\longrightarrow
 \attnlimit\boxtimes\MP_{1/\valueratio}\boxtimes\freepois_{\outputratio}.
\]
The relevant right multipliers are
$n^{-1/2}\Vall$ and $n^{-1/2}\Vall W^O$.
Both have bounded norms and satisfy \cref{ass:right-multipliers}.
Thus \cref{lem:replacement} transfers every fixed moment to the
original matrices. Since the limiting law is compactly supported and
therefore determined by its moments, this moment convergence also
implies weak convergence.
\end{proof}

For a compactly supported law with positive mean,
$\Stransform_\mu(w)=m_1^{-1}(1-(m_2/m_1^2-1)w+O(w^2))$.
Thus normalized variances add under $\boxtimes$ and are invariant
under dilation. The contribution of $\freepois_t$ is $1/t$;
that of $\attnlimit$ is
$1/h+\theta_2^2/(\keyratio\theta_1^2)$.
Consequently, for the limiting output law $\nu$,
\begin{equation}\label{eq:vo-second-moment}
 \frac{m_2(\nu)}{m_1(\nu)^2}-1
 =\frac1h+\frac1{\valueratio}+\frac1{\outputratio}
   +\frac{\theta_2^2}{\keyratio\theta_1^2}.
\end{equation}

\input{110_head_aggregation}

%% file: 093_spectral_transforms.tex
For a compactly supported probability law $\nu$ on $[0,\infty)$, define
the Cauchy transform
\begin{equation}\label{eq:cauchy-transform}
 G_\nu(z)=\int\frac{1}{z-x}\,\nu(dx),
 \qquad z\in\C\setminus\operatorname{supp}\nu,
\end{equation}
and the $\psi$-function
\begin{equation}\label{eq:psi-transform}
 \psi_\nu(z)=\frac{G_\nu(1/z)}{z}-1
           =\int\frac{zx}{1-zx}\,\nu(dx).
\end{equation}
The first expression requires $z\ne0$ and $1/z\notin\operatorname{supp}\nu$;
near zero, $\psi_\nu(z)=\sum_{q\ge1}m_q(\nu)z^q$.
For $m_1(\nu)>0$, write $\chi_\nu=\psi_\nu^{\langle-1\rangle}$ for the
local inverse at zero and for the real inverse on $(\nu(\{0\})-1,0)$.
The analytic $S$-transform is
\begin{equation}\label{eq:s-transform}
 \Stransform_\nu(w)=\frac{1+w}{w}\chi_\nu(w),
 \qquad \Stransform_\nu(0)=\frac1{m_1(\nu)},
\end{equation}
where the value at zero removes the apparent singularity.
For laws of positive mean, multiplicativity gives
$\Stransform_{\nu\boxtimes\mu}=\Stransform_\nu\Stransform_\mu$
\citep[Corollary 18.17]{nica2006lectures}.

We also use the $R$-transform
$R_\nu(z)=\sum_{q\ge1}\kappa_q(\nu)z^{q-1}$, whose constant term is
the first free cumulant. For positive mean, near zero it satisfies
\begin{equation}\label{eq:r-s-transform}
 \Stransform_\nu(w)R_\nu(w\Stransform_\nu(w))=1
\end{equation}
\citep[Definition 18.15 and Remark 18.16]{nica2006lectures}.

%% file: 110_head_aggregation.tex
\subsubsection{Head Aggregation}
\label{app:special-cases}
\begin{prop}[Head sum]\label{prop:sum}
Assume in addition that the key weights $(W_r^K)_r$ are jointly
independent. For $\bar A=h^{-1}\sum_rA_r$, the squared singular-value law of
$\sqrt{\ntoken h}(\bar A-P)$ converges almost surely, weakly and in every fixed
moment, to $\rho_{\keyratio}\boxtimes\freepois_1$.
\end{prop}
\begin{proof}
The head sum is the case $V_r=\sqrt\ntoken I_\ntoken$ and
$W^O=h^{-1/2}[I_\ntoken;\ldots;I_\ntoken]$ of \cref{thm:main}, with
$\dimvalue=\dimout=\ntoken$, $\valueratio=h$, $\outputratio=1$, and $G_O=\sqrt\ntoken[I_\ntoken;\ldots;I_\ntoken]$. Then
$\covVO=h^{-1}\sum_r\covGh{r}$, and
$h^{-1}\sum_r\covKh{r}=[K_1,\ldots,K_h][K_1,\ldots,K_h]^\top/(h\dimqk)$ is a
Gaussian Gram matrix of aspect ratio $h\dimqk/\ntoken\to\keyratio$ when the
keys are independent across heads, so $\mu_{\covVO}\to\rho_{\keyratio}$ and the
limit is $\freepois(\rho_{\keyratio},1)=\rho_{\keyratio}\boxtimes\freepois_1$.
\end{proof}
\Cref{lem:replacement} holds for any joint law of the keys. If all heads
share one key matrix $K$, the Gaussian replacement becomes
\[
 \frac{\beta\bigl(h^{-1/2}\sum_rQ_r\bigr)K^\top}{\sqrt{\ntoken\dimqk}}
 +\sqrt{\frac{\theta_1-\theta_2}{h\ntoken}}\sum_r\noiseh{r},
\]
and $h^{-1/2}\sum_rQ_r$ is a standard Gaussian matrix. This is the single-head model of width
$\dimqk$ and width-to-token ratio $\keyratio/h$, and the limit is
$\rho_{\keyratio/h}\boxtimes\freepois_1$.

%% file: 106_key_dependent_values.tex
\subsection{Key-Dependent Value Projections}
\label{app:key-dependent-values}
Throughout this subsection, $n\to\infty$, $\dimqk=\dimqk(n)\to\infty$,
$h\dimqk/n\to\keyratio\in(0,\infty)$, and $h$ and $\beta>0$ are fixed.

\subsubsection{The Spectral Law for Shared Key and Value Weights}
\label{app:shared-values-proof}
Set $\dimvalue=\dimqk$. With $W_r^V=W_r^K$ and $W^O=I_{h\dimqk}$, write
\begin{equation}\label{eq:shared-value-output}
 \ZKV=[\Ahead{1}K_1,\ldots,\Ahead{h}K_h].
\end{equation}
Define the limiting covariance polynomial
\begin{equation}\label{eq:phikv}
 \kvmap(x)=(\theta_1-\theta_2)x+
                    \frac{h\theta_2}{\keyratio}x^2.
\end{equation}
\begin{cor}[Shared key and value weights]\label{cor:shared-values}
Put $\kvlaw=\pushforward{\kvmap}{\freepois_{\keyratio/h}}$.
Then, a.s., weakly and in every fixed moment,
\begin{equation}\label{eq:shared-value-law}
 \nu_{\ZKV}\longrightarrow\freepois(\kvlaw,h),
\end{equation}
and in particular
\begin{equation}\label{eq:shared-value-first-moment}
 \frac{\|\ZKV\|_F^2}{n}\longrightarrow\keyratio\theta_1+h\theta_2.
\end{equation}
\end{cor}
\begin{proof}
Set $T_n=n^{-1/2}\Kall$, so that $\ZKV=\scaledA T_n$, and put
$\ZKVG=\YG T_n$.
Gaussian operator-norm bounds give $\|T_n\|=O(1)$ almost surely.
Since $T_n$ is measurable in the keys, the keys and $T_n$ are
jointly independent of the queries, as required by
\cref{ass:right-multipliers}. The comparison noise is independent
of these variables jointly, so \cref{lem:replacement} gives
$m_q(\nu_{\ZKV})-m_q(\nu_{\ZKVG})\to0$ a.s.\ for every fixed $q$.

Write $G_{\mathrm K,n}:=T_nT_n^\top=(\dimqk/n)\covK$ and
$\varphi_n(x)=(\theta_1-\theta_2)x+(n/\dimqk)\theta_2x^2$.
The left Gram of $\ZKVG$ equals that of $\YG G_{\mathrm K,n}^{1/2}$.
Conditionally on the keys, the latter matrix has independent rows
with covariance $\varphi_n(G_{\mathrm K,n})/n$, because
$\covG=(\theta_1-\theta_2)I+(n/\dimqk)\theta_2G_{\mathrm K,n}$.

The marginal Gaussian Gram law gives
$\mu_{G_{\mathrm K,n}}\to\freepois_{\keyratio/h}$, and its norm is bounded a.s.
Since $n/\dimqk\to h/\keyratio$, the matrix
$\varphi_n(G_{\mathrm K,n})$ has empirical eigenvalue law tending to $\kvlaw$ and bounded norm.
Represent $\YG G_{\mathrm K,n}^{1/2}$ as
$Z\varphi_n(G_{\mathrm K,n})^{1/2}/\sqrt n$ with standard Gaussian
$Z\in\R^{n\times hn}$ independent of the keys, on a probability-space
extension if needed.
Applying \cref{lem:gaussian-law} with $D_m=\varphi_n(G_{\mathrm K,n})$ and $hn/n=h$
therefore gives \eqref{eq:shared-value-law} for the Gaussian matrix,
including its zero eigenvalues.
The moment comparison above transfers every fixed moment to $\ZKV$.
Since the limiting law is compactly supported and therefore determined
by its moments, this also gives weak convergence.

For the first moment, since
$m_1(\freepois_{\keyratio/h})=\keyratio/h$ and
$m_2(\freepois_{\keyratio/h})=\keyratio/h+(\keyratio/h)^2$,
the defining cumulants of $\freepois(\kvlaw,h)$ give
\[
 m_1(\freepois(\kvlaw,h))
 =h\left[(\theta_1-\theta_2)\frac{\keyratio}{h}+
          \theta_2\left(1+\frac{\keyratio}{h}\right)\right]
 =\keyratio\theta_1+h\theta_2,
\]
which proves \eqref{eq:shared-value-first-moment}.
\end{proof}

The proof uses only the marginal law of each key matrix. Hence it
allows arbitrary dependence across heads, including a common key
and value matrix. The free cumulants of the output law are
\[
 \kappa_q=h\int\kvmap(x)^q\,d\freepois_{\keyratio/h}(x).
\]

\paragraph{Independent Gaussian output weights.}
\label{app:shared-values-output}
Let $G_O\in\R^{h\dimqk\times\dimout}$ be standard Gaussian,
independent of keys, queries, and comparison noise, and set
$W^O=G_O/\sqrt{h\dimqk}$, with $\dimout/n\to\outputratio$.
The normalized Gaussian output is
$n^{-1/2}\MHAG(X)=\ZKVG G_O/\sqrt{h\dimqk}$.
Given $\ZKVG$, its columns are independent with covariance
$D_n/(h\dimqk)$, $D_n=\ZKVG(\ZKVG)^\top$.
By Gaussian norm bounds, $\|D_n\|=O(1)$ a.s.
Equation~\eqref{eq:row-pop-law} and $n/(h\dimqk)\to1/\keyratio$ give
\begin{equation}\label{eq:shared-value-output-law}
 \nu_{\MHAperp(X)}\longrightarrow
 \dilate{1/\keyratio}{\freepois(\kvlaw,h)\boxtimes\freepois_{\outputratio}}.
\end{equation}
The convergence is a.s., weakly and in every fixed moment.
The multiplier $n^{-1/2}\Kall W^O$ has bounded norm, and the pair
consisting of this multiplier and the keys is independent of the queries.
Thus \cref{lem:replacement}
transfers the Gaussian limit to the attention output.

For the readout normalization, the limiting mean is
$m=\outputratio(\theta_1+h\theta_2/\keyratio)$.
By \eqref{eq:column-pop-S}, scaling \eqref{eq:shared-value-output-law}
by $1/m$ gives
\[
 \Stransform_\nu(z)=\frac{\keyratio m\Stransform_{\kvlaw}(z/h)}
 {(h+z)(\outputratio+z)}.
\]
Substitution into $\lambda w=(1-w)/\Stransform_\nu(-w)$ proves
\eqref{eq:readout-tied-root}. Since $\kvmap(x)>0$ for $x>0$,
$\kvlaw(\{0\})=\max\{1-\keyratio/h,0\}$.
For laws on $[0,\infty)$, the zero-atom rule
$(\mu\boxtimes\eta)(\{0\})=\max\{\mu(\{0\}),\eta(\{0\})\}$
\citep{belinschi2003atoms}, together with the row/column Gram relation
in \cref{lem:gaussian-law}, gives
\[
 \freepois(\kvlaw,h)(\{0\})
 =\max\{1-h[1-\kvlaw(\{0\})],0\}
 =\max\{1-\keyratio,0\}.
\]
Applying the same rule to \eqref{eq:shared-value-output-law}, and
noting that positive dilation preserves the mass at zero, yields
$\nu(\{0\})=1-\min\{1,\keyratio,\outputratio\}$.
Thus the range and strict monotonicity of $w_\nu$ from
\cref{app:effective-dimension} identify the unique root on the stated interval.

%% file: 109_projection_catalog.tex
\FloatBarrier
\Cref{tab:special-cases} collects value/output projection choices in
\cref{thm:main}. Write $\covVh{r}=V_r^\top\covGh{r}V_r/\ntoken$,
$\covV=\bigoplus_r \covVh{r}$, and $\bar A=h^{-1}\sum_r A_r$.
The stacked identity has $h$ blocks. The row for the head sum assumes
independent key matrices. If all heads share one key matrix, replace
$\rho_{\keyratio}$ with $\rho_{\keyratio/h}$.
The tied-value law $\kvlaw$ is defined in \cref{cor:shared-values}.

\begin{table}[ht]
\centering
\small
\renewcommand{\arraystretch}{1.2}
\setlength{\tabcolsep}{4pt}
\caption{Special cases of \cref{thm:main}.}
\label{tab:special-cases}
\sbox0{%
\begin{tabular}{@{}llllll@{}}
\toprule
Output & $V_r$ & $W^O$ & $\covVO$ & $\outputratio$ & Limit\\
\midrule
$\sqrt{\ntoken h}(\bar A-P)$ & $\sqrt\ntoken I_\ntoken$
 & $[I_\ntoken;\ldots;I_\ntoken]/\sqrt h$
 & $h^{-1}\sum_r\covGh{r}$ & $1$ & $\rho_{\keyratio}\boxtimes\freepois_1$\\
$\ZV$ & Gaussian & $I_{h\dimvalue}$ & $\covV$ & $\valueratio$
 & $\attnlimit\boxtimes\freepois_{\valueratio}$\\
$\MHAperp(X)$ & Gaussian & $G_O/\sqrt{h\dimvalue}$
 & $G_O^\top\covV G_O/(h\dimvalue)$ & $\outputratio$
 & $\attnlimit\boxtimes\MP_{1/\valueratio}\boxtimes\freepois_{\outputratio}$\\
$\ZKV$ & $K_r$ & $I_{h\dimqk}$
 & $\Kall^\top\covG\Kall/\ntoken$ & $\keyratio$ & $\freepois(\kvlaw,h)$\\
\bottomrule
\end{tabular}}%
\ifdim\wd0>\textwidth\resizebox{\textwidth}{!}{\usebox0}\else\usebox0\fi
\end{table}

%% file: 107_effective_dimension.tex
\subsection{Effective Degrees of Freedom per Token}
\label{app:effective-dimension}
Let $Z_n$ be a centered feature matrix with
$n$ rows whose squared singular-value law converges almost surely to
$\nu$, as in \cref{thm:main,prop:values-only,cor:shared-values}.
Write $G_n=Z_nZ_n^\top$. For a fixed feature realization and $\lambda>0$,
the ridge problem
\[
 \min_v\left\{\frac12\|Z_nv-y\|^2+\frac\lambda2\|v\|^2\right\}
\]
has fitted values $H_n(\lambda)y$, where
$H_n(\lambda)=G_n(G_n+\lambda I_n)^{-1}$.
Its degrees of freedom are
$\sum_i\partial(H_n(\lambda)y)_i/\partial y_i=\Tr H_n(\lambda)$.
Since $x/(x+\lambda)$ is bounded and continuous on $[0,\infty)$,
weak convergence gives, for every fixed $\lambda>0$,
\[
 \tr_n H_n(\lambda)\longrightarrow
 w_\nu(\lambda)=\int\frac{x}{x+\lambda}\,\nu(dx).
\]
Thus $w_\nu(\lambda)$ is the limiting effective degrees of freedom
per token, with $0\le w_\nu(\lambda)\le1$.
The transform identities are
\[
 w_\nu(\lambda)=-\psi_\nu(-1/\lambda)=1+\lambda G_\nu(-\lambda).
\]
For positive mean, $w_\nu$ decreases strictly from
$1-\nu(\{0\})$ to zero as $\lambda$ increases from zero to infinity.
By \eqref{eq:s-transform}, its inverse relation is
\[
 \lambda=\frac{1-w}{w\Stransform_\nu(-w)},
 \qquad w=w_\nu(\lambda).
\]

To compare spectral shape at the same first moment, assume
$m_1=m_1(\nu)>0$ and put
$\bar\nu=\pushforward{(x\mapsto x/m_1)}{\nu}$.
Using $G_n/m_1$ with penalty $\lambda$ is equivalent to using $G_n$ with
penalty $m_1\lambda$ for this normalized trace. All limiting laws below
are compactly supported, so the identity
\[
 \frac{x}{x+\lambda}
 =\frac{x}{\lambda}-\frac{x^2}{\lambda^2}
   +\frac{x^3}{\lambda^2(x+\lambda)}
\]
gives, with $\bar m_j=m_j(\bar\nu)$,
\begin{equation}\label{eq:effective-dimension-expansion}
 w_{\bar\nu}(\lambda)
 =\frac1\lambda-\frac{\bar m_2}{\lambda^2}+O(\lambda^{-3}),
 \lambda\to\infty.
\end{equation}
The remainder lies between zero and $\bar m_3/\lambda^3$.

For independent Gaussian key and value families with the across-head
sharing allowed in \cref{prop:values-only},
$\nu=\attnlimit\boxtimes\freepois_{\valueratio}$ has
\[
 m_1=\valueratio\theta_1,
 \bar m_2=1+\frac1h+\valueratio^{-1}+\keyratio^{-1}(\theta_2/\theta_1)^2.
\]
Thus, at fixed $\valueratio,\keyratio,\beta$ and two fixed head counts $h_2>h_1$,
\begin{equation}\label{eq:effective-dimension-heads}
 w_{\bar\nu_{h_2}}(\lambda)-w_{\bar\nu_{h_1}}(\lambda)
 =\frac{1/h_1-1/h_2}{\lambda^2}+O(\lambda^{-3})>0
\end{equation}
for sufficiently large $\lambda$. The analogous conclusion holds after
an independent Gaussian output projection at fixed $\outputratio$, which adds
$\outputratio^{-1}$ to $\bar m_2$.

Across-head sharing preserves the entire limiting
law in these examples and therefore preserves the effective degrees of freedom per token
at every fixed positive penalty.

To compare independent values and $V_r=K_r$ at $\beta=1$,
set $\dimvalue=\dimqk=n$ (so $\valueratio=\keyratio=h$) and omit the output
projection. For $\ZV=[\Ahead{1}V_1,\ldots,\Ahead{h}V_h]$,
independent values give
\[
 m_{1,\mathrm{ind}}=h(e-1),\qquad
 \bar m_{2,\mathrm{ind}}=1+\frac{2+(e-1)^{-2}}h.
\]
For tied values, \cref{cor:shared-values} gives
$\nu_{\mathrm{tied}}=\freepois(\kvlaw,h)$,
where $\kvlaw=\pushforward{\kvmap}{\freepois_1}$ and
$\kvmap(x)=(e-2)x+x^2$.
Using the first four moments $1,2,5,14$ of $\freepois_1$
\citep[Proposition 12.11]{nica2006lectures} gives
$m_1(\kvlaw)=e$ and $m_2(\kvlaw)=2e^2+2e+2$.
Equation~\eqref{eq:compound-moments} then gives
\[
 m_{1,\mathrm{tied}}=he,\qquad
 \bar m_{2,\mathrm{tied}}=1+\frac{2+2e^{-1}+2e^{-2}}h.
\]
The tied model has a larger first moment and a larger normalized second
moment for every fixed $h$. By \eqref{eq:effective-dimension-expansion},
after matching first moments its effective degrees of freedom per token
are smaller for sufficiently strong regularization.

%% file: 117_readout_observables.tex
\subsection{Training Loss and Estimation Error from the Bulk Law}
\label{app:readout-observables}
Let $\Phi=\MHA(X)$ or $\Phi=[A_1V_1,\ldots,A_hV_h]$.
Let $\Phi_\perp$ denote the matrix obtained from $\Phi$ by replacing
each $A_r$ with $\Ahead{r}=A_r-P$.
We use the deterministic normalization
\begin{equation}\label{eq:readout-normalization}
 F=\Phi/\sqrt m,\qquad
 m=\lim_{\ntoken\to\infty}\frac{\|\Phi_\perp\|_F^2}{\ntoken}>0,
\end{equation}
where the limit is almost sure. Thus the limiting average squared
row norm of the centered features is one, while the common
component remains in the features used for learning.

Throughout the derivations below, $F_n$ has $n$ rows and
$p_n$ columns, its left-Gram empirical law $\nu_n$ converges almost
surely to a deterministic probability law $\nu$, and any feature
normalization is deterministic. The uncentered attention features
have the same weak limit as their centered counterparts because
their difference has rank at most one. No convergence of uncentered
spectral moments is needed for the bounded filters below.

\subsubsection{Fixed-Target Gradient Flow}
Fix $y$ independently of the network weights, with
$\mathbf1_n^\top y=0$ and $\|y\|^2=n$, and hold it fixed across
feature initializations. The relative training loss is
\begin{equation}\label{eq:readout-flow-finite}
 L_n(t;y)=\frac{\|F_n\vartheta(t)-y\|^2}{\|y\|^2}.
\end{equation}
For the objective $\frac12\|F_n\vartheta-y\|^2$ with $\vartheta(0)=0$, the residual
satisfies $\dot r=-F_nF_n^\top r$ and $r(0)=-y$. This gives
the loss in \eqref{eq:readout-flow-finite} as
\[
 L_n(t;y)=\frac{y^\top e^{-2tF_nF_n^\top}y}{\|y\|^2}.
\]
We show that, for every fixed $t\ge0$,
\begin{equation}\label{eq:readout-flow-limit}
 \lim_{n\to\infty}\E L_n(t;y)
 =\int e^{-2tx}\,d\nu(x).
\end{equation}

Let
\[
 B_n(t)=\E e^{-2tF_nF_n^\top},\qquad
 P=\mathbf1_n\mathbf1_n^\top/n.
\]
Under the orthogonal-input assumptions, the query rows are independent
and identically distributed and independent of the key, value, and
output family. Conditional on that family, permuting the query rows
permutes the rows of either feature matrix defined above.
Thus $B_n(t)$ commutes with every
permutation matrix. For $n\ge2$, it has the form
\[
 B_n(t)=a_n(t)(I_n-P)+c_n(t)P,
 \qquad 0\le a_n(t),c_n(t)\le1,
\]
since the exponential is a positive contraction. Every deterministic
$y$ with $Py=0$ therefore satisfies
\[
 \E L_n(t;y)=a_n(t),\qquad
 \frac1n\Tr B_n(t)=\frac{n-1}{n}a_n(t)+\frac1n c_n(t).
\]
In particular,
\begin{equation}\label{eq:readout-exchangeability-bound}
 \left|\E L_n(t;y)
 -\E\int e^{-2tx}\,d\nu_n(x)\right|
 \le\frac1n.
\end{equation}

For fixed $t\ge0$, weak convergence and bounded convergence now give
\eqref{eq:readout-flow-limit}.

\subsubsection{Ridge Estimation with an Isotropic Teacher}
For ridge regression, generate
\[
 y=\sqrt n F_n\vartheta+\varepsilon,\qquad
 \vartheta\sim\Normal(0,I_{p_n}/p_n),\qquad
 \varepsilon\sim\Normal(0,\sigma^2 I_n),
\]
with $\vartheta$, $\varepsilon$, and $F_n$ mutually independent.
The estimator $\vartheta_*$ minimizes
$n^{-1}\|\sqrt n F_nv-y\|^2+\lambda\|v\|^2$ over $v$, for $\lambda>0$.
We evaluate
$\mathbb E_{\vartheta,\varepsilon}
[\|\vartheta_*-\vartheta\|^2\mid F_n]$.

For a fixed $F_n$, put $C_n=F_n^\top F_n$.
The normal equations give
\[
 \vartheta_*-\vartheta
 =-\lambda(C_n+\lambda I)^{-1}\vartheta
  +\frac1{\sqrt n}(C_n+\lambda I)^{-1}F_n^\top\varepsilon.
\]
Independence and the stated covariances give the exact conditional
risk
\[
 \mathbb E_{\vartheta,\varepsilon}
 [\|\vartheta_*-\vartheta\|^2\mid F_n]
 =\frac{\lambda^2}{p_n}\Tr(C_n+\lambda I)^{-2}
 +\frac{\sigma^2}{n}\Tr C_n(C_n+\lambda I)^{-2}.
\]

The nonzero spectra of $C_n$ and $F_nF_n^\top$ agree. Accounting
for the different numbers of zero eigenvalues and writing
$\gamma_n=p_n/n$, this equals
\[
 1-\gamma_n^{-1}
 +\gamma_n^{-1}\int\frac{\lambda^2}{(x+\lambda)^2}\,d\nu_n(x)
 +\sigma^2\int\frac{x}{(x+\lambda)^2}\,d\nu_n(x).
\]

For fixed $\lambda>0$, both integrands are
bounded and continuous on $[0,\infty)$, so almost-sure weak
convergence and $\gamma_n\to\gamma>0$ give
\begin{equation}\label{eq:readout-risk}
 \begin{aligned}
 \mathcal{E}_\lambda&:=\lim_{n\to\infty}\mathbb E_{\vartheta,\varepsilon}
 [\|\vartheta_*-\vartheta\|^2\mid F_n]\\
 &=1-\gamma^{-1}
 +\gamma^{-1}\int\frac{\lambda^2}{(x+\lambda)^2}\,d\nu(x)
 +\sigma^2\int\frac{x}{(x+\lambda)^2}\,d\nu(x).
 \end{aligned}
\end{equation}
For the full MHA output, $p_n=\dimout$ and
$\gamma=\outputratio$, yielding \eqref{eq:readout-risk-mha}.

%% file: 115_additional_derivations.tex
\subsection{Scalar Equations for Spectra and Readouts}
\label{app:additional-derivations}
\input{111_cauchy_equation}

%% file: 111_cauchy_equation.tex
\subsubsection{The Normalized Cauchy Equation}
\label{app:cauchy}
For attention and the independent key/value laws of \cref{app:vo},
let $c_\beta=\theta_2/\theta_1=\beta^2/(e^{\beta^2}-1)\in(0,1)$.
To compare spectral shapes, normalize each squared-singular-value law
to have mean one: divide by $\theta_1$ for attention, by
$\valueratio\theta_1$ for the value output, and by
$\outputratio\theta_1$ for the MHA output.

The mean-one affine MP law is
$\rho_{\beta,t}=\pushforward{(x\mapsto1-c_\beta+c_\beta x)}{\MP_{1/t}}
=\mathcal D_{1/\theta_1}\rho_t$, where $\rho_t$ is defined in \eqref{eq:rho}.
Its $R$-transform is
$R_{\rho_{\beta,t}}(z)=1-c_\beta+c_\beta/(1-c_\beta z/t)$.
For $s=\Stransform_{\rho_{\beta,\keyratio}}(-w)$, the relation
$1/s=R_{\rho_{\beta,\keyratio}}(-ws)$ gives
\[
 1=(1-c_\beta w/\keyratio)s+c_\beta(1-c_\beta)ws^2/\keyratio.
\]

For a PSD matrix $C$, $G_{\mu_C}(z)=\tr_n[(zI_n-C)^{-1}]$.
For the final mean-one law $\nu$, put
$w=w_\nu(\lambda)=1+\lambda G_\nu(-\lambda)=-\psi_\nu(-1/\lambda)$, and
define $\Pi(z)=(1+z)\allowbreak(1+z/h)\allowbreak
(1+z/\valueratio)\allowbreak(1+z/\outputratio)$ for MHA.
For attention, omit both V/O factors; for values without an output
projection, omit only the O factor.
Multiplicativity and \eqref{eq:s-transform} give
$\chi_\nu(-w)=-ws/\Pi(-w)=-1/\lambda$, hence
$s=\Pi(-w)/(\lambda w)$ and
\begin{equation}\label{eq:cauchy-polynomial}
 \lambda^2w-\lambda(1-\keyratio^{-1}c_\beta w)\Pi(-w)-\keyratio^{-1}c_\beta(1-c_\beta)\Pi(-w)^2=0.
\end{equation}
The branch is the transform of the positive probability law already
constructed above and satisfies $w\sim1/\lambda$ at infinity.
The polynomial identity extends to $\C\setminus(-\operatorname{supp}\nu)$
by analyticity, including any zeros of $w$.

For attention, the first normalized moments are
\[
 m_1=1,  m_2=1+h^{-1}+\keyratio^{-1}c_\beta^2,
\]
\[
 m_3=1+3(h^{-1}+\keyratio^{-1}c_\beta^2)+h^{-2}+3\keyratio^{-1}c_\beta^2/h+\keyratio^{-2}c_\beta^3.
\]
These provide independent checks of the numerical moment recursion.

\subsubsection{Computing Readout Predictions}
\label{app:readout-computation}
For any limiting readout law $\nu$ of positive mean,
$w=w_\nu(\lambda)=-\psi_\nu(-1/\lambda)$ and
\eqref{eq:s-transform} give
$\lambda w=(1-w)/\Stransform_\nu(-w)$.
Also, $w_\nu'(\lambda)=-\int x/(x+\lambda)^2\,d\nu(x)$.
Using $\int\lambda^2/(x+\lambda)^2\,d\nu(x)
=1-w+\lambda w_\nu'(\lambda)$
in \eqref{eq:readout-risk-mha} proves \eqref{eq:readout-scalar-risk}.

For the independent Gaussian value and output weights of
\cref{cor:vo} and $\lambda>0$, the branch $w=w_\nu(\lambda)$ satisfies
\eqref{eq:readout-ridge-root}, or
$\lambda w=\Pi(-w)/\Stransform_{\rho_{\beta,\keyratio}}(-w)$.
The root lies in $0<w<\min\{1,h,\valueratio,\outputratio\}$.
The explicit $S$-transform shows that
$1/\Stransform_{\rho_{\beta,\keyratio}}(-w)$ is positive and nonincreasing
on this interval, while $\Pi(-w)$ decreases from one to zero.
Thus the right-hand side decreases strictly from one to zero,
whereas $\lambda w$ increases, proving uniqueness.
We solve this real scalar equation by bisection or Brent's method.
Implicit differentiation at fixed architecture parameters yields
\begin{equation}\label{eq:readout-root-derivative}
 w_\nu'(\lambda)
 =-\int\frac{x}{(x+\lambda)^2}\,d\nu(x)
 =\frac{-w}{\lambda-\partial_w[\Pi(-w)/\Stransform_{\rho_{\beta,\keyratio}}(-w)]}.
\end{equation}
The coefficient error in \eqref{eq:readout-risk-mha} can therefore be
evaluated from a scalar root and its implicit derivative.

For complete within-head K/V tying, use the law in
\eqref{eq:shared-value-output-law}, normalized by
$m=\outputratio(\theta_1+h\theta_2/\keyratio)$.
Write $F_{\mathrm{KV}}(w)$ for the right-hand side of
\eqref{eq:readout-tied-root}. We solve
$\lambda w=F_{\mathrm{KV}}(w)$ on
$0<w<\min\{1,\keyratio,\outputratio\}$ and use
$w_\nu'(\lambda)=-w/[\lambda-F_{\mathrm{KV}}'(w)]$.
Thus the tied prediction uses the same scalar-root procedure and
implicit differentiation as the independent-K/V prediction.

The same solver gives the Laplace transform of the limiting
training loss $\overline L(t):=\int e^{-2tx}\,d\nu(x)$
with respect to time:
\begin{equation}\label{eq:readout-loss-laplace}
 \int_0^\infty e^{-st}\overline L(t)\,dt
 =\int\frac{1}{s+2x}\,d\nu(x)
 =\frac{1-w_\nu(s/2)}{s},\qquad s>0.
\end{equation}
Solving \eqref{eq:readout-ridge-root} at $\lambda=s/2$ and numerically
inverting this transform produces the curves in
\cref{fig:readout-overview}(a); solving $\overline L(t)=0.1$ gives
the arrival times in \cref{fig:sharing-comparison}(a). The inversion checks are reported in
\cref{app:head-learning-experiment,app:cross-head-learning-experiment}.

%% file: 098_experiments.tex
\section{Experimental Details and Additional Comparisons}
\label{app:experiments}
\paragraph{Computing infrastructure and data.}
Numerical experiments were run on local CPUs and an NVIDIA GeForce
RTX 5090 GPU. All inputs, random weights, and targets were generated
synthetically. Public models supply architecture dimensions; no
pretrained weights or external datasets were used.

\paragraph{Readout setup.}\label{app:readout-experiment-protocol}
The readout experiments use $X=\sqrt d[I_\ntoken,0]$ and
centered Gaussian projection weights with i.i.d.\ entries of
variance $1/d$ for Q/K/V
and $1/(h\dimvalue)$ for O. Projections are independent except
for the sharing specified in each experiment. The attention
weights are frozen, and the linear readout uses
$F=\MHA(X)/\sqrt m$ with $m$ from \eqref{eq:readout-normalization}.
Unless varied, $\beta=1$, $\lambda=10^{-4}$, and $\sigma^2=0.1$.

\paragraph{Numerical methods and precision.}\label{app:numerical-precision}
Feature generation, Gram eigendecompositions, and finite-matrix readout
measurements use IEEE float64 arithmetic, including the GPU computations.
For independent K/V, the scalar ridge equation
\eqref{eq:readout-ridge-root} is solved by Brent's method with absolute
tolerance $5\times10^{-15}$. For complete K/V tying,
the covariance integrals use 96-point Gauss--Legendre quadrature
after the Marchenko--Pastur substitution
$x=1+\dimqk/\ntoken-2\sqrt{\dimqk/\ntoken}\cos\theta$,
$0<\theta<\pi$.
The scalar root uses absolute and relative tolerances
$2\times10^{-14}$.
Derivatives for the coefficient error are evaluated by implicit
differentiation.

The theoretical learning curves invert
\eqref{eq:readout-loss-laplace} with the Gaver--Stehfest method
\citep{stehfest1970inversion}, using 32 terms and 89 decimal digits
of working precision. The higher precision controls cancellation
in the alternating sum; transform evaluations solve the scalar
equation by bisection at the same precision. The comparison calculation
uses 48 terms and 121 digits. Theoretical arrival times use Brent's
method with absolute and relative tolerances $10^{-10}$;
finite-matrix and bootstrap arrival times use relative tolerance
$10^{-12}$ with absolute tolerance set to the smallest positive
normal float64 value.

\input{122_section5_uncertainty}
\input{123_public_model_finite_ge}
\input{119_head_learning_experiment}
\input{118_head_budget_experiment}
\input{126_temperature_width_experiment}
\input{121_cross_head_learning_experiment}
\input{127_sharing_head_count_experiment}
\input{098_mha_experiments}

\FloatBarrier

%% file: 122_section5_uncertainty.tex
\paragraph{Uncertainty conventions.}
All error bars and tabulated $\pm$ values report nominal pointwise
95\% confidence intervals (CIs) over independent initializations.
These intervals quantify uncertainty in the initialization-averaged
measurements, conditional on the fixed target in the gradient-flow experiments.
For coefficient errors, teacher coefficients and observation noise are
averaged analytically; the intervals reflect initialization variability.
For means over $N_{\mathrm{rep}}$ draws, we use Student-$t$ intervals
with half-width $t_{N_{\mathrm{rep}}-1,0.975}s/\sqrt{N_{\mathrm{rep}}}$,
where $s$ is the sample standard deviation and $t_{\nu,p}$ is the
$p$-quantile of the Student-$t$ distribution with $\nu$ degrees of freedom;
tabulated $\pm$ values give this half-width.
We use ten draws per condition.
The coefficient-error ratios in \cref{fig:readout-overview}(d) use
Fieller intervals (\cref{app:temperature-width-experiment}).
For arrival times in \cref{fig:sharing-comparison}(a), we instead use the
2.5th and 97.5th percentiles of $2{,}000$ bootstrap resamples of the ten
initializations, keeping the target fixed and solving for the arrival
time of each resampled mean curve.

We measure an error bar by the largest distance from a plotted estimate
to either interval endpoint.
The endpoint distance is the half-width for symmetric Student-$t$
intervals and also accommodates asymmetric intervals.
\Cref{tab:section5-uncertainty} gives upper bounds on these maxima over
the plotted points, rounded upward; its relative column divides each distance by
its own estimate before taking the maximum.

\begin{table}[htbp]\centering
\caption{Maximum error bars in the learning figures.
Loss, coefficient-error, and error-ratio distances are dimensionless;
arrival-time distances use the gradient-flow time units.}
\label{tab:section5-uncertainty}
\begin{tabular}{llrr}\toprule
Figure & Quantity or path & Absolute & Relative (\%)\\\midrule
\ref{fig:readout-overview}(a) & Relative training loss & $0.00513$ & $13.314$\\
\ref{fig:readout-overview}(c) & Coefficient error & $0.02381$ & $0.562$\\
\ref{fig:readout-overview}(d) & Coefficient-error ratio & $0.00168$ & $0.108$\\
\ref{fig:sharing-comparison}(a) & Arrival time, fixed widths & $4.43476$ & $10.425$\\
\ref{fig:sharing-comparison}(a) & Arrival time, fixed parameter count & $2.75124$ & $6.486$\\
\ref{fig:sharing-comparison}(b) & Coefficient error & $0.03015$ & $1.182$\\
\bottomrule\end{tabular}
\end{table}

The arrival-time rows both include the common unshared baseline.
Among the six shared configurations at fixed parameter count,
the largest distance is $0.40314$ ($3.365\%$ relative).
The largest relative interval for the learning curves occurs at
$h=128$, $t=100$, where the mean loss is $7.39\times10^{-5}$;
its absolute half-width is $9.84\times10^{-6}$.
Detailed protocols and finite-size discrepancies follow.

%% file: 123_public_model_finite_ge.tex
\subsection{Finite-Size GE at Dimensions from Publicly Available Models}
\label{app:public-finite-ge}
We test whether the Gaussian equivalent already approximates
softmax at finite dimensions taken from publicly available models.
\Cref{tab:public-finite-ge} retains their input and projection
widths, head counts, and K/V sharing, with $\ntoken\le d$.
For each row, we use a fixed input $XX^\top/d=I_\ntoken$,
fan-in Gaussian weights, and $\beta=1$, with ten independent
initializations.

\input{124_public_model_ge_table}

\paragraph{Sharing and lengths.}
Here $h$ counts query heads and $h_{\mathrm{KV}}$ counts K/V
heads. Rows marked ``None'' use independent K/V and independent heads.
Falcon's MQA shares one independent K/V pair across all query
heads, while Kayyam ties $W_r^K=W_r^V$ within each head.
The selected lengths are BERT's final training stage
\citep{devlin2019bert}, ModernBERT's initial stage using global
attention widths \citep{warner2024modernbert}, one 256-pixel
DINOv3 global crop \citep{simeoni2025dinov3}, LLaDA's main
training length \citep{nie2025llada}, GPT-3's training context
\citep{brown2020gpt3}, and Falcon's published sequence length
\citep[Table~16]{almazrouei2023falcon}. Kayyam uses the 2048-token
evaluation setting \citep[Table~16]{kayyam2026projections}.
The table includes one configuration per model family.

\paragraph{Comparison and metrics.}
The softmax and finite GE matrices use the same Q/K/V and output
weights. We compare all $\ntoken$ eigenvalues of
$\MHAperp(X)\MHAperp(X)^\top$ and
$\ntoken^{-1}\MHAG(X)\MHAG(X)^\top$ from \cref{thm:main}, including the largest.
Centering subtracts the uniform attention component before
applying the values. The GE noise is independent between query
heads and independent of all projections.
The relative $W_1$ is the mean absolute difference of ordered
eigenvalues divided by their common limiting mean:
$(\dimout/\ntoken)(e-1)$ for independent K/V, including MQA,
and $(\dimout/\ntoken)(e-1+\ntoken/\dimqk)$ for $K=V$.
KS is the largest absolute difference between the empirical
CDFs.
The table measures finite softmax--GE spectral agreement;
the main figures separately test the limiting predictions for
loss and coefficient error.

\paragraph{Dimensions of the learning experiments.}
\Cref{tab:learning-dimensions} lists the actual softmax
dimensions in the readout comparisons.
The head-count and across-head-sharing experiments take head
width 64 from the public configurations in \cref{tab:public-finite-ge}.
Choosing $\ntoken=4096$ makes the total value width equal to the
token count at $h=64$.
The temperature comparison uses GPT-3's projection dimensions with
an extended token length. The comparison of across-head sharing
with within-head tying fixes
total widths and varies $\dimqk=\dimvalue$ from 1024 to 64,
ending at the head width 64 used in several of these public models.
These choices combine public projection dimensions with controlled
variations of the dimension ratios.

\begin{table}[htbp]\centering\small
\caption{Dimensions used in the learning experiments.
Ranges cover the sampled softmax configurations.}
\label{tab:learning-dimensions}
\begin{tabular}{@{}lrrrrrrr@{}}\toprule
Experiment & $\ntoken$ & $d$ & $h$ & $\dimqk$ & $\dimvalue$ & $h\dimvalue$ & $\dimout$\\\midrule
\ref{fig:readout-overview}(a), learning curves & 4096 & 4096 & 32--128 & 64 & 64 & 2048--8192 & 8192\\
\ref{fig:readout-overview}(c), fixed output & 4096 & 4096 & 16--120 & 64 & 64 & 1024--7680 & 2048\\
\ref{fig:readout-overview}(c), fixed budget & 4096 & 4096 & 32--120 & 64 & 64 & 2048--7680 & 819--36864\\
\ref{fig:readout-overview}(d), temperature & 8192 & 12288 & 96 & 64--128 & 128--192 & 12288--18432 & 12288\\
\ref{fig:sharing-comparison}(a), across-head sharing & 4096 & 4096 & 64 & 64--106 & 64--106 & 4096--6784 & 8192\\
\ref{fig:sharing-comparison}(b), sharing schemes & 4096 & 8192 & 2--32 & 64--1024 & 64--1024 & 2048 & 2048\\
\bottomrule\end{tabular}
\end{table}

%% file: 124_public_model_ge_table.tex
% Generated by mha-paper render from public experiment summaries.
\begin{table}[htbp]\centering\footnotesize
\caption{Finite-size GE at dimensions taken from publicly available models, with $n\le d$.
Centered MHA--GE spectral distances at $\beta=1$, with $\dimout=d$.
$^\ast$Kayyam uses the 2048-token evaluation setting.}
\label{tab:public-finite-ge}
\begingroup\setlength{\tabcolsep}{3pt}
\begin{tabular}{@{}llrrrrcc@{}}\toprule
Dimensions of & Sharing & $n$ & $d$ & $h/h_{\mathrm{KV}}$ & $\dimqk=\dimvalue$ & Rel. $W_1$ (\%) & KS\\\midrule
BERT-base & None & 512 & 768 & 12/12 & 64 & $1.66 \pm 0.23$ & $0.0109 \pm 0.0010$ \\
ModernBERT-large & None & 1,024 & 1,024 & 16/16 & 64 & $1.37 \pm 0.21$ & $0.0052 \pm 0.0005$ \\
DINOv3-L/16 & None & 261 & 1,024 & 16/16 & 64 & $2.82 \pm 0.40$ & $0.0280 \pm 0.0018$ \\
LLaDA-8B & None & 4,096 & 4,096 & 32/32 & 128 & $0.92 \pm 0.06$ & $0.0022 \pm 0.0001$ \\
GPT-3 175B & None & 2,048 & 12,288 & 96/96 & 128 & $0.62 \pm 0.07$ & $0.0040 \pm 0.0004$ \\
Falcon-7B & MQA & 2,048 & 4,544 & 71/1 & 64 & $1.92 \pm 0.46$ & $0.0061 \pm 0.0012$ \\
Kayyam 1.2B$^\ast$ & $K=V$ & 2,048 & 2,048 & 32/32 & 64 & $0.24 \pm 0.01$ & $0.0020 \pm 0.0001$ \\
\bottomrule\end{tabular}\endgroup
\end{table}

%% file: 119_head_learning_experiment.tex
\subsection{Head-Count Training Curves}
\label{app:head-learning-experiment}
For \cref{fig:readout-overview}(a), we use $d=\ntoken=4096$,
$\dimqk=\dimvalue=64$, $\dimout=2\ntoken$, and
$h\in\{32,64,128\}$, with the setup in
\cref{app:readout-experiment-protocol}.
We fit an additional linear readout from the full, uncentered MHA
features, scaled by $1/\sqrt{2(e-1)}$ as in
\eqref{eq:readout-normalization}. Per-head and output widths are
fixed; total feature widths and projection parameter counts
increase with $h$.

Before drawing any network weights, we draw one standard Gaussian
vector, subtract its sample mean, and rescale it to have squared
norm $\ntoken$. This target $y$ is held fixed across all conditions
and initializations. We use zero-initialized readout weights and
unregularized continuous-time gradient flow. For each head count,
ten independent softmax feature matrices give the exact losses
\[
 L_\ntoken(t;y)=\frac1\ntoken
 \sum_j e^{-2t\ell_j}|u_j^\top y|^2,
 \qquad FF^\top u_j=\ell_j u_j.
\]
Thus the measurements retain the alignment of the fixed target
with the empirical eigenvectors.
Implementation tests compare the spectral formula with direct
matrix-exponential calculations.
Over the plotted times and three conditions, the largest absolute
difference between the measured mean relative loss and the limiting
prediction is below $0.003$.

The theoretical curves use the bulk law in \eqref{eq:volaw},
rescaled by $1/[2(e-1)]$, with
$\keyratio=\valueratio=h/64$ and $\outputratio=2$.
We evaluate \eqref{eq:learning-curve-bulk} by the inverse-Laplace
method in \cref{app:numerical-precision}. Checks at $t=1,10,100$
compare 32 and 48 terms and independent spectral quadrature.
For $h=32$ the limiting law has mass $1/2$ at zero, yielding
the loss plateau. The other two conditions have no atom at zero;
the displayed curves compare their different decay rates.

%% file: 118_head_budget_experiment.tex
\subsection{Head Count and Projection Budget}
\label{app:head-budget-experiment}
\Cref{fig:readout-overview}(b,c) evaluates the ridge coefficient error in
\cref{sec:readout-setting}. We fix $d=\ntoken=4096$,
$\dimqk=\dimvalue=64$, with the setup in
\cref{app:readout-experiment-protocol}.
The baseline head width 64 follows BERT and ModernBERT
\citep{devlin2019bert,warner2024modernbert}; $\ntoken=4096$
places the value-width boundary at $h=64$.
The features are the full output divided by
$\sqrt{\outputratio(e-1)}$, as in
\eqref{eq:readout-normalization}.

The map varies $h$ and $\outputratio$ with
$\keyratio=\valueratio=h/64$.
The fixed-output path has $\outputratio=1/2$.
The fixed-budget path holds the number of query, key, value, and
output projection weights at $6\ntoken^2$, giving
\[
 h\dimqk(3\ntoken+\dimout)=6\ntoken^2,
 \qquad \outputratio=6/\valueratio-3.
\]
The width equalities are $h\dimvalue=\dimout$ at $h=32$ on the
first path, and $h\dimvalue=\ntoken$ at $h=64$ and
$\dimout=\ntoken$ at $h=96$ on the second.
The dotted map boundaries indicate changes in which width is
smallest. Colors represent the continuous estimation error.

The curves and map use the limiting formula
\eqref{eq:readout-risk}, with the law in \eqref{eq:volaw} rescaled
by $1/[\outputratio(e-1)]$. Curves between integer head counts
interpolate the proportional-limit formula.
For each sampled integer head count, we generate ten independent
softmax feature matrices. Each point averages the exact conditional
coefficient error over these initializations, using the finite-matrix
formula in \cref{app:readout-observables}; the teacher and noise
are averaged analytically. Output widths are
rounded to the nearest integer, and numerical checks of the theory
use the realized widths. Teacher coefficients have covariance
$I_{\dimout}/\dimout$, so their dimension varies with output width.
Across the 19 configurations, the largest relative difference
between the measured mean coefficient error and GE is $1.72\%$.

%% file: 126_temperature_width_experiment.tex
\subsection{Temperature and Q/K--V Width Allocation}
\label{app:temperature-width-experiment}
\Cref{fig:readout-overview}(d) fixes $h=96$, $\ntoken=8192$,
and $d=\dimout=3\ntoken/2$, and compares
$(\dimqk,\dimvalue)=(128,128)$ with $(64,192)$.
These dimensions use the projection widths of GPT-3
\citep{brown2020gpt3}, with a longer, hypothetical sequence and
a controlled reallocation of widths. We use the setup in
\cref{app:readout-experiment-protocol}.
The full uncentered softmax output is scaled by
$1/\sqrt{\outputratio(e^{\beta^2}-1)}$ at each inverse temperature.

\paragraph{Equal parameter counts.}
For independent heads, the distinct projection weight count is
\[
 \nparam=2dh\dimqk+dh\dimvalue+h\dimvalue\dimout
         =2dh(\dimqk+\dimvalue),
 \qquad d=\dimout.
\]
Both allocations have $\dimqk+\dimvalue=256$, hence exactly the
same number of Q/K/V/O weights and the same total K/V storage width.

\paragraph{Crossover equation.}
Write $\mathcal{E}_\lambda(\beta;\dimqk,\dimvalue)$ for \eqref{eq:readout-risk-mha}
at inverse temperature $\beta$ and per-head widths $(\dimqk,\dimvalue)$,
with the fixed ridge parameter and noise variance above.
The crossover $\beta_{\mathrm c}$ in the displayed range satisfies
\begin{equation}\label{eq:temperature-crossover}
 \mathcal{E}_\lambda(\beta_{\mathrm c};128,128)=\mathcal{E}_\lambda(\beta_{\mathrm c};64,192).
\end{equation}
We compute it using the scalar solver in \cref{app:readout-computation}.
For each allocation $(\dimqk,\dimvalue)$, we set
$\keyratio=h\dimqk/\ntoken$, $\valueratio=h\dimvalue/\ntoken$,
$\outputratio=3/2$, and solve \eqref{eq:readout-ridge-root} before
evaluating \eqref{eq:readout-scalar-risk}.
Solving \eqref{eq:temperature-crossover} in the displayed range gives
$\beta_{\mathrm c}\approx0.559$.

\paragraph{Theory and measurements.}
The curve shows the ratio
$\mathcal{E}_\lambda(\beta;64,192)/\mathcal{E}_\lambda(\beta;128,128)$,
evaluating \eqref{eq:readout-risk-mha} independently of the measurements.
A ratio below one favors wider V at the same inverse temperature and
parameter count.

The plotted softmax points use
$\beta\in\{0.25,0.5,0.75,1,1.25,1.5\}$, with ten initializations
per point and exact conditional averages over teacher coefficients
and observation noise. Within each width allocation and initialization,
the same Q/K/V/O weights are used across temperatures.
All $\ntoken$ Gram eigenvalues enter the finite-matrix ridge error.
Each plotted point is the ratio of the two sample means.
Its nominal 95\% Fieller interval consists of ratios $r$ satisfying
$|\bar Y-r\bar X|\le t_{9,0.975}s_{Y-rX}/\sqrt{10}$,
where $X$ and $Y$ are the replicate-level errors for $(128,128)$ and
$(64,192)$, respectively, and $s_{Y-rX}$ is their sample standard deviation
after forming $Y-rX$ within each replicate.

Across the six plotted ratios, the largest relative difference
between the softmax estimate and GE is $1.52\%$ (rounded upward).

The two allocations have total Q/K-to-token ratios $1.5$ and $0.75$,
respectively, while both total V widths exceed $\ntoken$.
At low inverse temperature, the narrow-Q/K model has weak directions
associated with its smaller Q/K span. As $\beta$ increases,
$\beta^2/(e^{\beta^2}-1)$ decreases and the nonlinear contribution
to the normalized spectrum grows. The ridge coefficient error reflects
both this spectral change and the wider V projections.

\paragraph{Relation to existing width choices.}
The original Transformer uses $\dimqk=\dimvalue=d/h$ as a
computationally balanced default, but its Table~3(B) also tests
$\dimqk=16,32$ at $\dimvalue=64$ \citep{vaswani2017attention}.
Those smaller-key models have lower translation quality, but
also fewer parameters.

LeViT explicitly uses $\dimvalue=2\dimqk$ in ordinary attention
and $4\dimqk$ in downsampling attention. Its positional bias permits
smaller keys, reducing score-computation cost while retaining wider
values \citep[Section~4.2]{graham2021levit}.
\citet{yao2026thinkeys} study smaller Q/K projections with full-width
values for language-model compression.
DeepSeek-V2 provides another distinction: its MLA scores combine
128 content dimensions with 64 positional dimensions, while each
value head has width 128 \citep{deepseek2024v2}. The coupled latent
projections and positional transformations make this a different
architecture from the independent projections tested here.
Our example adds a GE calculation of the temperature-dependent
ranking of two width allocations under the stated ridge objective
and parameter constraint.

%% file: 121_cross_head_learning_experiment.tex
\subsection{Across-Head Sharing and Gradient-Flow Time}
\label{app:cross-head-learning-experiment}
\Cref{fig:sharing-comparison}(a) uses $h=64$, $d=\ntoken=4096$,
and $\dimout=2\ntoken$.
The feature scaling and readout protocol are those of
\cref{app:head-learning-experiment}, including the same fixed target.
We draw $h_{\mathrm{KV}}\in\{64,32,16,8,4,2,1\}$ key projections and $h_{\mathrm{KV}}$ value
projections independently, and share each K/V pair among $h/h_{\mathrm{KV}}$
query heads. All query projections remain
independent, and the output projection is independent of all
heads. The horizontal coordinate $h/h_{\mathrm{KV}}$ therefore runs from $1$
(no sharing) to $64$ (multi-query attention
\citep{shazeer2019fast}).

\paragraph{Width and parameter-count constraints.}
For $\dimqk=\dimvalue$ the number of distinct projection weights is
\[
 \nparam=dh\dimqk+2dh_{\mathrm{KV}}\dimqk+h\dimqk\dimout.
\]
The first, second, and third terms count Q, K/V, and O weights,
respectively.
The fixed-width path keeps $\dimqk=\dimvalue=64$, so $\nparam$ decreases
with $h_{\mathrm{KV}}$. The fixed-parameter-count path retains the unshared
baseline's budget $\nparam=5\ntoken^2$.
Since $d=\ntoken$ and $\dimout=2\ntoken$, its continuous-width
relation is
\[
 \keyratio=\valueratio=\frac{5}{3+2h_{\mathrm{KV}}/h}.
\]
For the seven values of $h_{\mathrm{KV}}$ listed above, rounding gives head
widths $64,80,91,98,102,104,106$. The realized parameter counts differ
from the target by at most $0.47\%$.
The smooth theoretical curve evaluates the continuous-width
relation over $h/h_{\mathrm{KV}}\in[1,64]$, solving
\eqref{eq:learning-curve-bulk} for the arrival time at each value.
Rounding the head widths of these configurations changes the theoretical arrival times
by at most $0.99\%$ relative to the continuous-width curve.
The two paths coincide at $h_{\mathrm{KV}}=h$.

\paragraph{Arrival times and uncertainty.}
For each configuration, we average the exact finite-matrix
learning curves over ten feature initializations, retaining the
fixed target's projections onto the empirical eigenvectors as
in \cref{app:head-learning-experiment}. The plotted time is the
root at which this mean loss equals $0.1$; it is not the mean of
the ten individual arrival times.
We generate ten softmax draws for each of the six values $h_{\mathrm{KV}}<h$
under each constraint, giving $120$ shared-architecture
draws. The ten existing $h=64$ unshared draws supply the common
baseline from \cref{app:head-learning-experiment}, displayed once as a gray diamond.

The theoretical times solve \eqref{eq:learning-curve-bulk} at loss
$0.1$, using \eqref{eq:volaw} and the same scaling as the
measurements. Fixed widths give the common prediction $41.22$.
At fixed target parameter count, the continuous-width prediction
decreases to $7.80$ at $h_{\mathrm{KV}}=1$, or $7.76$ at its rounded width.
For the fully shared configuration, the measured times are $39.11$
at fixed widths and $7.57$ at fixed parameter count.
Over both paths, the largest relative difference
between a measured time and the prediction evaluated at the rounded
head width is $5.14\%$.

The 32- and 48-term inverse-Laplace calculations in
\cref{app:numerical-precision} agree in loss to within
$3\times10^{-12}$ at the predicted arrival times.

%% file: 127_sharing_head_count_experiment.tex
\subsection{Head Count and Sharing across or within Heads}
\label{app:sharing-head-count-experiment}
\Cref{fig:sharing-comparison}(b) compares GQA with complete within-head
K/V tying at equal widths and projection parameter count.
We set $\ntoken=4096$, $d=8192$,
$h\dimqk=h\dimvalue=\dimout=\ntoken/2$.
The head counts $h=2,4,8,16,32$ give
$\dimqk=\dimvalue=1024,512,256,128,64$, respectively.
Projection weights follow \cref{app:readout-experiment-protocol}.

GQA partitions the query heads into pairs, each with a shared key
projection and a separate shared value projection. Distinct pairs
are independent. Complete within-head tying sets $W_r^V=W_r^K$
in every head, with distinct heads independent.
Both schemes therefore have
\[
 \nparam=h\dimqk(2d+\dimout)
\]
distinct Q/K/V/O weights, excluding biases and the fitted readout.
The unshared counterpart at these widths uses
$h\dimqk(3d+\dimout)$ weights.
The comparison between the two sharing schemes thus holds total
K/V widths, output width, and parameter count fixed as $h$ varies.

We scale the full uncentered features by $1/\sqrt m$ as in
\cref{sec:readout-setting}. For GQA,
$m=\outputratio\theta_1$; for complete tying,
$m=\outputratio(\theta_1+\theta_2\ntoken/\dimqk)$.
GQA uses the independent-K/V law in \cref{cor:vo}.
Complete tying uses the law in \eqref{eq:shared-value-output-law}.
After scaling by $1/m$, we compute its coefficient error from
\cref{eq:readout-tied-root,eq:readout-scalar-risk}, using the
scalar solver in \cref{app:readout-computation}.
Curves interpolate the GE formula between integer head counts.

Each marker averages ten independent softmax initializations,
with teacher coefficients and observation noise integrated
analytically as in \cref{app:readout-observables}.
The two schemes share query and output weights within each draw;
each GQA pair uses the first head's key and value weights from
an independent-K/V realization.
The nominal 95\% intervals follow
\cref{tab:section5-uncertainty}. Applying the same Student-$t$
construction to the paired differences gives
$0.04600\pm0.00379$ for tied minus GQA error at $h=4$;
the difference intervals exclude zero at all five head counts,
supporting the reported ordering.
Across the displayed points,
the largest relative difference between the softmax mean and GE
is $0.75\%$, supporting the predicted change in the preferred
sharing scheme at these finite dimensions.

%% file: 098_mha_experiments.tex
\subsection{MHA Equivalence and Architectural Comparisons}
\label{app:mha-experiments}
\Cref{fig:architectural-spectra} uses 10
independent realizations per condition and $\beta=1$.
The panels plot squared singular values of $\sqrt{\ntoken/h}\Aperp$ for attention and
$\MHAperp(X)$ for MHA; the corresponding Gaussian
matrices are $\YG/\sqrt h$ and $\ntoken^{-1/2}\MHAG(X)$.
All $\ntoken$ eigenvalues of each left Gram matrix, including zeros,
enter each empirical distribution with mass $1/\ntoken$.

\input{022_spectral_figures}

\paragraph{Controlled architectural comparisons.}
\Cref{fig:architectural-spectra} fixes $\ntoken=256$ and varies one
architectural choice per panel. For attention, $h=4$ and $h=16$
use $\dimqk=1024$ and 256 respectively, keeping $h\dimqk=4096$.
The MHA baseline has $h=16$, $\dimqk=\dimvalue=256$, and $\dimout=4096$.
Panel (b) changes only $\dimout$ to 256. Panel (c) compares independent
heads, one key matrix and one independent value matrix shared by
all heads, and $V_r=K_r$ within otherwise independent heads.
The independent and within-head-tied comparisons share queries,
keys, output weights, and comparison noise within each draw;
the across-head-shared condition uses a separate seed sequence.
All empirical histograms use 64 bins covering their full ranges.

\paragraph{Deterministic theoretical curves.}
The attention curve is \eqref{eq:limiting-law}; independent and
across-head-shared MHA use \eqref{eq:volaw}; within-head tying with
Gaussian output uses \eqref{eq:shared-value-output-law}.
These curves depend only on $\beta$ and the dimension ratios.
We use multiplicativity of the $S$-transform and
\eqref{eq:column-pop-S} to form the Cauchy-transform equation,
then track its solution from the upper-half-plane asymptotic
$G(z)\sim1/z$. Numerical Stieltjes inversion evaluates
$-\pi^{-1}\operatorname{Im}G(x+10^{-6}i)$, with 192-point
Gauss--Legendre quadrature for the Marchenko--Pastur integral.
At $\dimout=256$, the limiting density diverges at zero,
whereas each histogram bin records a finite average density.

%% file: 022_spectral_figures.tex
\ifdefined\PaperArxivLayout
 \begin{figure}[t]
\else
 \begin{figure*}[t]
\fi
 \centering
 \includegraphics[width=\textwidth]{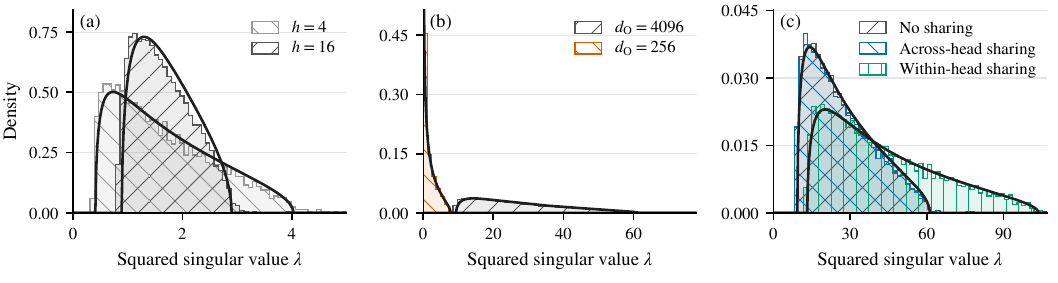}
 \caption{Head splitting concentrates the attention bulk (a),
 output width changes the MHA spectral scale and shape (b), and
 within-head key--value tying changes the limit while across-head
 sharing preserves it (c).
 Histograms show original-model spectra; black curves show the
 theoretical limits. Experimental details are in
 \cref{app:mha-experiments}.}
 \label{fig:architectural-spectra}
\ifdefined\PaperArxivLayout
 \end{figure}
\else
 \end{figure*}
\fi